\documentclass[journal]{IEEEtran}
\IEEEoverridecommandlockouts  

\usepackage{graphicx} % Required for inserting images
\usepackage{amsmath}
\usepackage{amsthm}
\usepackage{amsfonts}
\usepackage{amssymb}
\usepackage{xcolor}
\usepackage{theoremref}
\usepackage{subcaption}
\usepackage{algpseudocode}
\usepackage{algorithm}
\usepackage[hyphens]{url}
\usepackage{cite}
\usepackage[colorlinks=true, allcolors=blue]{hyperref}
\usepackage{comment}

\newtheorem{theorem}{Theorem}
\newtheorem{corollary}{Corollary}[theorem]
\newtheorem{lemma}[theorem]{Lemma}
\newtheorem{exmp}{Example}[section]
\newtheorem{remark}{Remark}
\newtheorem{definition}{Definition}[section]

\title{
Rapidly Converging Time-Discounted Ergodicity on Graphs for Active Inspection of Confined Spaces
}

\author{Benjamin Wong$^{1}$, Ryan H. Lee$^{2}$, Tyler M. Paine$^{3}$, Santosh Devasia$^{1}$, and Ashis G. Banerjee$^{1,4}$
\thanks{This work was supported in part by a Naval Engineering Education Consortium (NEEC) award \# N00174-22-1-0013.}
\thanks{$^{1}$B. Wong and S. Devasia are with the Department of Mechanical Engineering, University of Washington, Seattle, WA 98195, USA,
        {\tt\small bycw,devasia@uw.edu}}%
\thanks{$^{2}$R. H. Lee is with the Department of Electrical \& Computer Engineering, University of Washington, Seattle, WA 98195, USA,
        {\tt\small rhlee1@uw.edu}}%
\thanks{$^{3}$T. M. Paine is with the Naval Undersea Warfare Center, Keyport, WA 98345, USA,
{\tt\small tyler.m.paine.civ@us.navy.mil}}%
\thanks{$^{4}$A. G. Banerjee is with the Department of Industrial \& Systems Engineering and the Department of Mechanical Engineering, University of Washington, Seattle, WA 98195, USA,
        {\tt\small ashisb@uw.edu}}%
}

\begin{document}

\maketitle

\begin{abstract}
Ergodic exploration has spawned a lot of interest in mobile robotics due to its ability to design time trajectories that match desired spatial coverage statistics. However, current ergodic approaches are
for continuous spaces, which require detailed sensory information at each point and can lead to fractal-like trajectories that cannot be tracked easily. This paper presents a new ergodic approach for graph-based discretization of continuous spaces. It also introduces a new time-discounted ergodicity metric, wherein early visitations of information-rich nodes are weighted more than late visitations. A Markov chain synthesized using a convex program is shown to converge more rapidly to time-discounted ergodicity than the traditional fastest mixing Markov chain. The resultant ergodic traversal method is used within a hierarchical framework for active inspection of confined spaces with the goal of detecting anomalies robustly using SLAM-driven Bayesian hypothesis testing. Experiments on a ground robot show the advantages of this framework over three continuous space ergodic planners as well as greedy and random exploration methods for left-behind foreign object debris detection in a ballast tank.
\end{abstract}

\section{Introduction}
\label{section:intro}
Robots have been proposed to 
inspect confined spaces~\cite{swri_crawler, WongBenjamin2023Hrdo, inspection_drone, DimitropoulosNikos2024Brfc},
since they are not designed for prolonged human occupancy \cite{osha_confined_space}.
Examples of such spaces include wing tanks in airplanes, ballast tanks in ships, sewers, and storage tanks such as grain silos. They are usually topologically constrained with different interconnected sections, several functional structures or devices,  
and limited spaces for manuevers. 
The spaces can also be hazardous 
%to humans 
due to the accumulation of harmful gases, engulfment of materials, 
%such as water or grain, 
and presence of inflammable substances \cite{potential-confined-space}. As a result, numerous worker injuries and deaths have occurred in the past decade while working in confined spaces \cite{Zamudio}. However, inspection of such spaces can be quite critical, as for example, retrieval of foreign object debris (FODs) left inside aircraft wings can help prevent accidents~\cite{LeBeau_2020a}. This motivates the use of intelligent robots (rather than humans) 
%for work 
in such confined spaces.
%which enables human workers to stay out of  confined-space work and focus on higher-level decision making. 

%Path planning for confined-space inspection needs to balance both efficient information acquisition as well as coverage of the space. However, 
Ergodic control is well suited for robotic inspection of confined spaces~\cite{MathewGeorge2011Mfea}. 
Traditionally, inspection is treated as a coverage problem, where the planning method focuses on going through every point in the space exactly once (uniform coverage) using pre-computed trajectories
\cite{GalceranEnric2013Asoc}. However, uniform coverage cannot adapt to changing information and uncertainties. 
An alternative approach is to use information acquisition methods that 
%account for changing information and 
consider the inherent uncertainty in the observation and motion models of the robot using 
%uncertainty-based 
metrics, such as entropy and the determinant or trace of the covariance matrices. 
%These information metrics are updated online based on newly acquired observations and the controller adapts to them accordingly. These information acquisition methods 
Accordingly, they are often solved in a model predictive manner, where the control inputs corresponding to the maximum estimated information gain are selected \cite{8260881, 7139863, BestGraeme2019DDpf}. This implies that the 
%When the information gain is maximized, 
robots tend to prioritize inspecting areas with the highest uncertainty, which does not necessarily lead to sufficient coverage of the space. Hence, methods are needed to balance robust coverage and efficient information gain~\cite{MathewGeorge2011Mfea}, which are afforded by ergodic control.

The notion of ergodicity is adapted from statistical mechanics, where a system's long-term average (time average) is equal to the average over the entire state space (space average)~\cite[1.2 p.2]{bremaud2013markov}. Specifically, ergodic control allows one to design a robot trajectory that achieves a specified spatial statistic~\cite{MathewGeorge2011Mfea}, such as information-guided exploration with adequate coverage. This paper develops a new ergodic control approach for graph-based discretization of confined spaces. This is in contrast to existing ergodic control approaches that are designed for continuous spaces~\cite{MathewGeorge2011Mfea}. 

There are three main challenges of 
%the frequency-domain, 
continuous-space ergodic control for inspection of confined spaces, all of which are addressed in our discrete-space method. 
%with constraints. 
First, the continuous space methods tend to yield complex, fractal-like trajectories to satisfy the ergodicity requirements, which are difficult to execute in confined spaces. %However, confined spaces 
%with constraints 
%have limited rooms for maneuvers needed for such trajectories. 
Moreover, the complex 
%fractal-like 
trajectories (i)~can potentially destabilize 
%the linearization-based localization of the 
robot localization by introducing large nonlinearities in the inputs and (ii)~can cause the robot to inspect local featureless areas for long periods of time. 
%when maneuvering the robot. 
A graph-based discretization of the confined space into relatively large regions allows designing a region-wise ergodic traversal method that
%, region-level ergodicity can 
works without considering local collision-free motion planning directly. 
%This facilitates the development of a hierarchical planning framework where ergodicity between regions of the continuous space is achieved at the higher-level and path-planning and collision avoidance are   optimized  by a lower-level controller within each region. 
Second, 
%generating 
the information field for the continuous-space ergodic method requires a sensor model to consider occlusion at every point in space. In contrast, the information is bulk estimated at each node on the graph (region in the space) in our discrete ergodic control method. %instead of for each point in the space. 
Third, the continuous-space ergodic methods cannot model directional traversability constraints in the confined space, such as an unclimbable slope in one direction, or a door that only opens on one side. In contrast, such connectivity constraints are modeled as directed graph edges in our discrete method

In addition to the ergodic graph traversal method, this paper also presents a convex optimization approach to compute a 
%rapidly converging 
policy that rapidly reaches time-discounted ergodicity. 
%faster than traditional methods. 
In time-discounted ergodicity, earlier visitations of information-rich regions have more weights than later visitations, which is necessary for effective inspection within a limited time frame. Note here that ergodicity on a graph can be achieved asymptotically (in time) by sampling from a Markov chain with a stationary distribution that equals a specified target distribution. 
Traditionally, this is done using the \underline{f}astest \underline{m}ixing \underline{M}arkov \underline{c}hain (FMMC) method~\cite{BoydStephen2004FMMC}, which selects the Markov chain that provides 
%the approach is to select the Markov chain associated with the graph traversal for 
fastest convergence from the initial distribution to the target distribution. 
%While FMMC achieves asymptotic ergodicity, 
We show here that FMMC is not the best option to achieve time-discounted ergodicity. 
%Specifically, 
Instead, minimizing the second largest eigenvalue of the stochastic matrix of the associated Markov chain yields a new method, termed \underline{r}apidly \underline{e}rgodic \underline{M}arkov \underline{c}hain (REMC), with faster convergence to time-discounted ergodicity than 
%traditional approach of minimizing the second largest eigenvalue modulus (SLEM) in the   
FMMC. 

This graph-based discretization enables the development of a hierarchical planning framework, which allows decomposition and modularization of both robot control and estimation to reduce the overall computational complexity of active inspection. We focus on anomaly detection as the inspection goal. Consequently, REMC sits at the top of the hierarchy as the region-level planner, where coarse-grained information on the current uncertainty estimates of region-wise anomalies are fed from the region graph. At the second level, fine-grained information on anomaly uncertainty estimates within the regions are used for inspection waypoint placement. At the bottom level, robot state estimates and an occupancy grid map are provided as inputs by a graph-SLAM algorithm for local waypoint navigation. The uncertainty estimates come from a Bayesian anomaly detector at this level, which recursively updates the posterior of the robot's local observations corresponding to anomalies based on the signed Mahalanobis distances of the observations from their nearest points on a reference model of the confined space. The overall inspection framework is termed \underline{h}ierarchical \underline{e}rgodic \underline{Ma}rkov \underline{p}lanner (HEMaP). Both simulation and physical experiments are carried out inside a representative ballast tank with FODs as anomalies, using a ground robot equipped with a depth camera and GPU-based on-board processing. The results show that HEMaP achieves high FOD detection rates (averaging 85\%), and consistently outperforms greedy and random exploration methods regardless of the allowable inspection time frame.    

The main contributions of this work are, therefore, summarized as follows. 
\begin{itemize}
\item First-of-its-kind development of ergodic control for graph-based spatial discretization of continuous spaces
\item
Synthesis of rapidly converging ergodic Markov chains using convex optimization to provide theoretical guarantees when ergodicity (on graphs) is discounted over time
% by 
% minimizing the second largest eigenvalue (rather than the modulus of the second largest eigenvalue) of the stochastic matrix associated with the Markov chain.
\item Application of ergodic graph traversal within a hierarchical planning framework for active inspection of confined spaces with the goal of detecting anomalies robustly
%complex constraints.  
\item Formulation of anomaly detection problem as recursive Bayesian hypothesis testing using signed Mahalanobis distance 
%between local observations and reference model 
in the likelihood function
\end{itemize}

\section{Related Work}
\subsection{Ergodic Control: Continuous Versus Discretized Space}
\label{section:relatedwork_ergodic}
One major area of related work is on ergodic control-based robot exploration in continuous workspaces. The earliest and one of the most widely adopted work proposes a Fourier-based metric to quantify ergodicity in a finite vector space and yield an optimal feedback controller \cite{MathewGeorge2011Mfea}. However, this limits its applicability to regular, rectangular domains, and makes it challenging to scale up with respect to the dimensionality of the domain and the size of the Fourier series. 

There are three main branches of work that attempt to address these challenges. One branch utilizes the Fourier metric as an objective function and expands the method into a trajectory optimization problem, where collision avoidance is handled as an optimization constraint \cite{LerchCameron2023SEEi, AbrahamIan2018DECD, MillerLaurenM2016EEoD, GkouletsosDimitris2021DTOf}. 
%This method opens ergodic control to a wider range of applications as robot motion models and various constraints can be considered. 
However, it inherits a similar problem as the Fourier-based metric, where the finite number of frequency coefficients are not able to capture complex obstacle geometry such as structures in the confined space. When modeled as constraints, the complex geometries, together with the robot motion model, introduce high non-convexity, which generates more computational complexity. 
The second branch of work uses the radial basis function instead of the Fourier basis to adapt to different domain topologies \cite{IvicStefan2017ECMA, IvicStefan2022Cmea, IvicStefan2022MCfA, IvicStefan2023Mtpf}. 
%This method requires a set of 
The basis functions are specific to the domain topology and are pre-computed using prior maps. When unexpected obstacles such as FODs appear, a new set of basis functions may be required. 
%to optimally capture the geometry. 
In contrast, new obstacles are handled by a standard local planner in our hierarchical approach. 
The third line of work changes the ergodic metric 
%from difference in Fourier-coefficient 
to KL-divergence between the robot distribution and the target distribution. 
%This metric is minimized by using a cross-entropy method to find the optimal input. 
%By using KL-divergence, this metric 
As a result, this metric is able to adapt to a wider range of sensor models and has a lower
%the cross-entropy method can reduce the 
computational complexity than the Fourier method \cite{AyvaliElif2017Ecic}. 
This line of work is similar to our approach in that both utilize stochasticity to solve the ergodic control problem. However, instead of using stochastic optimization to efficiently sample trajectories that solve the finite-horizon optimization problem, we find a stochastic policy that is inherently ergodic. 
%Being a finite horizon optimization, the trajectory parameters found is only optimal for the specific initial condition and time horizon. In contrast, 
%Further, our method is globally ergodic and can be adapted to a large set of agents without significant modification.
%

Several works have adapted these three branches of ergodic control methods on various robotics tasks, including target localization \cite{MavrommatiAnastasia2018RACa}, multi-link manipulator control \cite{LowTobias2022dUEC, ShettySuhan2022EEUT, PignatEmmanuel2022Lfdu, Bilaloglu_25}, and swarm control  \cite{PrabhakarAhalya2020ESfF}. In these tasks, the information field is either given or has a well-behaved analytical form. 
%This makes them well suited for continuous space ergodic control. 
However, information is harder to model in inspection applications, which instead benefit from a graph-based abstraction. Recently, advances have also been made on ergodic exploration of cluttered environments with many obstacles. For example, this has been addressed by sampling random graphs in the free space and then using a Dijkstra-like method to search for an ergodic path~\cite{ShiroseBurhanuddin2024GGES}. 
However, this method is computationally challenging and does not guarantee convergence. Alternatively, the method presented in \cite{IvicStefan2017ECMA} has been adapted to a maze-like environment \cite{crnkovic2023fast}. However, this is only applicable to square-lattice graphs instead of arbitrary strongly connected graphs, as in our method. Ergodicity in continuous space with arbitrary topology has also been achieved by a Markov process that satisfies measure-preserving flows \cite{XuAlbert2024MPFf}. %Nevertheless, it is difficult to optimize the convergence rate with this method. 
A similar Markov process is used in our method, but over graphs in discretized spatial regions, which formulates a linear specification of the transition model and enables convex optimization of the convergence rate to ergodicity.

\subsection{Rapidly Converging Markov Chains}

As mentioned before, the ergodic graph traversal problem is solved using a Markov chain method. A Markov chain is characterized by a sequence of events with each event dependent only on the previous one. 
%It has been applied successfully in a wide range of fields to model system behaviors. 
%In robotics, one would like to synthesize a Markov chain to achieve some desired properties. 
In robotics, it is often used to realize a desired system-level property, such as the 
%One such objective is to specify 
stationary or target distribution of visitation frequency of the different regions of a workspace. This can be done using the Metropolis-Hastings (M-H) algorithm, where an arbitrary stationary distribution can be achieved using the detailed balance condition \cite{HastingsW.K.1970MCsm}. Although the M-H algorithm provides asymptotic convergence, it does not consider the rate of convergence. For a finite graph, 
%the fastest mixing Markov chain (FMMC) method 
FMMC can instead be used to compute the Markov chain that converges the fastest to the stationary distribution \cite{BandyopadhyaySaptarshi2017PaDC}. Consequently, FMMC or its variants have been used for 
%In the robotics literature, Markov chain has been applied to 
swarm formation control and task allocation \cite{DjeumouFranck2023PCoH,PatelRushabh2015RSaM,Behcet2015MCAt,BermanS.2009OSPf}. 
%using FMMC or variants of FMMC. 
These works address the problem of distributing a team of robots to follow a target distribution, 
%the fastest, 
for which FMMC is optimal in terms of convergence to the target distribution. However, this paper investigates the problem of actively inspecting confined spaces efficiently. Therefore, we consider a different time-discounted metric of ergodicity, 
%In contrast, for time-discounted ergodicity, 
where earlier visitations of regions with high information have more weights than later visitations so that all the information-rich regions are inspected as thoroughly and quickly as possible. 
% the time-weighted average of the distributions over time should be close to the target distribution. It is shown in the current paper that
In this context, the proposed REMC method is shown to 
%achieve better 
converge faster to
time-discounted ergodicity than FMMC. 

\subsection{Anomaly Detection in Confined Spaces}

Most existing work treats anomaly detection in confined spaces as a single-instance image classification problem \cite{Cao2018, MoradiSaeed2020AADa,GaoQiang2021RoFO, MeyerJohannes2022Viva, KozamernikNejc2025AnFA}. 
%HungYu-Hsin2022DaAD, 
Although these methods handle sensing uncertainty through explicit modeling or machine learning, they do not take into account the context of where the images are taken and the corresponding uncertainty. Furthermore, these methods do not consider the effect of recurring online observations, which is crucial for active inspection. Our work includes the robot localization uncertainty provided by SLAM. The uncertainty is propagated to obtain probabilistic estimates of the observations, which enables online recursive hypothesis testing using the Mahalanobis distance-based method developed in \cite{WongBenjamin2023Hrdo}. 

\section{Hierarchical Planning}
\begin{figure}[thpb]
      \centering
    \includegraphics[width=0.45\textwidth]{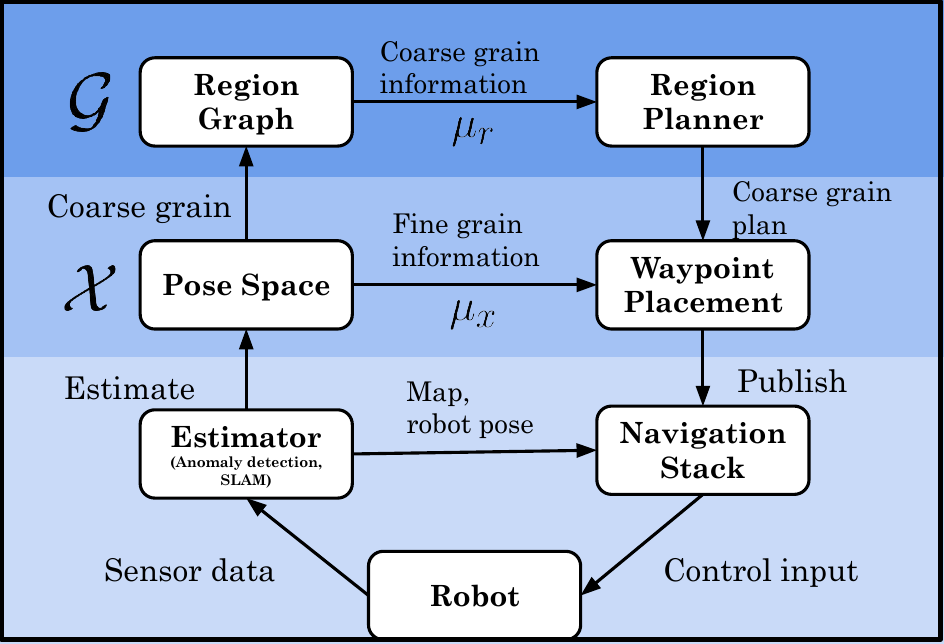}
     % {Figures/hierarchical_framework.png}
      \caption{
      Hierarchical Ergodic Markov Planning (HEMaP) framework.
      % for separating the region-level planning problem and path problem for efficient inspection. 
      At the top level, the graph-level region planner uses coarse-grained information $\mu_r$ on the current uncertainty estimates of region-wise anomalies.
       At the middle-level, 
       fine-grained information on the anomaly uncertainty   within the region $\mu_x$ is used for waypoint placement inside the  region. 
       At the bottom level, robot state estimates and occupancy grid map are provided as inputs for waypoint navigation. }
      \label{fig:h_framework}
   \end{figure}
   
For scalable and interpretable inspection targeted toward anomaly (FOD) detection, planning is decomposed into several levels, as shown in Fig. \ref{fig:h_framework}. The robot first decides which confined space region to go to according to coarse-grained information $\mu_r$ on anomaly uncertainty, defined on a region graph \(\mathcal{G}\) that is stated formally below. The robot then selects a predefined number of poses in the underlying pose space \(\mathcal{X}\) corresponding to the selected region. This is done according to fine-grained information on anomaly uncertainty in the part of the region that is visible to the robot from each pose. The selected poses are used as waypoints to move the robot to region-wise inspection points based on the robot operating system (ROS) navigation stack \cite{nav_core}.  %Formal definition of the region graph and the accompanying graph planning method is described below.
Fig. \ref{h_graph} shows an example of the pose space \(\mathcal{X}\) and the region graph \(\mathcal{G}\) of the ballast tank in Fig. \ref{fig:ballast_cad}. 
%the formal definition will be defined in this section. 
A graph-based ergodic traversal method is chosen as the region planner, which is discussed in Section \ref{section_graph_ergodic}. The estimator is a SLAM-based anomaly detector, whose formulation 
%and the associated information metric 
is discussed in Section \ref{section_anomaly_detection}. The waypoint placement algorithm is presented in Section \ref{section_algorithm}. 

\subsection{Hierarchical Representation of Workspace}
\label{section_hierechical}
\begin{figure}[thpb]
      \centering
      \framebox{\parbox{0.4\textwidth}{ \includegraphics[width=0.4\textwidth]{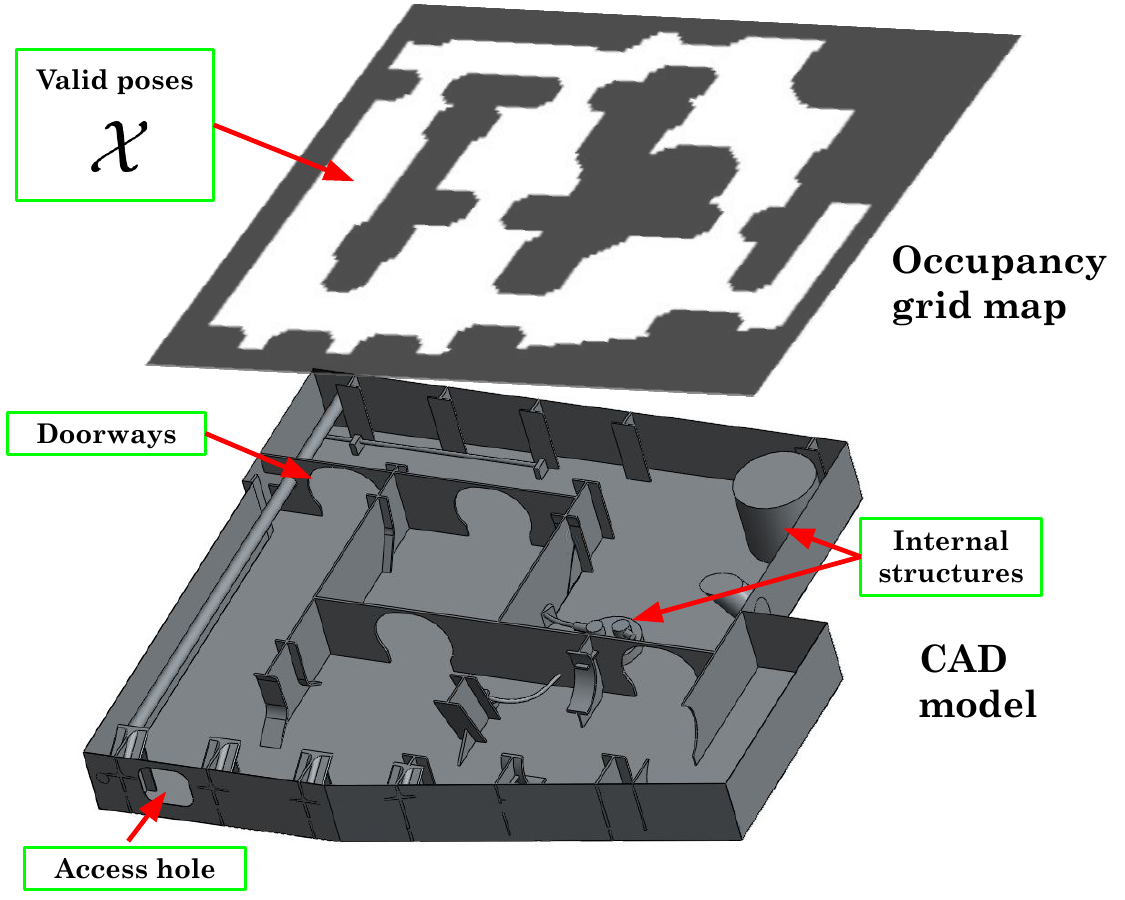}}
}
      \caption{CAD model of an example ballast tank with the roof removed and the corresponding occupancy grid map rendered above. The white grid cells denote feasible 2D poses in the pose space, \(\mathcal{X}\), that can be occupied by a ground robot without colliding with the internal structures.}
      \label{fig:ballast_cad}
   \end{figure}
   
\begin{figure}[thpb]
      \centering
      \framebox{\parbox{0.3\textwidth}{ \includegraphics[width=0.3\textwidth]{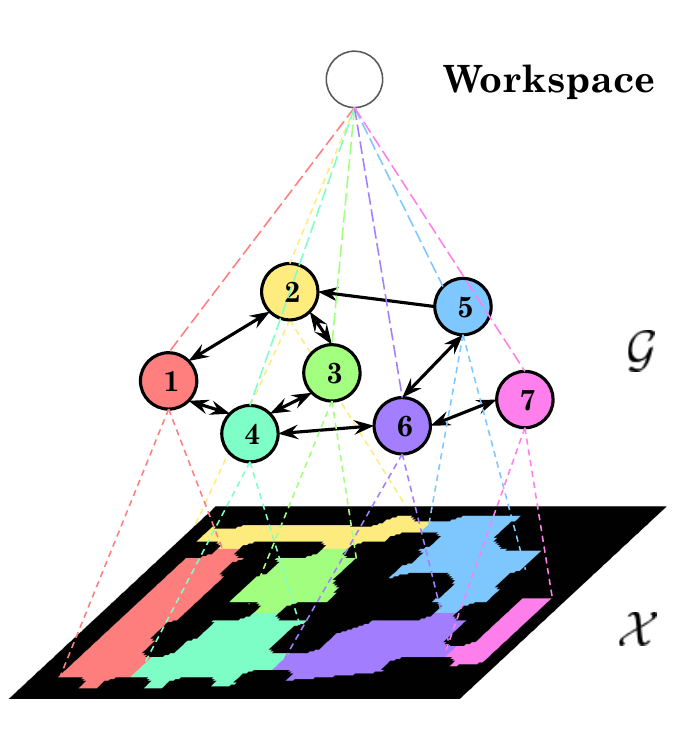}}
}
      \caption{Hierarchical representation of the ballast tank. Top-level is the workspace node, second-level is the region graph \(\mathcal{G}\), the lowest level is the pose space \(\mathcal{X}\) from Fig. \ref{fig:ballast_cad}. }
      \label{h_graph}
   \end{figure}
   
We consider the set of regions in the workspace as an arbitrary coarse graining of the pose space \(\mathcal{X} \subset \text{SE}(3)\) into a finite disjoint set \(\mathcal{R} = \{r_1, r_2, \cdots, r_n\}\), where the space within each region is connected.
\begin{definition} Region nodes:
   \begin{equation}
   \begin{aligned}
        r_i \subset \mathcal{X} \quad & \forall r_i \in \mathcal{R} \\
        r_i \text{ is connected} \quad& \forall r_i \in \mathcal{R}\\
        r_i\cap r_j = \varnothing \quad& \forall r_i,r_j \in \mathcal{R}.
   \end{aligned}
\end{equation}
\end{definition}

For any pair of regions (nodes), if the union is path connected, we denote an edge between the regions.
\begin{definition}
Region edge: 
If \(r_i \cup r_j \) is connected, \((r_i,r_j) \in \mathcal{E}\).
\end{definition}
The graph \(\mathcal{G}=(\mathcal{R}, \mathcal{E})\) forms a region graph of the workspace. Additionally, we require the region graph to be strongly connected such that the robot can travel from any region to any other region. Otherwise, each strongly connected component in the graph would be considered a separated workspace. An example of the pose space and the associated graph representation of the ballast tank in Fig. \ref{fig:ballast_cad} is shown in Fig. \ref{h_graph}.

\section{Graph Ergodic Traversal}
\label{section_graph_ergodic}
For the region planner at the top level of the hierarchical planning framework in Fig. \ref{fig:h_framework}, we borrow the notion of ergodic exploration. We present here that graph ergodic exploration can be achieved by sampling from an ergodic Markov chain. Further, we show that the ergodicity of such a Markov chain can be improved by convex optimization. 

\subsection{Continuous and Discrete Ergodic System}
A system is ergodic when the time average of a dynamic system is equal to the space average over the set the system is operating on, as formally defined below. 
\begin{definition}
Ergodic System

The system \((\mathcal{X}, \Sigma, \mu, \Phi)\) is ergodic in a \(\mu\)-measurable function \(F\)  if
   \begin{equation}
\label{continuous_ergodic}
  \lim_{T\to \infty} \frac{1}{T}  \int_{t=0}^T F(x(t))dt  =  \frac{1}{\mu(\mathcal{X})} \int_{\mathcal{X}}F(x) \mu(dx)
\end{equation} 
for almost all initial conditions \(x(0)\). Here, the right-hand side is the space average of \(F\) over all the elements \(x\) in the set \(\mathcal{X}\) with respect to a measure \(\mu\) over the sigma algebra \(\Sigma\). 
The left-hand side is the (long-term) time average of \(F\) evaluated along the trajectory \(x(t)\) generated by a dynamic map \(\ x(t) = \Phi (t,x(0)): (\mathcal{T} \times \mathcal{X}) \rightarrow \mathcal{X} \) with
 initial state \(x(0)\).
 %under a dynamical map \(\Phi: (t \times \mathcal{X}) \rightarrow \mathcal{X} \). 
\end{definition}

%We now define a similar ergodic system on a graph representation of the space in discrete time as 

 The following is a similar ergodic system on a graph representation of the space in discrete time.
 
 \begin{definition} Graph-based Ergodic System 
\label{def_graph_ergodicity}

The system \((\mathcal{G}, \Sigma, \mu, M)\) is ergodic in a \(\mu\)-measurable function \(F\)  if   
   \begin{equation}
\label{eq_graph_ergodicity}
     \lim_{K \to \infty} \frac{1}{K}  \sum_{k=0}^{K-1} \left(F(r[k])\right )   = \frac{1}{\mu(\mathcal{R})} \sum_{i=0}^nF(r_i)\mu(r_i) 
\end{equation} 
for all initial conditions \(r[0]\). Here, the right-hand side is the space average of \(F\) over all elements \(r\) in the set \(\mathcal{R}\) with respect to a measure \(\mu\) over the sigma algebra \(\Sigma\). 
The left-hand side is the (long-term) time average of \(F\) evaluated along the sequence \(r[k]\) generated by a dynamic map \(\ r[k+1] = M(r[0:k]): \mathcal{R} \rightarrow \mathcal{R} \) with
 initial state \(r[0]\).
\end{definition}

The graph-based ergodicity in~\eqref{def_graph_ergodicity} implies that the relative frequency of the robot visiting each region (i.e., the time average \(\hat{\rho}\)) is equal to the relative measure of each region on the graph (i.e., the space average \(\bar{\rho}\)) if the measurable function \(F\) in \eqref{eq_graph_ergodicity}  is selected as the indicator function \(I\)
\begin{equation}
\label{indicator_function}
\begin{aligned}
 I(r_i) &\triangleq \begin{bmatrix}
 \delta_{1,i} \\
 \delta_{2,i}\\
 \vdots\\
 \delta_{n,i}
\end{bmatrix} 
\end{aligned}
\end{equation} 
with \(\delta_{i,j}\) representing the Kronecker delta,
and where 

\begin{equation}
\label{time_average}
\begin{aligned}
     \hat{\rho}   &\triangleq \lim_{K \to \infty} \frac{1}{K}  \sum_{k=0}^{K-1} \left(I(r[k])\right ) 
\end{aligned}
\end{equation} 
\begin{equation}
\label{space_average}
\begin{aligned}
     \bar{\rho}  &\triangleq \frac{1}{\mu(\mathcal{R})} \sum_{r_i \in \mathcal{R}}I(r_i)\mu(r_i).
 \end{aligned}
\end{equation}

\subsection{Asymptotic Ergodicity}
\label{section:asymp_ergodic}
We propose to achieve graph-based ergodicity by using a \textit{Markov chain}, where
the transition probability between regions is determined by a stochastic matrix $P$ with the stationary distribution equal to the desired distribution. Then, at the limit, the space average is almost surely equal to the desired distribution based on the ergodic theorem of Markov chains~\cite[Theorem 1.10.2]{bremaud2013markov}.

Specifically, consider a stochastic dynamic map $M$ , where the transition probability $P_{j,i}$ of the robot from region $r_i$ to $r_j$ forms a stochastic matrix  $P \in \mathbb{R}^{n\times n}$. Then, asymptotic ergodicity, in the expected sense, is achieved for any stochastic matrix  $P$ that is irreducible, i.e., can go from any region to any other region in the graph \(\mathcal{G}\) (not necessarily in one move), with the space average \(\bar{\rho}\) as the unique stationary distribution, as shown below.

\begin{theorem}[Ergodic Theorem for Ergodic Graph Traversal]
\label{th_asym_ergodic}
    Let the sequence $ r[ \cdot ]$ of regions of the strongly connected graph \(\mathcal{G} \) be sampled according to an irreducible stochastic matrix  \(P\), i.e., 
    \( r[k+1] \sim  \mathbb{P}(R_{k+1}|R_k = r[k]) \), with the expected value $\rho_k = \mathbb{E}[I(r[k])]$ satisfying
    \begin{equation}
    \rho_{k+1} = P\rho_k
    \end{equation}
    and  an expected 
     (infinite) time average  
    \begin{equation}
    \label{eq_expectedtimeaverage}
       \mathbb{E}(\hat{\rho}) = \lim_{K \to \infty} \frac{1}{K}  \sum_{k=0}^{K-1} \rho_k =  \lim_{K \to \infty} \frac{1}{K}  \sum_{k=0}^{K-1} \left(P^k \rho_0 \right ).
    \end{equation}
    Then, the stochastic system  \((\mathcal{R}, \mathcal{B}, \mu, P)\)  is asymptotically ergodic in expected value, i.e.,  
    \begin{equation}
        \mathbb{E}[\hat{\rho}]  = \bar{\rho}
        \label{eq_asymptotic_ergodicity}
    \end{equation}
    for any initial distribution $\rho_0 $, if the space average \( \bar{\rho}\) defined in \eqref{space_average} is a 
    stationary distribution  of the stochastic matrix  
    $P$, i.e.,
    \begin{equation}
    \label{eq_stable_distribution} 
        P \bar{\rho} =  \bar{\rho}.
    \end{equation}
\end{theorem}

\begin{remark}
\label{remark:ergodic_markov}
The term ergodic Markov chain has multiple definitions in the literature, e.g., when the Markov chain is irreducible \cite[p.245]{ParzenEmanuel1929-20161962Sp}, meaning it can transition from any state to any other state in a finite amount of time; or when the Markov chain is both irreducible and aperiodic \cite[p.316]{MeynSean2009MCaS}. In this paper, we adopt the definition most closely aligned with ergodic theory \cite[Theorem 4.1]{bremaud2013markov}, which requires irreducibility (but not aperiodicity) and further stipulates that the time average must equal a specific space average as in Definition (\ref{def_graph_ergodicity}).
\end{remark}

\subsection{Time-Discounted Ergodic Objective}
\label{section:discounted_ergodic}
In contrast to the infinite time for asymptotic ergodicity, a time-discounted ergodicity is defined next to formalize the need for faster ergodicity in time.

\begin{definition}
\label{def_weighted_time_avrg}
The stochastic system \((\mathcal{R}, \mathcal{B}, \mu, P, w)\) is discounted ergodic in expected value if 
\begin{equation} 
\label{eq_weighted_ergodic}
\begin{aligned}
    \mathbb{E}[\hat{\rho}_w]  = \bar{\rho} 
    \end{aligned}
\end{equation}
where the left hand side is 
the expected time-discounted average of the sequence of regions 
\(r[\cdot]\)
\begin{equation} 
\label{weightedtimeaverage}
\begin{aligned}
    \mathbb{E}[\hat{\rho}_w] & \triangleq\lim_{K \to \infty}    \frac{1}{\sum_{k{=0}}^{K-1} w_k}\left(\sum_{k=0}^{K-1} w_kP^k\rho_0 \right )  . 
    \end{aligned}
\end{equation}
and \(w = (w_0, w_1, \cdots, w_k, \cdots)\) is a sequence of weights.
Moreover,  
the time-discounted average operator $f_w$ associated with the weight \(w \) is defined as  
\begin{equation} 
\begin{aligned}
    f_w(\lambda)\triangleq    \lim_{K \to \infty}    \frac{1}{\sum_{k{=0}}^{K-1} w_k}\left(\sum_{k=0}^{K-1} w_k\lambda^k \right ).
    \end{aligned}
\label{eq_fw_operator} 
\end{equation}

\end{definition}

\begin{remark}[Factorial discounting] 
\noindent 
An example of time weighting is a factorial discounting over time, 
\begin{equation}
\begin{aligned}
w_k = 1/k!\\
\end{aligned}
\label{eq_factorial_weighting}
\end{equation}
with the associated time-discounted-average operator in~\eqref{eq_fw_operator} 
\begin{equation} 
\begin{aligned}
    f_w(\lambda)\triangleq    \lim_{K \to \infty}    \frac{1}{\sum_{k{=0}}^{K-1} 1/k! }\left(\sum_{k=0}^{K-1} \frac{1}{k!}\lambda^k \right ) 
    \end{aligned} ~= 
    e^{-1}{e^{\lambda}} .
\label{eq_fw_operator_exp} 
\end{equation}
\end{remark} 

In general, discounted ergodicity \eqref{eq_weighted_ergodic} cannot be achieved for an arbitrary initial condition. Instead, the objective is to select the stochastic matrix \(P\) optimally to reduce the difference between the expected time-discounted average $\mathbb{E}[\hat{\rho}_w]$ in Eq.~\eqref{weightedtimeaverage} and the space average \(\bar{\rho}\) over different initial distributions $\rho_0$, normalized over the initial difference. In other words, we want to minimize the normalized ergodicity deviation (NED) 
\begin{equation} 
\label{NED_objective}
   \begin{aligned}
&  
J_{\mathbb{E}} 
~ = 
\max_{\rho_0}
\frac{\|Q^{1/2}\left(\mathbb{E}[\hat{\rho}_w] - \bar{\rho}\right) \|_2}{\|Q^{1/2}\left(\rho_0-\bar{\rho}\right)   \|_2} ,
\end{aligned}
\end{equation}
where \(Q\) is a given weighting matrix\footnote{A popular method for defining the convergence rate is to calculate the error for an arbitrarily finite time horizon. This is a special case of the discounted time average with a weighting function \(w = [1,1,1\cdots 0,0,0\cdots]\).}.

\subsection{Symmetric Stochastic Matrix}
\label{section:symmetric}
The following shows that for symmetric stochastic matrices \(P\), time-discounted ergodicity can be exactly optimized by minimizing the second largest eigenvalue of \(P\). This is in contrast to minimizing the second largest eigenvalue modulus (SLEM) of \(P\), which would speed up convergence to the desired stationary distribution, e.g., as in~\cite{BoydStephen2004FMMC}.

The  theorem below shows that for symmetric stochastic matrices \(P\),  under conditions typical for discount factors on the weighting scheme \(w \),  optimal time-discounted ergodicity (across all initial distributions) can be achieved by minimizing the second largest eigenvalue of \(P\).

\begin{theorem}
\label{thm_optmimalsym}
    Let the stochastic matrix \(P\) be  symmetric 
    and irreducible. Moreover,  let the time-discounted-average operator $f_w$  in~\eqref{eq_fw_operator} 
    be positive semi-definite and 
    strictly monotonic in $[-1, 1]$ namely 
    $f_w(\lambda_1)> f_w(\lambda_2) \geq 0 $ if 
    $  \lambda_1> \lambda_2  $ for all $\lambda_1,\lambda_2 \in [-1,1]$.
    Then,  the normalized ergodicity deviation in~\eqref{NED_objective} with weighting matrix \(Q = I\) is minimized
    for the stationary distribution \(\bar{\rho} = \frac{1}{n}\textbf{1}\), 
    if and only if
    the second largest eigenvalue  (SLE)
    is minimized, i.e.,
    \begin{equation}
    \label{eq_second_largest_eig}
    \begin{aligned}
     & \arg\min_P \lambda_2(P) 
    \end{aligned}
    \end{equation}
    where 
    $  1=\lambda_1 > \lambda_2 \geq \lambda_3 \geq\cdots \lambda_n$.
\end{theorem}

\begin{proof}
    Since \(P\) is a symmetric stochastic matrix, the eigenvalues of \(P\), \(\lambda(P)\), are real and 
    lies within the domain \([-1,1]\), and it can be diagonalized as $P  = V\Lambda U^T$, 
    where \(\Lambda\) is a diagonal matrix of all the eigenvalues of \(P\), \(V\) is a matrix with its column being the right eigenvectors of \(P\), and \(U\) is a matrix with its column being the left eigenvectors of \(P\)  ordered according to \(\Lambda\).
    Moreover, \(f_w\) is a power series, and the set of eigenvalues \(\lambda(P)\) lies within it's domain, \(f_w(P)\), 
    obtained by replacing the  scalar $\lambda$ in \eqref{eq_fw_operator} with the matrix $P$, 
    \begin{equation}
    \begin{aligned}
        f_w(P) & \triangleq \lim_{K \to \infty}    \frac{1}{\sum_{k{=0}}^{K-1} w_k}\left(\sum_{k=0}^{K-1} w_kP^k\right )
    \end{aligned}
    \end{equation}
    exists, has the same eigenvectors as \(P\), and the eigenvalues of \(f_w(P)\) are the scalar mapping \(f_w(\lambda(P))\) of the eigenvalues $\lambda(P)$  of the matrix \(P\) \cite[Definition 1.2. on p.3]{matrix_function}, i.e., 
    \begin{equation}
    \label{analytic}
    \begin{aligned}
    %    f_w(P) & \triangleq \lim_{K \to \infty}    \frac{1}{\sum_{k{=0}}^{K-1} w_k}\left(\sum_{k=0}^{K-1} w_kP^k\right ) \\
        f_w(P) &= Vf_w(\Lambda) U^T.
    \end{aligned}
    \end{equation}
    
    \noindent 
    With \eqref{analytic}, the numerator in \eqref{NED_objective}  can be written as 
    \begin{equation}
    \label{eq_theorem_2_0}
    \begin{aligned}
             & \left \|\mathbb{E}[\hat{\rho}_w] - \bar{\rho} \right \|_2  =    \left \|f_w(P)\rho_0 - \bar{\rho}\right \|_2 
            %  \\
            % & \quad 
            =   \left \|  Vf_w(\Lambda) U^T\rho_0 -\bar{\rho} \right \|_2 \\
            & ~ =   \left \|f_w(\lambda_1)v_1u_1^T\rho_0+ \left(\sum_{i=2}^{n}  f_w(\lambda_i)v_iu_i^T\rho_0 \right) -\bar{\rho}\right \|_2 \\
            \end{aligned}
    \end{equation}

    \noindent 
    where \(\lambda_1=1\) is the largest eigenvector of the stochastic matrix \(P\), \(u_1 = \textbf{1}\) is the associated left eigenvector and \(v_1 = \bar{\rho}\). This results in 
            \begin{equation}
    \begin{aligned}
    & \left \|\mathbb{E}[\hat{\rho}_w] \!- \bar{\rho} \right \|_2 
    % \\
    % & \quad 
    \!= \!\! \left \|f_w(1)\bar{\rho}\textbf{1}^T\rho_0 + \left(\sum_{i=2}^{n}  f_w(\lambda_i)v_iu_i^T\rho_0 \right) -\bar{\rho}\right \|_2 \!.
            \end{aligned}
    \end{equation}
    Using  \(\textbf{1}^T\rho_0 = 1\) for any probability-distribution vector, and  \(f_w(1) =1\)  by construction in~\eqref{eq_fw_operator}, the previous expression can be rewritten as 
    \begin{equation}
    \label{eq_temp_1_theorem_2}
    \begin{aligned}
    & \left \|\mathbb{E}[\hat{\rho}_w] - \bar{\rho} \right \|_2 
    % \\
    % & \quad 
    =
    \left \|\bar{\rho} + \left(\sum_{i=2}^{n}  f_w(\lambda_i)v_iu_i^T\rho_0 \right) -\bar{\rho} \right \|_2  
    \\
    & \quad 
    =  
            \left \|
    \sum_{i=2}^{n}  f_w(\lambda_i)v_iu_i^T\rho_0
            \right\|_2 
    %  \\
    % & \quad 
    ~=
    \left\|
     V_{\perp} f_w(\Lambda_{\perp}) U_{\perp}^T\rho_{0{\perp}}
     \right \|_2  \\
      & \quad  = \left\|f_w\left(P|_{u_1^\perp}\right)\rho_{0{\perp}} \right \|_2 . 
    \end{aligned}
    \end{equation}
    Here, \(P|_{u_1^\perp} \) is the restriction of \(P\) in the subspace orthogonal to \(u_1\),  \( \rho_{0{\perp}}\) is the component of \( \rho_0\) that is orthogonal to \(u_1\), and 
    $V_{\perp}$,\(I_{\perp}\), $\Lambda_{\perp}$ are the associated restrictions of $V$,\(U\), $\Lambda$ in ~\eqref{eq_theorem_2_0}, respectively. 
    Then, the NED objective in~\eqref{NED_objective} can be written using~\eqref{eq_temp_1_theorem_2} as 
    \begin{equation}
    \begin{aligned}
    &J_{\mathbb{E}} =   \max_{\rho_{0}}
    \frac{\|\mathbb{E}[\hat{\rho}_w] - \bar{\rho} \|_2}{\|\rho_0-\bar{\rho}   \|_2} 
    %\\ =&
    ~= \max_{\rho_{0\perp}}\frac{\left\|f_w\left(P|_{u_1^\perp}\right)\rho_{0\perp} \right \|_2}{\|\rho_{0\perp}\|_2} .
    \end{aligned}
    \label{eq_temp_2_theorem_2}
    \end{equation}
    Since the set of \(\rho_{0\perp} = \epsilon(\rho_0) x\) spans the orthogonal space, (see lemma in Appendix \ref{tightness}), 
the right hand side (rhs) of~\eqref{eq_temp_2_theorem_2} is equivalent to the induced matrix norm as 
    \begin{equation}
    \label{orthogonal_objective}
    \begin{aligned}
   & \max_{\rho_{0\perp}}\frac{\left\|f_w\left(P|_{u_1^\perp}\right)\rho_{0\perp} \right \|_2}{\|\rho_{0\perp}\|_2} = \max_{\|x\|_2=1}\frac{\left\|f_w\left(P|_{u_1^\perp}\right)\epsilon(\rho_0) x \right \|_2}{\|\epsilon(\rho_0) x\|_2} 
     \\
      & \quad =  \max_{\|x\|_2=1}\left\|f_w\left(P|_{u_1^\perp}\right) x \right \|_2 
     % \\
     % =& 
    = \left\|f_w\left(P|_{u_1^\perp}\right) \right\|_2.
    \end{aligned}
    \end{equation}
    Moreover, \(f_w(P|_{u_1^\perp})\) is positive semi-definite since, from symmetry of \(P\), \(P|_{u_1^\perp}\) is symmetric. Therefore, its  2-norm is equal to the maximum eigenvalue, i.e., 
    \begin{equation}
    \label{eq_induced_2_norm}
    \begin{aligned}
    & \min_P \left\|f_w\left(P|_{u_1^\perp}\right) \right\|_2
    %\\= & 
    ~= \min_P \lambda_{\text{max}}\left( f_w\left(P|_{u_1^\perp}\right)\right)
    \end{aligned}
    \end{equation}
    and since \(f_w \lambda \) is strictly monotonic in $\lambda$, the eigenvalue of $f_w\left(P|_{u_1^\perp}\right)$ is maximized 
    if and only if the eigenvalue of \(P|_{u_1^\perp}\) is maximized, i.e.,
    \begin{equation}
    \begin{aligned}
     \min_P \lambda_{\text{max}}\left( f_w\left(P|_{u_1^\perp}\right)\right)
    =  
    f_w\left( \min_P 
    \lambda_{\text{max}}\left( P|_{u_1^\perp}\right)\right)
    \end{aligned}
    \label{eq_theorem_2_proof_3}
    \end{equation}
    where 
    $    \lambda_{\text{max}}\left( P|_{u_1^\perp}\right) = \max_{i\neq1} \lambda_i(P) $.
    From \eqref{eq_temp_2_theorem_2}-\eqref{eq_theorem_2_proof_3},   
    \begin{equation} 
    \label{NED_P_objective}
       \begin{aligned}
    \arg\min_P \left(   \left\|f_w\left(P|_{u_1^\perp}\right) \right\|_2 \right)
    = 
    \arg\min_P \left( \max_{i\neq1} \lambda_i(P) \right)
    \end{aligned}
    \end{equation}
    leading to the theorem's result in~\eqref{eq_second_largest_eig}. 
\end{proof}

\begin{remark}[Uniform weighting]
    \label{remark:uniform_weighting}
    The  time-discounted-average  operator $f_w$ in Eq.~\eqref{eq_fw_operator} does not satisfy the strictly monotonic condition in Theorem~\ref{thm_optmimalsym} 
    with uniform weighting 
     \(w = (1,1,1, \cdots)\), 
    since 
    $f_w(\lambda) =0$ if $ \lambda \in [-1,1)$ and 
     $f_w(\lambda) =1$ $ \lambda =1$.
    Therefore, for any symmetric 
    and irreducible
    stochastic matrix \(P\) satisfying 
     $   P\bar{\rho} = \bar{\rho}$
    in~\eqref{eq_stable_distribution} 
    with \(\bar{\rho} = \frac{1}{n}\textbf{1}\), the ergodic deviation becomes zero from~\eqref{eq_temp_1_theorem_2} , i.e., 
    \begin{equation}
    \label{eq_remark_uniform_w_2}
    \begin{aligned}
    & \left \|\mathbb{E}[\hat{\rho}_w] - \bar{\rho} \right \|_2 
    ~ =  
            \left \|
    \sum_{i=2}^{n}  f_w(\lambda_i)v_iu_i^T\rho_0
            \right\|_2 
    ~=
    0 .
    \end{aligned}
    \end{equation}
    Thus, the uniformly-weighted  ergodicity deviation in~\eqref{NED_objective} is minimized with any  symmetric 
    and irreducible
    stochastic matrix \(P\),
    as in Theorem~\ref{th_asym_ergodic}.
\end{remark}
The  symmetric optimal stochastic matrix $P$,
that minimizes the normalized ergodicity deviation as in Theorem~\ref{thm_optmimalsym}, 
 can be found by solving a convex optimization problem, as shown in the following corollary.

\begin{corollary}
    The optimal symmetric stochastic matrix \(P\)  in Theorem~\ref{thm_optmimalsym},  with the minimum, second-largest eigenvalue,  can be  obtained by solving the following convex optimization problem.
    \begin{equation} 
    \label{opt_Corollary_Theorem_2}
       \begin{aligned}
        \arg \min_p  &\quad \lambda_{\text{max}}\left( P - \frac{2}{n}\textbf{1}\textbf{1}^T \right) & \\
        \text{s.t.} &\quad  \textbf{1}^TP = \textbf{1}^T &\quad  
        {(\mbox{Stochastic}}~P)
        \\
            & \quad P_{i,j} \geq 0
            &\quad 
        {(\mbox{Stochastic}}~P)
        \\
            & \quad P = P^T
                &\quad 
        {(\mbox{Symmetric}}~P)
        \\
            & \quad P_{i,j} = 0 \quad \text{if} \quad (j,i) \notin \mathcal{E} &
            ~~ 
        {(\mbox{Transitions in}}~\mathcal{G}).
    \end{aligned}
    \end{equation}
\end{corollary}
\begin{proof}
    The optimization follows since the second largest eigenvalue \(\max_{i\ne1} \lambda_i (P) \) 
    %in \eqref{eq_SLE}  
    is  the largest eigenvalue 
    of $ P - \frac{2}{n}\textbf{1}\textbf{1}^T$, which is a    Wielandt deflation \cite{cdi_unpaywall_primary_W1879119993} of the largest eigenvalue $\lambda_1$  of $P$ from $\lambda_1 =1$  to $-1$ ( the smallest possible eigenvalue of a symmetric stochastic matrix), i.e., 
    \begin{equation}
    \begin{aligned}
    \lambda(P - \frac{2}{n}\textbf{1}\textbf{1}^T) &= \{\lambda_1-2, \lambda_2, \lambda_3, \cdots \}\\
    &= \{-1, \lambda_2, \lambda_3, \cdots \} \\
    \Rightarrow \lambda_{\text{max}}(P - \frac{2}{n}\textbf{1}\textbf{1}^T) &= \max_{i\ne1} \lambda_i (P) ~~{\mbox{since}~} \lambda_i \ge -1.
    %\\ &= \lambda_{\text{max}}(P|_{v_1^\perp})
    \end{aligned}
    \end{equation}
    
    The minimization of the largest eigenvalue of a symmetric matrix ($P - \frac{2}{n}\textbf{1}\textbf{1}^T$) is convex~\cite[Example 3.10 p.82]{bremaud2013markov} and all the constraints are linear, and therefore ~\eqref{opt_Corollary_Theorem_2} is a convex optimization problem. 
\end{proof}

\subsection{Nonuniform Target Distribution}
Although the optimal symmetric stochastic matrix \(P\) with the minimum second-largest eigenvalue exactly minimizes 
normalized ergodicity deviation, the only possible target distribution is the uniform distribution \(\bar{\rho} = \frac{1}{n}\textbf{1}\). For a general target distribution \(\bar{\rho} \) with no zero entries,  if the detailed-balance condition is satisfied, i.e., 
\begin{equation}
\label{eq_detailed_balance_condition}
    P\Pi = \Pi P^T, \quad \Pi \triangleq \text{diag}(\bar{\rho}),
\end{equation}
it is shown that the stochastic matrix  $P$ with the minimum second largest eigenvalue minimizes the normalized ergodicity deviation with the weighting matrix  \(Q  \) chosen to be inversely proportional to the target distribution \(\bar{\rho}\). In other words, \(Q = \Pi^{-1} \), 
which can be interpreted as normalizing the error vector by the size of the target distribution.

\begin{remark}[Detailed balance usage]
The detailed balance condition in~\eqref{eq_detailed_balance_condition}, i.e., that the Markov chain being reversible, has been applied to generate Markov chain with the desired target distribution in the Metropolis-Hasting algorithm \cite{HastingsW.K.1970MCsm}, and subsequently required to select the optimal stochastic matrix $P$ for the fastest mixing Markov chain in \cite{BoydStephen2004FMMC}. 
\end{remark}

\begin{corollary}
%\thlabel{detailed_balance}
\label{detailed_balance}
Selecting the stochastic matrix \(P\) to have the  smallest,  second-largest eigenvalue as in Eq.~\eqref{eq_second_largest_eig} minimizes the 
normalized ergodicity deviation in~\eqref{NED_objective} with weighting \(Q = \Pi^{-1} \) for a given target distribution \(\bar{\rho} \) with no zero entries. 
Additionally,  the optimal stochastic matrix \(P\) can be obtained by solving the following convex optimization problem 
\begin{equation}
\begin{aligned}
    \arg\min_P & \quad \lambda_{\text{max}}(\Pi^{-1/2}(P-2\bar{\rho}\textbf{1}^T)\Pi^{1/2} ) 
\end{aligned}
\end{equation}
with the same conditions as in Eq.~\eqref{opt_Corollary_Theorem_2} but with the symmetry $ P = P^T$ condition replaced by the 
detailed-balance condition in Eq.~\eqref{eq_detailed_balance_condition}.

\end{corollary}

\begin{proof}
The normalized ergodic deviation in~\eqref{NED_objective} can be rewritten as 
\iffalse
\begin{equation}
\begin{aligned}
      & \max_{\rho_0}
\frac{\|Q^{-1/2}(\mathbb{E}[\hat{\rho}_w]  - \bar{\rho}) \|_2 }{\|Q^{-1/2}\left(\rho_0-\bar{\rho}\right)   \|_2} \\
&=  \max_{\rho_0}\frac{\|Q^{-1/2}f_w(P)\rho_0  - Q^{-1/2}\bar{\rho}) \|_2 }{\|\left(Q^{-1/2}\rho_0-Q^{-1/2}\bar{\rho}\right)   \|_2}\\
&= \max_{\rho_0}\frac{\|Q^{-1/2}f_w(P)Q^{1/2}\Pi^{-1/2}\rho_0  - \Pi^{-1/2}\bar{\rho}) \|_2 }{\|\left(Q^{-1/2}\rho_0-Q^{-1/2}\bar{\rho}\right)\|_2}\\
& = \max_{\rho_0}\frac{\|f_w(Q^{-1/2}PQ^{1/2})Q^{-1/2}\rho_0  - Q^{-1/2}\bar{\rho}) \|_2 }{\|\left(Q^{-1/2}\rho_0-Q^{-1/2}\bar{\rho}\right)\|_2}\\
& = \max_{\tilde{\rho}_0}
\frac{\|(f_w(\tilde{P})\tilde{\rho}_0 - \bar{\rho}^{1/2}) \|_2}{\|(\tilde{\rho}_0-\bar{\rho}^{1/2})   \|_2}. 
\end{aligned}
\end{equation}
where 
\begin{equation}
    \tilde{\rho} = Q^{-1/2}\rho, \quad 
\tilde{P} = Q^{-1/2}PQ^{1/2} 
\end{equation}
and \(\bar{\rho}^{1/2}\) is the entry-wise square root of \(\bar{\rho}\)

\fi
\begin{equation}
\begin{aligned}
\label{corollary2_eq1}
      & \max_{\rho_0}
\frac{\|Q^{1/2}(\mathbb{E}[\hat{\rho}_w]  - \bar{\rho}) \|_2 }{\|Q^{1/2}\left(\rho_0-\bar{\rho}\right)   \|_2} 
%\\ &
= \max_{\rho_0}
\frac{\left\|Q^{1/2}f_w\left(P|_{u_1^\perp}\right)\rho_{0{\perp}} \right\|_2 }{\|Q^{1/2}\rho_{0\perp}   \|_2} \\
 & \qquad= \max_{\rho_0}
\frac{\left\|Q^{1/2}f_w\left(P|_{u_1^\perp}\right)Q^{-1/2}Q^{1/2}\rho_{0{\perp}} \right\|_2 }{\|Q^{1/2}\rho_{0\perp} \|_2} \\
& \qquad= \max_{\rho_0}
\frac{\left\|f_w\left(Q^{1/2}P|_{u_1^\perp}Q^{-1/2}\right)Q^{1/2}\rho_{0{\perp}} \right\|_2 }{\|Q^{1/2}\rho_{0\perp}   \|_2}  \\
& \qquad= \max_{\rho_0}
\frac{\left\|f_w\left(\tilde{P}|_{u_1^\perp}\right)\tilde{\rho}_{0{\perp}} \right\|_2 }{\|\tilde{\rho}_{0\perp}   \|_2} 
\end{aligned}
\end{equation}
with 
\begin{equation}
\label{definition_P_tilde} 
    \tilde{\rho} = Q^{1/2}\rho = \Pi^{-1/2}\rho, \quad 
\tilde{P} = Q^{1/2}PQ^{-1/2} = \Pi^{-1/2}P\Pi^{1/2} 
\end{equation}
where 
%Since \(\bar{\rho}\) is strictly positive, 
\(\Pi =  \text{diag}(\bar{\rho}) \) is invertible and its square root exists since \(\bar{\rho}\) is strictly positive. 
Note that \(\tilde{P}\) is symmetric since  the 
detailed balance condition
in~\eqref{eq_detailed_balance_condition}, results in 
\begin{equation}
\begin{aligned}
%         P\Pi &= \Pi P^T\\
        \Pi^{-1/2} P\Pi^{1/2} &= \Pi^{1/2} P^T\Pi^{-1/2} 
%        \\        \tilde{P} &= \tilde{P}^T, 
\end{aligned}
\end{equation}
i.e., $\tilde{P} = \tilde{P}^T$, 
which also implies that the 
restriction of \(\tilde{P}\) to the subspace orthogonal to \(u_1\), i.e.,  \(\tilde{P}|_{u_1^\perp}\) is symmetric, since all the eigen-components, \(\lambda_iv_iu_i^T\), of \(\tilde{P}\) are symmetric.
Then, following arguments as in Eqs.~\eqref{eq_temp_2_theorem_2}- \eqref{NED_P_objective} in the proof ofTherorem~\ref{thm_optmimalsym}, 
\begin{equation}
\begin{aligned}
     &\arg \min_P \max_{\rho_0}
\frac{\left\|f_w\left(\tilde{P}|_{u_1^\perp}\right)\tilde{\rho}_{0{\perp}} \right\|_2 }{\|\tilde{\rho}_{0\perp}   \|_2} \\
    & \qquad \qquad =  \arg \min_P \lambda_{\text{max}}(\tilde{P}|_{u_1^\perp}).
\end{aligned} 
\end{equation}
The corollary follows since a coordinate transform preserves eigenvalues, and therefore, 
\begin{equation}
\begin{aligned}
      \lambda_{\text{max}}(\tilde{P}|_{u_1^\perp}) = \max_{i \neq 1} \lambda_{i}(\tilde{P}) = \max_{i \neq 1} \lambda_{i}(P).
\end{aligned}
\end{equation}
\end{proof}

\subsection{Convex Bounding for General Markov Chain}
The detailed balance condition in the previous subsection can be restrictive for ergodicity. An obvious example is when transitions are allowed in only one direction between some regions, because the detailed balance condition leads to zero probability even for the available transition. Thus, in such one-way transition cases,  the detailed balance condition results in a larger normalized ergodicity deviation compared to the case without the condition, as illustrated in Fig.~\ref{fig_detailed_balance}.
\begin{figure}[thpb]
      \centering
      \framebox{\parbox{0.4\textwidth}{ \includegraphics[width=0.4\textwidth]{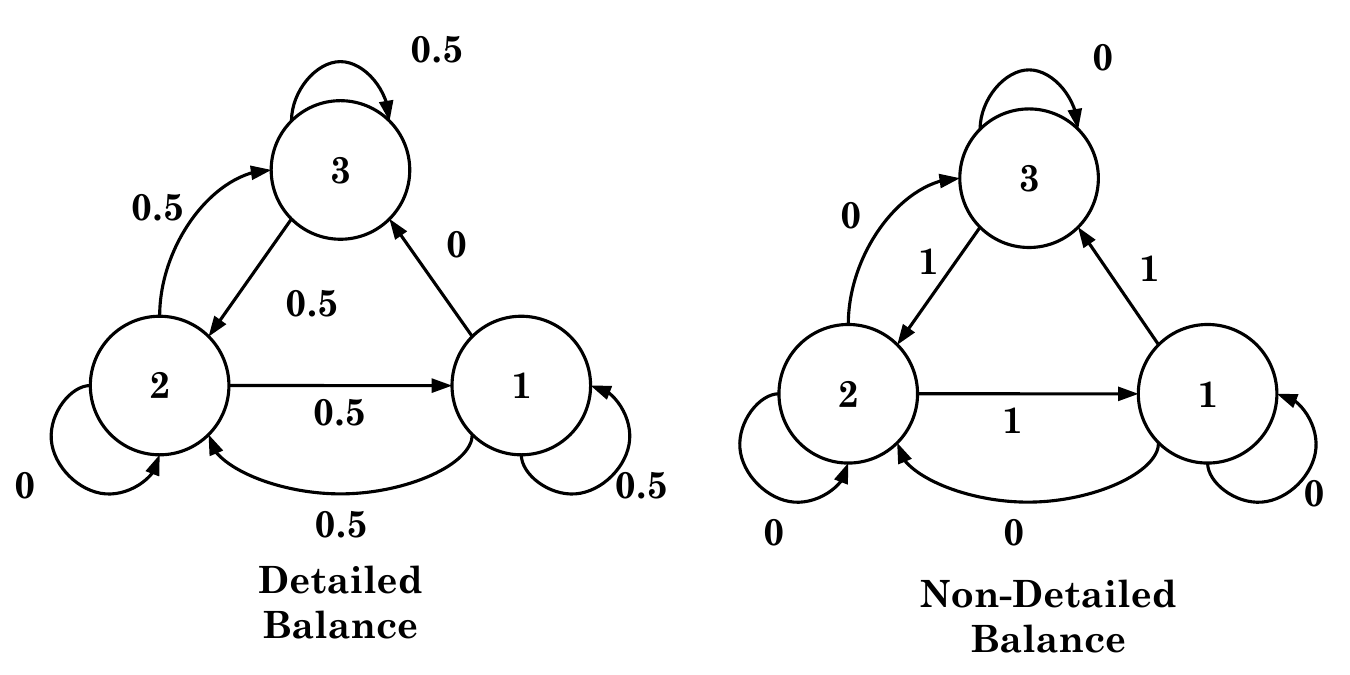}}
}
      \caption{Example three-node-graph with one directed edge with the uniform final distribution \(\bar{\rho}\) and weighting  \(w_k = 1/k!\) illustrating the detailed balance condition can be detrimental to ergodicity.  When detailed balance is enforced, the REMC stochastic matrix has a zero transition probability on the one-way directed edge, resulting in a normalized ergodicity deviation  \eqref{NED_objective} of  
      \(  J_{\mathbb{E}}  =   0.606\) (left),  which is substantially higher than 
      \(  J_{\mathbb{E}}  =   0.223\) for the case 
      when detailed balance is not enforced (right). }

      \label{fig_detailed_balance}
   \end{figure}

   The detailed balance requirement can be relaxed for general directed graphs (e.g., with potential one-way transitions), by optimizing an upper bound on the normalized ergodicity deviation, as shown in the theorem below. 
   
\begin{theorem}
%\thlabel{thm_optimalbound}
\label{thm_optimalbound}
    The normalized ergodicity deviation in~\eqref{NED_objective} with the same weighting \(Q = \Pi^{-1} \) as in Corollary~\ref{detailed_balance} and a factorial weighting \( w_k = \frac{1}{k!} \) in Eq.~\eqref{eq_factorial_weighting} 
    can be bounded as 
     \begin{equation}
    \begin{aligned}
    & \max_{\rho_0}
    \frac{\|\Pi^{-1/2}\left(\mathbb{E}[\hat{\rho}_w]  - \bar{\rho}\right) \|_2}{\|\Pi^{-1/2}\left(\rho_0-\bar{\rho} \right)   \|_2}    \\
     & \qquad \qquad \le \max_{i\neq 1}\lambda_i\left(e^{-1}\exp\left(\frac{1}{2}\left(\tilde{P}+\tilde{P}^T\right)\right)\right)     \\
    \end{aligned}
         \label{bound_theorem_3}
        \end{equation}
    where \( \tilde{P} = \Pi^{-1/2}P\Pi^{1/2} \), as in~\eqref{definition_P_tilde}.
    %      \begin{equation}
    % \tilde{P} = \Pi^{-1/2}P\Pi^{1/2}.
    %     \end{equation}
    Moreover, the bound in~\eqref{bound_theorem_3}  can be minimized by reducing the second largest eigenvalue of the symmetric part of \(\tilde{P}\)
         \begin{equation}
    \arg\min_P \left(\max_{i \neq 1} \lambda_i \left(\frac{1}{2}\left(\tilde{P}+\tilde{P}^T\right)\right) \right)
        \end{equation}
\end{theorem}
\begin{proof}
    With the factorial weighting \( w_k = \frac{1}{k!} \) in~\eqref{eq_factorial_weighting},  
    the associated time-discounted-average operator \(f_w\) 
    in~\eqref{eq_fw_operator_exp} 
    is the exponential function, 
    \begin{equation}
    \begin{aligned}
     %        w_k &= \frac{1}{k!}\\
          f_w(P) &= \frac{1}{e}\exp(P)
    \end{aligned}
    \end{equation}
    and therefore, \(f_w(P)\) exists for any square matrix \(P\).
    Using the steps in~\eqref{corollary2_eq1}, the left hand side (lhs) of~\eqref{bound_theorem_3} can be rewritten as 
    \begin{equation}
    \label{upperboundobjective}
    \begin{aligned}
    J_{\mathbb{E}} ~ & = 
    \max_{\rho_0}
    \frac{\|\Pi^{-1/2}\left(\mathbb{E}[\hat{\rho}_w]  - \bar{\rho}\right) \|_2}{\|\Pi^{-1/2}\left(\rho_0-\bar{\rho} \right)   \|_2}  \\
    ~ & = \left\|f_w(\tilde{P}|_{u_1^\perp})\right\|_2  = \left\|\exp(\tilde{P}|_{u_1^\perp})\right\|_2 \\
    ~ & 
    \le \lambda_{max}\left(\exp\left(\frac{1}{2}(\tilde{P}|_{u_1^\perp}+(\tilde{P}|_{u_1^\perp})^T)\right)\right)
    ~ {\mbox{from~\cite{C.A.Desoer2012FSIP} }} \\
    ~ & = \overline{J_{\mathbb{E}}} (P) 
    \end{aligned}
    \end{equation}
    \noindent 
    since~\cite{C.A.Desoer2012FSIP} 
    \begin{equation}
    \begin{aligned}
        \left\|\exp(\tilde{P}|_{u_1^\perp})\right\|_2 \leq \lambda_{max}\left(\exp\left(\frac{1}{2}(\tilde{P}|_{u_1^\perp}+(\tilde{P}|_{u_1^\perp})^T)\right)\right).
    \end{aligned}
    \end{equation}
    
    \noindent 
    Since \((\tilde{P}|_{u_1^\perp}+(\tilde{P}|_{u_1^\perp})^T)\) is symmetric,  the eigenvalues are real and the exponential of the eigenvalues is monotonic. Thus, the bound $\overline{J_{\mathbb{E}}} (P)$ can be minimized by selecting \(P\) as 
    
    {\small{
    \begin{equation}
    \begin{aligned}
    \arg \min_P \overline{J}_{\mathbb{E}} (P)
    = 
    \arg \min_P\lambda_{max}\left(\exp\left(\frac{1}{2}(\tilde{P}|_{u_1^\perp}+(\tilde{P}|_{u_1^\perp})^T)\right)\right)\\
    =\arg \min_P\exp\left(\lambda_{max}\left(\frac{1}{2}(\tilde{P}|_{u_1^\perp}+(\tilde{P}|_{u_1^\perp})^T)\right)\right)\\
           = \arg \min_P \lambda_{max}\left(\frac{1}{2}(\tilde{P}|_{u_1^\perp}+(\tilde{P}|_{u_1^\perp})^T)\right)\\
    \end{aligned}
    \end{equation}
    }}

    \noindent 
    The symmetric part of 
    \( \tilde{P}|_{u_1^\perp}\) 
    % \(\left(\frac{1}{2}(\tilde{P}|_{u_1^\perp}+(\tilde{P}|_{u_1^\perp})^T)\right)\) 
    is related to the symmetric part of \(\tilde{P}\) by 
    \begin{equation}
    \label{th3_sym_perp}
    \begin{aligned}
    & \frac{1}{2}\left(\tilde{P}  +\tilde{P}^T\right)\\
         & =  \frac{1}{2}\left(\Pi^{-1/2}P\Pi^{1/2}+(\Pi^{-1/2}P\Pi^{1/2})^T\right)\\
     & = \frac{1}{2}\left(\Pi^{-1/2}(\bar{\rho}\textbf{1}^T + P|_{u_1^\perp})\Pi^{1/2}+(\Pi^{-1/2}(\bar{\rho}\textbf{1}^T + P|_{u_1^\perp})\Pi^{1/2})^T\right) \\
         & = \frac{1}{2}\left((\bar{\rho}^{1/2}\bar{\rho}^{T/2} + \tilde{P}|_{u_1^\perp})+(\bar{\rho}^{1/2}\bar{\rho}^{T/2} + \tilde{P}|_{u_1^\perp})^T\right) \\
         & = \bar{\rho}^{1/2}\bar{\rho}^{T/2} +\frac{1}{2}\left( \tilde{P}|_{u_1^\perp}+(\tilde{P}|_{u_1^\perp})^T\right) 
    \end{aligned}
    \end{equation}
    Furthermore, the symmetric part of \(\tilde{P}\) has a eigenpair of \(\lambda_1 = 1\) and \(u_1 = v_1 = \bar{\rho}^{1/2}\), 
    \begin{equation}
    \begin{aligned}
    \frac{1}{2}(\tilde{P} & +\tilde{P}^T) \bar{\rho}^{1/2} \\
          =& \frac{1}{2}(\Pi^{-1/2}P\Pi^{1/2}+\Pi^{1/2}P^T\Pi^{-1/2}) \bar{\rho}^{1/2}\\
          =& \frac{1}{2}(\Pi^{-1/2}P\bar{\rho}+\Pi^{1/2}P^T\textbf{1})\\
          =& \frac{1}{2}(\Pi^{-1/2}\bar{\rho}+\Pi^{1/2}\textbf{1}) \\
          =& \frac{1}{2}(\bar{\rho}^{1/2}+\bar{\rho}^{1/2})~ = \bar{\rho}^{1/2}
    \end{aligned}
    \end{equation}
    Then, \(\bar{\rho}^{1/2}\bar{\rho}^{T/2}\) in \eqref{th3_sym_perp} is an eigen-component of the symmetric part of \(\tilde{P}\), thus, the maximum eigenvalue of \(\left(\frac{1}{2}(\tilde{P}|_{u_1^\perp}+(\tilde{P}|_{u_1^\perp})^T)\right)\) is the largest eigenvalue of the symmetric part of \(\tilde{P}\) other than \(\lambda_1 = 1\)
    \begin{equation}
    \begin{aligned}
     \lambda_{\max}(\frac{1}{2}\left( \tilde{P}|_{u_1^\perp}+(\tilde{P}|_{u_1^\perp})^T\right) ) = \max_{i \neq 1} \lambda_i \left(\frac{1}{2}\left(\tilde{P}+\tilde{P}^T\right)\right).
    \end{aligned}
    \end{equation}
\end{proof}

The following corollary states the convex optimization problem for Theorem \ref{thm_optimalbound}, which is referred in the paper as \textit{REMC}.

\begin{corollary}[Rapidly Ergodic Markov Chain (REMC)]

The optimal \(P\) for the upper-bound in Theorem~\ref{thm_optimalbound} can be obtained by solving the convex optimization problem
    \begin{equation}
\label{eq_remc}
\begin{aligned}
    \arg\min_P & \quad \lambda_{\text{max}}\left(\frac{1}{2}\left(\tilde{P}+\tilde{P}^T\right)-2\bar{\rho}^{1/2}\bar{\rho}^{T/2}\right) 
\\
     \text{s.t.} &\quad  \textbf{1}^TP = \textbf{1}^T \quad (\text{Stochastic }P)\\
        & \quad P\bar{\rho} = \bar{\rho} \quad (\text{Target Distribution})\\
        & \quad P_{i,j} \geq 0 \quad (\text{Stochastic }P)\\
        & \quad P_{i,j} = 0 \quad \text{if} \quad (j,i) \notin \mathcal{E} \quad {(\mbox{Transitions in}}~\mathcal{G})
\end{aligned}
\end{equation}
with \(\bar{\rho}^{1/2}\) being the entry-wise square root of \(\bar{\rho}\).
\end{corollary}
\begin{proof}
    By Perron-Forbenius theorem, for any irreducible non-negative matrix, there can only be one non-negative eigenvector, and such eigenvector is associated with the largest eigenvalue \(\lambda_1\) \cite{bremaud2013markov}. Since \(\bar{\rho}^{1/2}\) is a non-negative eigenvector of \(\frac{1}{2}\left(\tilde{P}+\tilde{P}^T\right)\) with associated eigenvalue of \(1\), then \(1\) is the largest eigenvalue of \(\frac{1}{2}\left(\tilde{P}+\tilde{P}^T\right)\). Therefore similar to previous theorems, the second largest eigenvalue of  \(\left(\frac{1}{2}(\tilde{P}|_{\tilde{v}_1^\perp}+(\tilde{P}|_{\tilde{v}_1^\perp})^T)\right)\) can be access by the deflation 
\begin{equation}
\begin{aligned}
     &\lambda_{max}\left(\frac{1}{2}(\tilde{P}|_{\tilde{v}_1^\perp}+(\tilde{P}|_{\tilde{v}_1^\perp})^T)\right)\\
     & \quad = \max_{i\neq 1}\lambda_i \left(\frac{1}{2}(\tilde{P}+(\tilde{P})^T)\right)\\
     & \quad =     \lambda_{max}\left(\frac{1}{2}(\tilde{P}+(\tilde{P})^T) -  2 \bar{\rho}^{1/2}\bar{\rho}^{T/2}\right).
\end{aligned}
\end{equation}
\end{proof}
\subsection{Connection to Fastest Mixing Markov Chain } 
An alternative approach to speed-up ergodicity is to reach the desired stationary distribution \(\bar{\rho}\) faster at every time step, consider some 
finite time horizon, \(K\), i.e., 
\begin{equation}
\begin{aligned}
&\arg\min_P \max_{\rho_0} \sum_{k=0}^{K-1} \frac{\|\Pi^{-1/2}(P^k \rho_0 - \bar{\rho}) \|_2}{\|\Pi^{-1/2}(\rho_0 - \bar{\rho}) \|_2} .
\end{aligned}
\end{equation}
 Similar to \eqref{corollary2_eq1}, this can be reduced to
\begin{equation}
\begin{aligned}
&\arg\min_P \max_{\rho_0} \sum_{k=0}^{K-1} \frac{\|\Pi^{-1/2}(P^k \rho_0 - \bar{\rho}) \|_2}{\|\Pi^{-1/2}(\rho_0 - \bar{\rho}) \|_2} \\
 & \qquad = \arg\min_P \max_{\rho_0} \sum_{k=0}^{K-1} \frac{\|\Pi^{-1/2}(P^k \rho_0 - \bar{\rho}) \|_2}{\|\Pi^{-1/2}(\rho_0 - \bar{\rho}) \|_2} \\
 & \qquad = \arg\min_P  \sum_{k=0}^{K-1}\|(\tilde{P}|_{u_1^\perp})^k \|_2.
\end{aligned}
\end{equation}
If the detailed balance is enforced, 
then the 2-norm is equal to the largest absolute value of the eigenvalue, since 
\begin{equation}
\begin{aligned}
% & 
 \arg\min_P  \sum_{k=0}^K \|(\tilde{P}|_{u_1^\perp})^k \|_2
 %\\  & \qquad 
  & =  \arg\min_P  \sum_{k=0}^K |\lambda_2(P)^k| \\
  & \qquad =  \arg\min_P  |\lambda_2(P)| .
\end{aligned}
\end{equation}
This is equivalent to the fastest mixing Markov chain (FMMC) formulation, 
where the stochastic matrix  $P$ is selected such that the SLEM, \(|\lambda_2|\), is minimized. 
\iffalse
\begin{equation}
\begin{aligned}
 & \arg\min_P |\lambda_2(P) |, \quad 1=\lambda_1 > |\lambda_2| \geq |\lambda_3| \geq\cdots |\lambda_n|
\end{aligned}
\end{equation}
\fi
The FMMC formulation upper-bounds the REMC by ensuring fast convergence on the discounted time average since,  
\begin{equation}
\begin{aligned}
 \left \| \frac{1}{K}\sum_{k=0}^{K-1}P^k \rho_0 - \bar{\rho}\right \| & = \frac{1}{K}\left \| \sum_{k=0}^{K-1}P^k \rho_0 - \bar{\rho}\right \|  \\
 & \leq  \frac{1}{K} \sum_{k=0}^{K-1}\left \|P^k \rho_0 - \bar{\rho}\right \|.
\end{aligned}
\end{equation}

The difference in ergodicity achieved with the FMMC formulation that minimizes the SLEM and the proposed REMC approach that minimizes the second largest eigenvalue is illustrated in the following 2-node example. An additional example with a 3-node graph is illustrated in Appendix \ref{section:3_node_example}.

\begin{exmp}
\label{example:fmmc_v_remc}
Comparison of FMMC and REMC on a two node graph: 
FMMC aims to reach the target distribution as quickly as possible by minimizing the SLEM, while REMC aims to reach ergodicity as quickly as possible by minimizing the second largest eigenvalue. To illustrate the difference, both methods are applied to a  two-node graph example (where all transitions are possible)  for a uniform target distribution, and the resulting optimal Markov chains are shown in 
Fig.\ref{fmmc_v_ergodic}. 
The left-hand-side graph in Fig.~\ref{fmmc_v_ergodic} shows the result for FMMC, with the  stochastic matrix 
\begin{equation}
P_{\text{FMMC}} = \begin{bmatrix}
0.5 & 0.5 \\
0.5 & 0.5 \\
\end{bmatrix}
\end{equation}
which has a second eigenvalue of \(\lambda_2 = 0\) that is the smallest possible SLEM. 
For the same two-node graph, the right-hand-side graph in Fig.~\ref{fmmc_v_ergodic}  shows the result for REMC, with the stochastic matrix 
\begin{equation}
P_{\text{REMC}} = \begin{bmatrix}
0 & 1 \\
1 & 0 \\
\end{bmatrix}
\end{equation}
which has a second eigenvalue of \(\lambda_2 = -1\) that is the smallest possible second eigenvalue. 

    \begin{figure}[thpb]
      \centering
      \framebox{\parbox{0.4\textwidth}{ \includegraphics[width=0.4\textwidth]{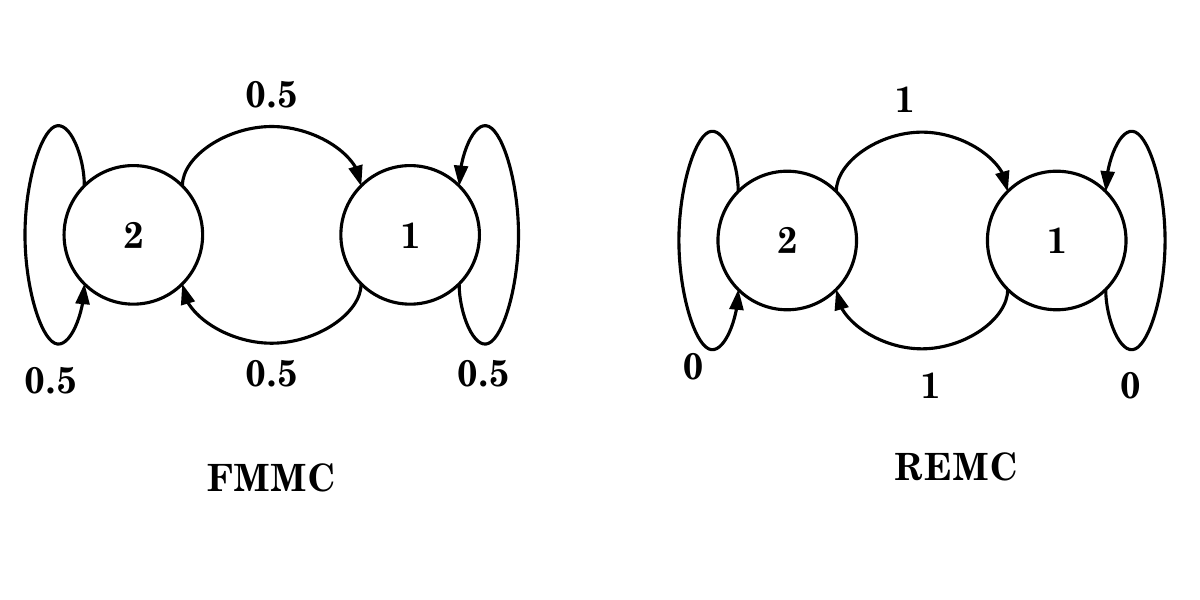}}
}
      \caption{A two-node example illustrating the difference between fastest mixing (left) and optimally ergodic (right) for uniform target distribution.}
      \label{fmmc_v_ergodic}
   \end{figure}

With an initial distribution \(\rho_0 = (1,0) \), the FMMC reaches the target distribution in one step, with the distribution \(\rho_k = P_{\text{FMMC}}^k\rho_0\) at each time step \( k\) given by 
\begin{equation}
\begin{aligned}
      \{\rho_{\text{FMMC, k}}\} = \left\{\begin{bmatrix}
    1 \\ 0
\end{bmatrix}, \begin{bmatrix}
    0.5 \\ 0.5
\end{bmatrix}, \begin{bmatrix}
    0.5 \\ 0.5
\end{bmatrix},\begin{bmatrix}
    0.5 \\ 0.5
\end{bmatrix} \cdots \right\}.
\end{aligned}
\end{equation}
In contrast, the REMC is cyclic and never reaches the target distribution, and the distribution
\(\rho_k = P_{\text{REMC}}^k\rho_0\)
at each time step \( k\) is given by 
\begin{equation}
\begin{aligned}
  \{\rho_{\text{REMC,k}}\} = \left\{\begin{bmatrix}
    1 \\ 0
\end{bmatrix}, \begin{bmatrix}
    0 \\ 1
\end{bmatrix}, \begin{bmatrix}
    1 \\ 0
\end{bmatrix},\begin{bmatrix}
    0 \\ 1
\end{bmatrix} \cdots \right\}.
\end{aligned}
\end{equation}
However, although FMMC reaches the target uniform distribution fast (in one time step), the time average at each time step \(\hat{\rho}_K = \frac{1}{K}\sum_{k=0}^{K-1}\rho_k\) 
\begin{equation}
\begin{aligned}
  \{\hat{\rho}_{\text{FMMC},K}\} = \left\{\begin{bmatrix}
    1 \\ 0
\end{bmatrix}, \begin{bmatrix}
    0.75 \\ 0.25
\end{bmatrix}, \begin{bmatrix}
    0.66 \\ 0.33
\end{bmatrix},\begin{bmatrix}
    0.625 \\ 0.375
\end{bmatrix} \cdots \right\}
\end{aligned}
\end{equation}
does not converge as quickly due to the non-zero initial condition that decays slowly. 
On the other hand, the time average of REMC 
\begin{equation}
\begin{aligned}
  \{\hat{\rho}_{\text{REMC}, K}\} = \left\{\begin{bmatrix}
    1 \\ 0
\end{bmatrix}, \begin{bmatrix}
    0.5 \\ 0.5
\end{bmatrix}, \begin{bmatrix}
    0.66 \\ 0.33
\end{bmatrix},\begin{bmatrix}
    0.5 \\ 0.5
\end{bmatrix} \cdots \right\}
\end{aligned}
\end{equation}
is exactly the target distribution at every second step. 
Thus, when considering the time average, instead of having \(\rho_k\) reaching the target distribution quickly as in FMMC, %it can be made cyclic such that 
the effects of the initial condition can be reduced more rapidly with REMC.
\end{exmp}

Thus, an ergodic graph traversal for the region planner in Fig. \ref{fig:h_framework} can be achieved by sampling from a Markov chain, and a fast convergence ergodic Markov chain can be obtained for any target distribution by using \eqref{eq_remc}. 

\iffalse
{\color{blue} This section showed that in Theorem~\ref{th_asym_ergodic} that graph ergodicity can be achieved by sampling from a irreducible Markov chain. 
Furthermore, as shown in \thref{thm_optmimalsym}, the rate of convergence can be optimized for symmetric Markov chain and reversible Markov chain (satisfies detailed balance condition) by minimizing the second largest eigenvalue of the stochastic matrix \(P\). And for any general Markov chain, the rate can be bounded by optimizing the second largest eigenvalue of the convex upper-bound \(\left(\tilde{P}+\tilde{P}^T\right)\), as shown in \thref{thm_optimalbound}. }
\fi

\section{Estimation Framework}
\label{section_anomaly_detection}

We now discuss the estimation module in Fig.~\ref{fig:h_framework} that comprises SLAM and anomaly detection. In addition to the occupancy grid map and pose estimates for navigation, the estimator also provides the measure \(\mu\) (defined in \eqref{eq_graph_ergodicity}) to the REMC method. Specifically, this measure is selected as information entropy using a Bayesian anomaly (FOD) detection method. 
The core idea of detecting FODs is to evaluate the hypothesis that the difference between the local (spatial) point cloud given by the robot's depth camera and the reference model of the corresponding section of the confined space is significant. Consequently, the FOD detection method uses the marginal posterior of our SLAM formulation described in the following Section. 

In short, the SLAM module (Section~\ref{section:slam}) provides optimal estimates of the robot's poses and map features along with their uncertainties. These estimates are then used to transform the local spatial point cloud, sampled by the RGB-D camera, into the global coordinate system, which subsequently help with anomaly detection (Section.\ref{section:detection}). Finally, the information metric \(\mu_p\) is generated from the confidence level of the anomaly detection. A flow chart outlining the steps is shown in Fig. \ref{fig:est_module}. 

\begin{figure}[thpb]
      \centering
      \framebox{\parbox{0.4\textwidth}{ 
      \includegraphics[width=0.4\textwidth]
      {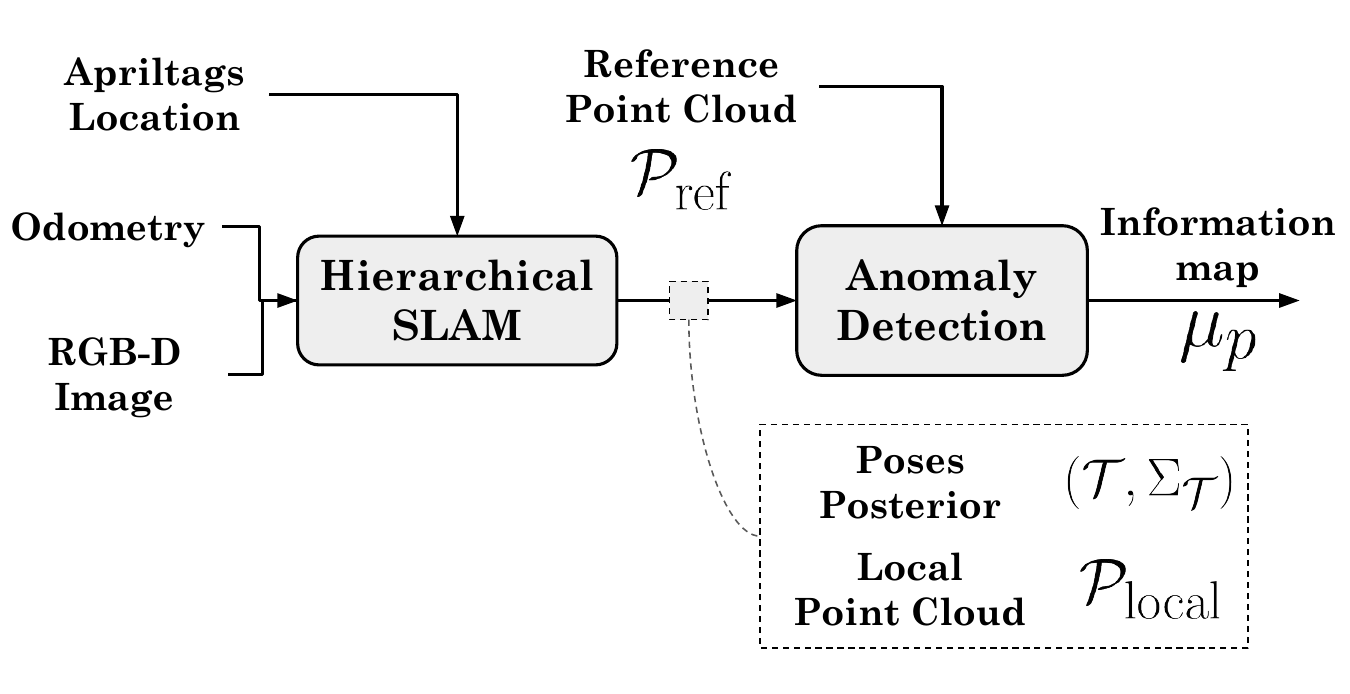}
      }}
      \caption{Flow chart of the estimation module. Hierarchical SLAM processes the sensor information to generate a Gaussian posterior of the pose, \((\mathcal{T}, \Sigma_\mathcal{T})\), detailed implementation of which is discussed in Section \ref{section:slam}. The anomaly detector compares the local point cloud from the RGB-D camera \(\mathcal{P}_\text{local}\) with the reference point cloud \(\mathcal{P}_\text{ref}\) using the posterior to generate the information map \(\mu_p\).
      }
      \label{fig:est_module}
   \end{figure}
   
\subsection{Hierarchical SLAM}
\label{section:slam}
The SLAM module serves two purposes in our inspection framework. First, it provides a probabilistic mapping of the environment in the form of a 3D point cloud of the structures that is used for FOD detection via Bayesian inference. Second, it estimates the robot pose (location and orientation with respect to the map origin), enabling navigation within the space.To accommodate non-linear robot motion and sensor models, we use the Graph-SLAM algorithm. In this method, a factor graph is constructed during run-time based on the collected observations and odometry, with nodes representing the robot’s poses and map features. %\cite[ch.11.3, p.346]{alma99139772740001452}. 

Specifically, the poses are calculated on the \(\text{SE}(3)\) manifold, i.e., a set of 3D rigid transformation matrices, according to the formulation in \cite{micro_lie}, and are optimized by the Gauss–Newton algorithm described in \cite{GrisettiG2010AToG}. We modify it as a hierarchical SLAM, similar to \cite{BosseM.2003AAff} and \cite{GrisettiGiorgio2010Hoom}. Instead of maintaining the full resolution factor graph and performing hierarchical optimization, marginalization is done as an intermediate EKF-SLAM, and only the sparse factor graph is maintained with the posterior from the EKF-SLAM as the factor between the nodes. 
In particular, EKF-SLAM is reset and a new pose node is created based on the distance traveled from the most recent pose node. Special nodes, called \textit{inspection nodes}, are created once the robot reaches the waypoint specified by the waypoint placement method. The corresponding depth point clouds, i.e., the point clouds generated by the depth images taken by the robot at the waypoints, are also stored at these inspection nodes.
%for inspection purposes. 
Furthermore, nodes that fall outside the current \textit{estimation horizon} are pruned and replaced by a factor node through marginalization. Once the Graph-SLAM is optimized, the estimated poses of all the unpruned nodes, \(\mathcal{T} \triangleq \{T_1, T_2, \cdots T_i \cdots\}\), and the corresponding covariance matrices, \(\Sigma_\mathcal{T} \triangleq \text{Var}(\mathcal{T})\), are used for recursive anomaly detection described in the next Section.

\subsection{Bayesian Anomaly Detection}
\label{section:detection}

As indicated earlier, we perform Bayesian inference on the likelihoods of multiple competing hypotheses for anomaly detection. In particular, we use the following null and alternate hypotheses, denoted by $H_0$ and $H_1$, respectively.
\begin{equation}
\begin{aligned}
        H_0 \text{: }& \text{The local point cloud is sampled from the reference}
        \\& \text{model of the confined space.}\\
        H_1 \text{: }& \text{The local point cloud is sampled from a space outside a}
        \\& \text{buffer zone around the reference model.}
\end{aligned}
\end{equation}
The buffer zone indicates an area around a confined space structure that is neither considered a FOD nor a part of the structure. This buffer zone allows for deviation in the modeling of the structure or objects that are too small to be considered FODs. This buffer zone is specified by a distance threshold, \(d_\text{buffer}\), offset in the direction of the normal vector of the walls of the reference model, as illustrated %by the white area 
in Fig. \ref{fig:anomaly_detect}.  
 
The probabilistic observation model requires the posterior estimation of the pose of the inspection nodes from the hierarchical SLAM, which is in the form of a multivariate Gaussian, \(\mathcal{N}(x, \Sigma_x\)), with \(x\) being the mean vector and \(\Sigma_x \) being the covariance matrix. In particular, \textcolor{blue}{\(x\) } is a concatenation of all the estimated poses and the blocks of \(\Sigma_x\) are the covariance matrices associated with the poses,  
\begin{equation}
\label{eq_graph_posterior}
x= \begin{bmatrix}
\vdots  \\
x_i \\
x_j \\
\vdots \end{bmatrix}\Sigma_x = \begin{bmatrix}
\ddots  &   &  &  \\
  & \Sigma_{i,i} & \Sigma_{i,j} &  \\
 & \Sigma_{i,j}^T & \Sigma_{j,j} &  \\
 &  &  & \ddots  \\
\end{bmatrix}.
\end{equation}

The marginal posterior of a particular inspection node can be obtained from the graph posterior by marginalization. For a Gaussian distribution, it is done by selecting the corresponding entries in \(x\) and blocks from the diagonal of \(\Sigma_x\) \cite{DellaertFrank2017FGfR}. For instance, the marginal posterior of node \(i\) in \eqref{eq_graph_posterior} is 
\begin{equation}
\label{eq:pose_node_posterior}
\mathbb{P}(x_i) = \mathcal{N}(x_i, \Sigma_{i,i}).
\end{equation}
In the \(\text{SE}(3)\) formulation, \(x_i\) is represented by a \(4 \times 4\) rigid transformation matrix \(T_i\). The corresponding covariance is a \(6 \times 6\) matrix with respect to the 6 degrees of freedom of \(T_i\) on the manifold. The local point cloud, \(\mathcal{P}_{\text{local}, i}\), with a set of points, \(p_{\text{local}} \in \mathbb{R}^3\), is stored in the local coordinate of its associated pose node, with some independent observation noise \(\Sigma_{j, \text{,local}}\) for each point \(p_{j,\text{local}}\). An example is illustrated in Fig. \ref{fig:anomaly_detect}(a). 

Once the pose graph is optimized, the local point cloud, \(\mathcal{P}_{\text{local},i}\), is transformed into the global coordinate with \(T_i\)
\begin{equation}
p_{j} = T_i p_{j,\text{local}},\quad \forall p_{j, \text{local}} \in \mathcal{P}_{\text{local},i}
\end{equation}
with the noise propagated as 
\begin{equation}
\Sigma_{j} = J_\mathcal{X}^{\mathcal{X}p}\Sigma_{i,i}\left (J_\mathcal{X}^{\mathcal{X}p}\right)^T +  J_p^{\mathcal{X}p}\Sigma_{j, \text{local}}\left (J_p^{\mathcal{X}p}\right)^T.
\end{equation}
The exact form of the Jacobians (\(J_\mathcal{X}^{\mathcal{X}p}, J_p^{\mathcal{X}p}\)) is taken from \cite{micro_lie}.

After transformation, the point cloud is compared with the reference model \((p_{\text{ref}}, n_{\text{ref}})\), where \(p_{\text{ref}}\) is the representation of the point cloud and \(n_{\text{ref}}\) are the normal vectors of the points that are inherited from the plane the points are sampled from. As illustrated in Fig. \ref{fig:anomaly_detect}(b), for each point, \(p_{j}\), in the observed point cloud, we find the corresponding reference point as 
\begin{equation}
\begin{aligned}
    p_{\text{ref},j}^* = \arg \min_{p_\text{ref}} & (p_j-p_{\text{ref}})^T\Sigma_{j}^{-1}(p_j-p_{\text{ref}}).
\end{aligned}
\end{equation}

 \begin{figure}[thpb!]
    \centering
      %\fbox{ 
    \begin{subfigure}[t]{0.23\textwidth}
        \centering
        \framebox{\parbox{0.95\textwidth}{ \includegraphics[width=0.95\textwidth]{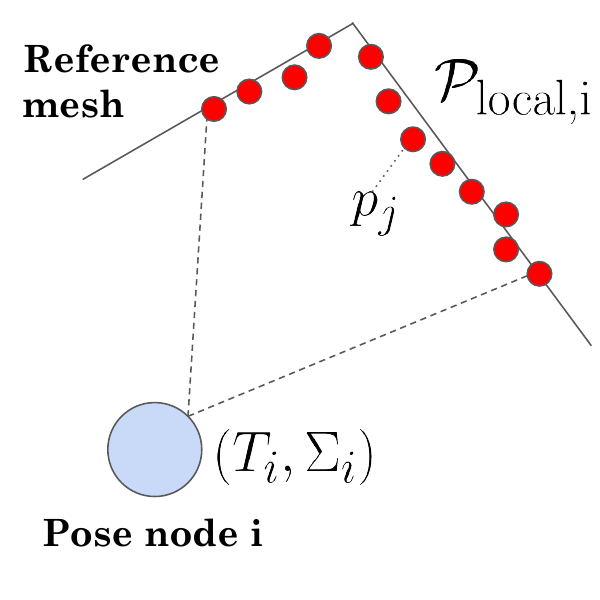}}}
        \caption{}
        \label{fig:depth}
    \end{subfigure}%
    ~ 
    \begin{subfigure}[t]{0.23\textwidth}
        \centering
        \framebox{\parbox{0.95\textwidth}{
        \includegraphics[width=0.95\textwidth]{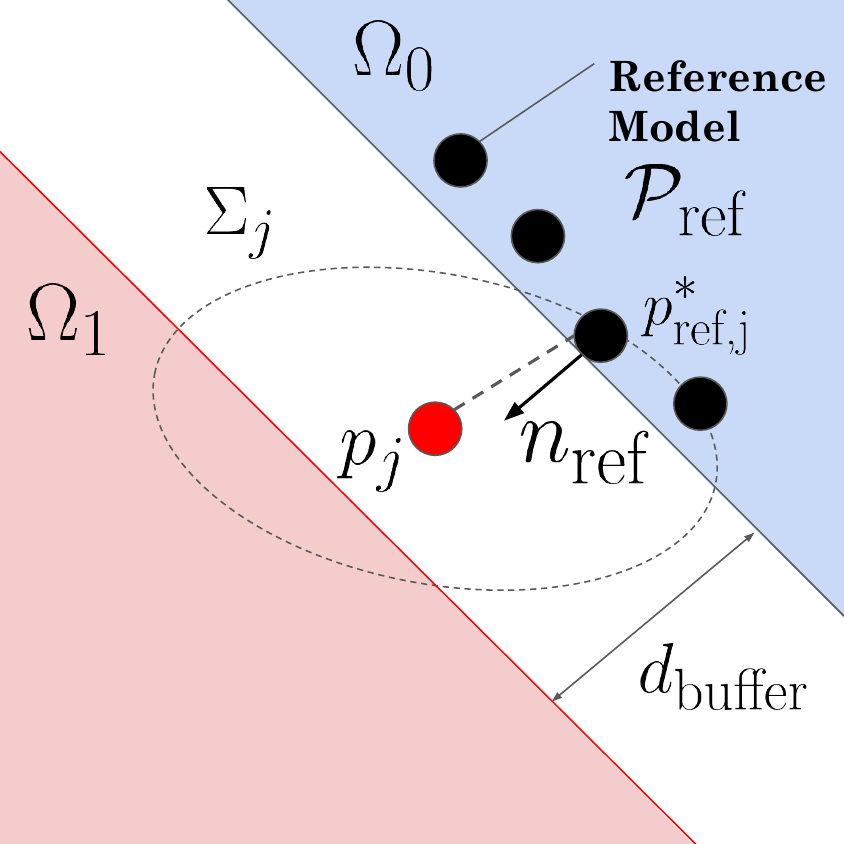}}}
        \caption{}
    \end{subfigure}
    %}
    \caption{SLAM-based anomaly detection. (a) shows an illustration of a local point cloud \(\mathcal{P}_\text{local}\) of an inspection pose node. \(T_i\) is the \(4\times4\) rigid transformation of the pose node \(i\), \(\Sigma_i\) is the associated covariance matrix, and the small circles, \(p_j\), are the point cloud samples collected by the depth camera. (b) illustrates the associated FOD detection on a single point \(p_j \in \mathcal{P}_\text{local}\) with the uncertainty \(\Sigma_j\) propagated from (a). The blue area, \(\Omega_0\), is the half space wherein the point would be considered to be inside the reference mesh; and the red area, \(\Omega_1\) is the half space wherein the point would be considered a foreign object. The black points are from the reference model, \(p_{\text{ref}}\) is statistically closest point to \(p_j\) with normal vector \(n_{\text{ref}}\), and \(d_\text{buffer}\) is the buffer zone threshold. The likelihood of the point \(p_j\) being sampled from the blue and red areas are obtained by integrating the Gaussian \(\mathcal{N}(p_j, \Sigma_j)\) over the corresponding halfspaces.  }
    \label{fig:anomaly_detect}
\end{figure}
Consequently, the likelihood of \(p_j\) being sampled from the plane formed by \((p_{\text{ref},j}^*, n^*_{\text{ref},j})\) or behind with respect to the normal vector, \(n_\text{ref}\), is calculated by the half-space integral 
\begin{equation}
\label{halfspaceint}
\begin{aligned}
\mathbb{P}(p_j | H_0) &= \int_{\Omega_0} \mathcal{N}(p_j, \Sigma_{j})  dx \\
 \Omega_0 & \triangleq \{x|n_{\text{ref}}^Tx -n_\text{ref}^Tp_{\text{ref}}\leq 0 \}.
\end{aligned}
\end{equation}
This is calculated by the normal cumulative distribution function (cdf) as
\begin{equation}
\label{halfspaceintsol}
\begin{aligned}
\mathbb{P}(p_j | H_0) & = 1-\text{normcdf}(c_{0,j})\\
\end{aligned}
\end{equation}
where 
\(c_{0,j}\) is the signed Mahalanobis distance 
\begin{equation}
\begin{aligned}
c_{0,j} & \triangleq \frac{(n_{\text{ref},j}^*)^T(p_{\text{ref},j}^*-p_j)}{\sqrt{(n_{\text{ref},j}^*)^T\Sigma_{j}n_{\text{ref},j}^*}}.
\end{aligned}
\end{equation}

\(\mathbb{P}(p_j | H_1)\) is calculated similarly by offsetting \(p^*_{\text{ref}}\) by \(d_\text{buffer}\) and integrating from the opposite halfspace by \(\mathbb{P}(p_j | H_1) = \text{normcdf}(c_{1,j})\), with 
\begin{equation}
    \label{halfspaceintsol}
    \begin{aligned}
    c_{1,j} & \triangleq c_{0,j} - \frac{d_\text{buffer}}{\sqrt{(n_{\text{ref},j}^*)^T\Sigma_{j}n_{\text{ref},j}^*}}.
    \end{aligned}
\end{equation}

Since the signed Mahalanobis distance follows the standard normal distribution, assuming independence, the average of \(c_j\) follows a normal distribution with
\begin{equation}
\label{sum_c}
\begin{aligned}
\bar{c} = \frac{1}{n} \sum_{j=1}^n c_j, \quad \mathbb{E}(\bar{c}) =0 ,\quad \sigma(\bar{c}) = \frac{1}{\sqrt{n}}
\end{aligned}
\end{equation}
where \(\sigma(\cdot)\) denotes the standard deviation. This forms a new z-score
\begin{equation}
\begin{aligned}
z = \frac{\bar{c}- \mathbb{E}(\bar{c})}{\sigma(\bar{c})} = \frac{1}{\sqrt{n}} \sum_{j=1}^n c_j.
\end{aligned}
\end{equation}
This means that instead of performing the integral one point at a time, the points can be batch processed for each reference point. For each point \(p_{\text{ref},i}\) in the reference model \(\mathcal{P}_{\text{ref}}\), find the set of observed points, \(\mathcal{P}_{\text{sample}, i}\), which corresponds to itself. Additionally, for smoothing purposes, consider the point set of the neighboring reference points given by  
\begin{equation}
\begin{aligned}
\mathcal{P}_{\text{sample}, i} = \{p_j | p^*_{\text{ref}, j} \in \mathcal{P}_{\text{NN}, i}\}
\end{aligned}
\end{equation}
where \(\mathcal{P}_{\text{NN}, i}\) is the set of $k$-nearest neighbor points in \(\mathcal{P}_{\text{ref}}\) of the $i$-th reference point \(p_{\text{ref,i}}\). 

Putting everything together, the likelihood for all the points in the reference model is calculated as
\begin{equation}
\begin{aligned}
\mathbb{P}_i(\mathcal{P}_{\text{sample}, i}|H_0) = 1-\text{normcdf}(c_{0,j})\\
z_{0,i} = \frac{1}{\sqrt{n}}\sum_{\{j | p_j \in \mathcal{P}_{\text{sample}, i}\}} c_{0,j}
\end{aligned}
\label{eq:null_likelihood}
\end{equation}

\begin{equation}
\begin{aligned}
\mathbb{P}_i(\mathcal{P}_{\text{sample}, i}|H_1) = \text{normcdf}(c_{1,j})\\
z_{1,i} = \frac{1}{\sqrt{n}}\sum_{\{j | p_j \in \mathcal{P}_{\text{sample}, i}\}} c_{1,j}.
\end{aligned}
\label{eq:alt_likelihood}
\end{equation}

Finally, the posterior of the hypothesis for each reference point is updated using Bayes law as
\begin{equation}
\label{eq_bayesian_detect}
\begin{aligned}
\mathbb{P}_i(H_0|\mathcal{P}_{\text{sample}, i}) = \zeta\mathbb{P}(\mathcal{P}_{\text{sample}, i}|H_0) \mathbb{P}_i(H_0) \\
\mathbb{P}_i(H_1|\mathcal{P}_{\text{sample}, i}) = \zeta\mathbb{P}(\mathcal{P}_{\text{sample}, i}|H_1) \mathbb{P}_i(H_1)\\
\zeta \triangleq \frac{1}{\mathbb{P}_i(H_0|\mathcal{P}_{\text{sample}, i}) + \mathbb{P}_i(H_1|\mathcal{P}_{\text{sample}, i})}.
\end{aligned}
\end{equation}

\subsection{Information Measure}
The region measure, \(\mu_r\), for generating the space average, \(\bar{\rho}\), in \eqref{space_average} is defined using the posterior of the hypothesis. 
In particular, we define information measure using the entropy of the posterior for each reference point as
\begin{equation}
\begin{aligned}
\mu_p(p_{\text{ref},i}) = -\mathbb{P}_i(H_0)\ln(\mathbb{P}_i(H_0)) - \mathbb{P}_i(H_1)\ln(\mathbb{P}_i(H_1)).
\end{aligned}
\end{equation}
The region measures are calculated by adding up the entropies of the reference points belonging to the region,
\begin{equation}
\begin{aligned}
 \mu_r(i) = \sum_{p_{\text{ref},j}\in r_i} \mu_p(p_{\text{ref},j}) .
\end{aligned}
\end{equation}
Conceptually, the reference point entropy measures the expected surprise of sampling from the reference point \(p_{ref,i}\), and the region measure is the maximum amount of surprise the robot can get by sampling from the region. Combining this with the REMC, the robot travels to the region with a high probability of being frequently surprised.   

\section{Hierarchical Ergodic Inspection Algorithm}
\label{section_algorithm}
This section explains how graph ergodic planning and information measure are used within the hierarchical planning framework (as shown in Fig. \ref{fig:h_framework}), which we call \underline{h}ierarchical \underline{e}rgodic \underline{Ma}rkov \underline{p}lanner (HEMaP), to inspect a confined structure such as a ballast tank. The complete inspection planning algorithm is described in Algorithm \ref{alg:remc}. The algorithm requires the initial region of the robot \(r[0]\); the region graph \(\mathcal{G}\); the replanning horizon \(K\); the reference point cloud \(\mathcal{P}_\text{ref}\); and the initial anomaly entropy \(\mu_p\). 
%With all the previous information, 
The robot initializes the graph-level time step \(k\) to 0. Then, for every \(K\) steps, the robot calculates the region-based measure \(\mu_r\) by summing the anomaly entropy of the structures in each region. Using the new region-based measure \(\mu_r\), the robot calculates the optimal stochastic transition matrix \(P\) according to the REMC optimization problem in \eqref{eq_remc}. Additionally, to prevent an ill-conditioned target distribution with 0-entries, a smoothing factor \(\delta\) is added before optimization. The robot samples from this \(P\) for the next \(K\) steps. After the next region is sampled, the robot picks a waypoint inside the next region according to the structure anomaly entropy in the region (described in the next algorithm), navigates 
%and samples the depth camera from 
to the waypoint, and updates the anomaly detection via \eqref{eq_bayesian_detect}. This process is repeated until the inspection is finished, either defined by a maximum number of steps or terminated by a human. 

\begin{algorithm}[!thpb]
\caption{HEMaP algorithm}\label{alg:remc}
\begin{algorithmic}[1]
\State \textbf{Parameter:} number of waypoints per region: \(n_\text{waypoints}\)
\State \textbf{Inputs:} initial region: \(r[0]\), region graph: \(\mathcal{G}\), time horizon: \(K\), reference point cloud: \(\mathcal{P}_\text{ref}\), information measure: \(\mu_p\)
\State \textbf{Initialize} \(k \gets 0\)
\While{Inspecting}
\If {\(k \mod K = 0\)} 
\State \textbf{Initialize} \(\mu[n]\)
\For{\(r_i \in \mathcal{R}\)}  
\For{$ p_{ref} \in  r_i$}
    \State \(\mu_r[i] \gets \mu_r[i] + \mu_p(p_{ref})\)
\EndFor
\EndFor
\State \(\bar{\rho} = \mu_r/\sum\mu_r\)
\State \(\bar{\rho} = (\bar{\rho}+\delta)/\sum(\bar{\rho}+\delta)\)
\State \(P \gets \text{REMC}(\bar{\rho})\) \Comment{\eqref{eq_remc}}
\EndIf
\State \(r[k+1] \sim  \mathbb{P}(R_{k+1}|R_{k} = r[k])\)
\For {\(i\) in \(n_\text{waypoint}\)}
\State \(x \gets\) waypoint\_placement(\(r[k+1], \mu_p \))
\State \(\text{navigate\_to}(x)\) 
\State \(\mu_p \gets \text{anomaly\_detection}(\mu_p )\) \Comment{\ref{section:detection}}
\EndFor
\State \(k \gets k+1\)
\EndWhile
\end{algorithmic}
\end{algorithm}

\begin{algorithm}[!thpb]
\caption{waypoint\_placement}\label{alg:waypoint}
\begin{algorithmic}[1]
\State \textbf{Parameter:} number of samples: \(n_{\text{sample}}\)
\State \textbf{Inputs:} region: \(r\), feature information: \(\mu_p\)
\State \textbf{Initialize} \(x[n_{\text{sample}}], \mu_x[n_{\text{sample}}]\)
\For {\(i\) in  \(n_{\text{sample}}\)} 
\State \(x[i] \sim \text{uniform}(r)\) \Comment{uniformly sample a pose from the region}
\For{\(p_\text{ref}\in \mathcal{P}_\text{ref}\)} 
\If{\(p_\text{ref} \in \text{camera\_frustum}(x[i])\)}
\State \(\mu_x[i] \gets \mu_x[i] + \mu_p(p_\text{ref})\)
\EndIf
\EndFor
\EndFor
\State \(\bar{\mu}_x \gets \mu_x/\sum \mu_x\) 
\State \(x^* \gets \text{random\_choice}(x, p = \bar{\mu}_x)\)
\State \textbf{Return} \(x^*\)
\end{algorithmic}
\end{algorithm}

Algorithm \ref{alg:waypoint} shows the procedure for choosing the waypoint in the designated region from the ergodic planner. A sampling algorithm is used to select the waypoint, where the number of candidate waypoints, \(n_\text{sample}\), is determined based on the size of the region and the computing power of the robot. The robot samples \(n_\text{sample}\) waypoints from the designated region \(r\). For each candidate waypoint, it then finds the reference points that fall within the robot's camera frustum, defined by the camera field of view and the minimum and maximum depths. The entropies of all the anomalies within the frustum are added together and stored as the waypoint measure. The final waypoint is then selected by randomly sampling from this candidate set, with the selection probability being proportional to the relative measure of each waypoint. 

\section{Experiments}

In this section, four sets of experiments, with progressively increasing complexity, are conducted to demonstrate the advantages of REMC and HEMaP at various abstraction levels: 1) Graph world, 2) Grid world, 3) Gazebo simulation, 4) Real-world experiment. The first experiment shows the fast convergence to ergodicity of REMC by testing on a graph world. The second experiment applies HEMaP on a grid representation of the ballast tank and shows that the graph method can be applied to a continuous space partitioned into interconnected regions. In the third experiment, the method is applied in a Gazebo simulation of the ballast tank to study the effectiveness of ergodic graph traversal for FOD detection. In the final experiment, the inspection algorithm is implemented on a TurtleBot in a custom-built mock ballast tank to demonstrate real-world deployment capability. 

\subsection{Graph World Experiment}
The following experiment compares the performance between the ergodicity of the Markov chain generated by FMMC and REMC from Theorem~\ref{thm_optimalbound} using the region graph in Fig. \ref{h_graph}. 

\subsubsection{Methodology}
The target distribution \(\bar{\rho}\) and initial distribution \(\rho_0\) are generated by uniformly sampling the weight for each region from \(0-1\) and normalizing them so that they add up to \(1\). The expected ergodicity deviation at each time step, \(\|\mathbb{E}[\hat{\rho}_k] - \bar{\rho} \|_2\), is calculated both for FMMC and REMC, where 
\begin{equation}
   \mathbb{E}[\hat{\rho}_k] = \frac{1}{K}\sum_{k=0}^{K-1} P^k\rho_0 .
\end{equation}  

\begin{figure}[thpb]
      \centering
       \includegraphics[width=0.35\textwidth]{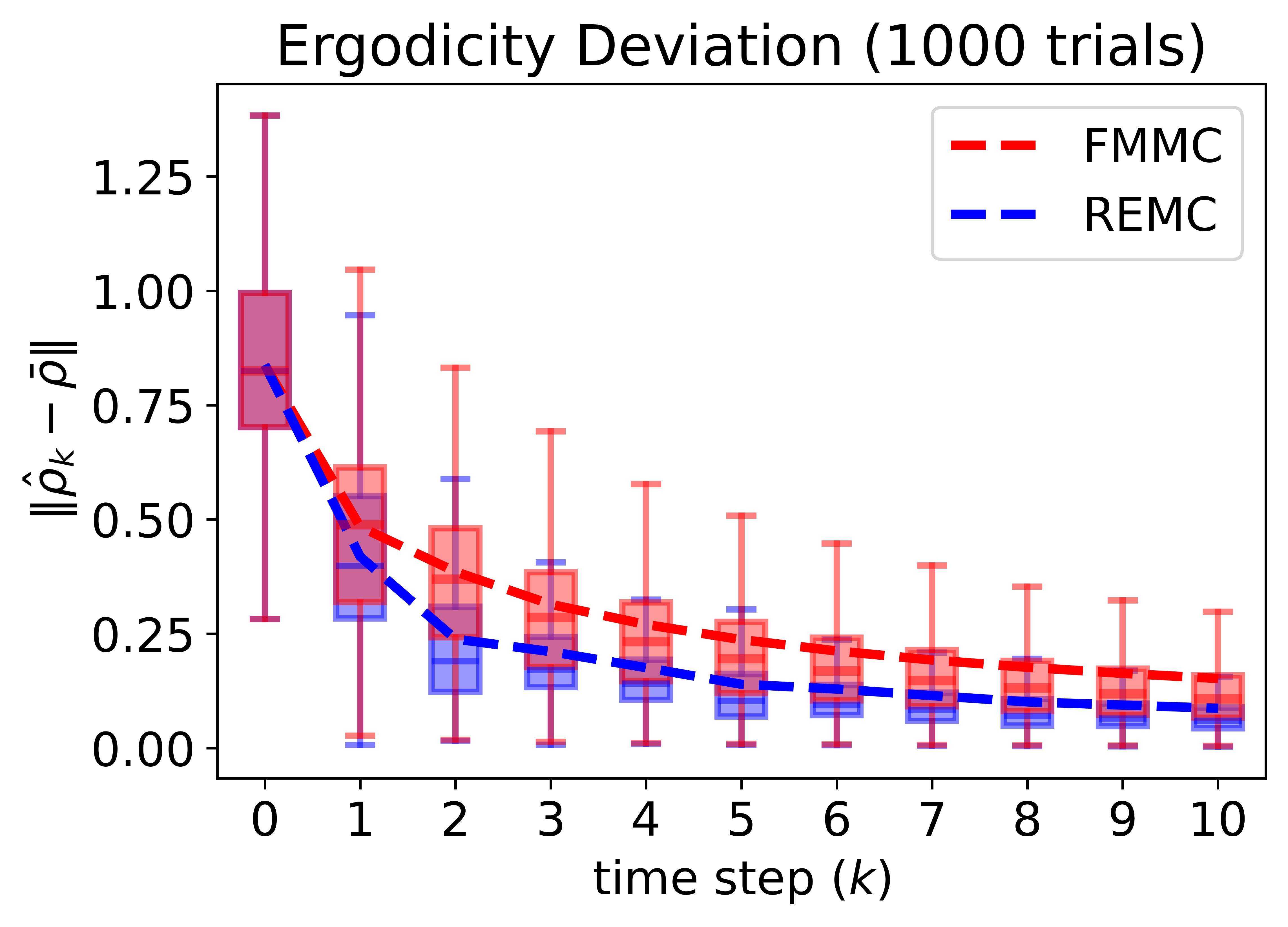}
      \caption{Ergodicity comparison between fastest mixing Markov chain (FMMC) and REMC from Theorem~\ref{thm_optimalbound} tested on the region graph from Fig. \ref{h_graph}, showing that REMC has a lower ergodicity error with a small number of steps on average. 
     }
      \label{fig_fmmc_vs_ergodic}
   \end{figure}

\subsubsection{Results} 
Fig. \ref{fig_fmmc_vs_ergodic} shows the expected ergodicity deviations in 10 time steps for FMMC and REMC separately for 1,000 trials. The result shows that on average the REMC method attains ergodicity faster than FMMC. Additionally, it has a substantially lower worst case deviation (shown by the top blue whiskers) than FMMC. Fig. \ref{fig_fmmc_vs_optimal_pairdiff} shows the pairwise difference (FMMC - REMC) between the two methods for the same trials. The result shows that FMMC is consistently less ergodic, with the first quartiles larger than 0 at all the time steps. 

   \begin{figure}[thpb]
      \centering
    \includegraphics[width=0.35\textwidth]{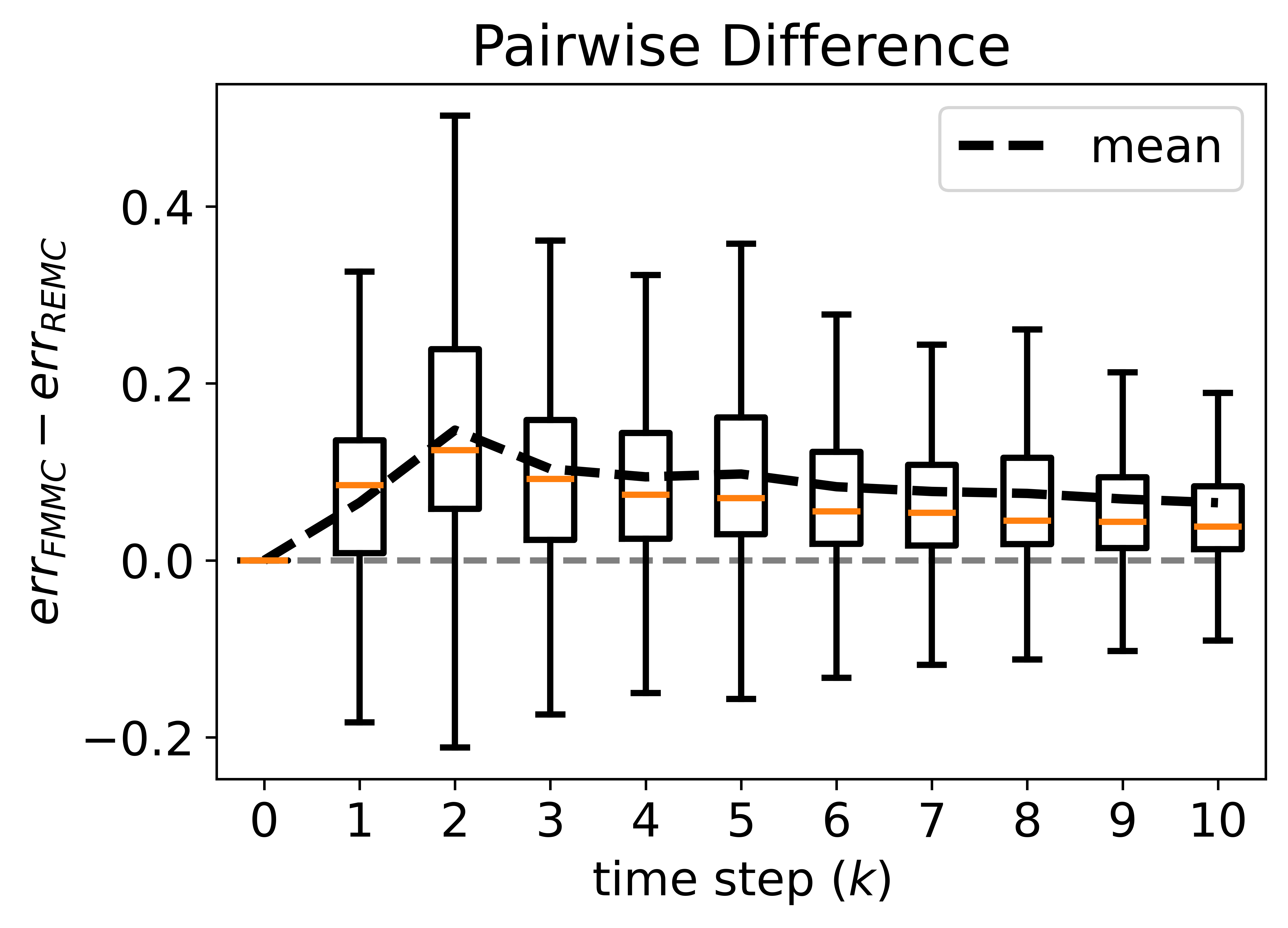}
      \caption{Boxplot of the pair-wise error difference from each trial calculated from the data from Fig. \ref{fig_fmmc_vs_ergodic}. For the same initial condition and target distribution, our method consistently has less error on an average as compared to FMMC.}
      \label{fig_fmmc_vs_optimal_pairdiff}
   \end{figure}

\subsection{Grid World Experiment }
\label{section:grid}
The following experiments shows the comparison of region-based ergodicity between HEMaP and three other continuous space ergodic planners: spectral multiscale coverage (SMC) \cite{MathewGeorge2011Mfea}; heat equation driven area coverage (HEDAC) \cite{IvicStefan2017ECMA}; and graph-based ergodic search in cluttered environments (GESCE) \cite{ShiroseBurhanuddin2024GGES}, in the occupancy grid of the ballast tank shown in Fig. \ref{fig:ballast_cad}. 

\subsubsection{Methodology} 
The occupancy grid map is created by projecting the ballast tank CAD model onto the ground, with a lower and upper height limit of \(0.05\) m and \(0.2\) m, respectively. Additionally, the occupied grids are inflated with the radius of a TurtleBot Burger. The robot is modeled as a point mass moving at a constant speed of 0.2 m/s. Specifically, for SMC, the robot follows a single integrator model
\begin{equation}
     \begin{bmatrix}
 x_{t+1}\\ y_{t+1}
\end{bmatrix} = \begin{bmatrix}
 x_{t}\\ y_{t}
\end{bmatrix} + \Delta t\begin{bmatrix}
 v_x\\ v_y
\end{bmatrix} 
\end{equation} 
with \(\sqrt{v_x^2+v_y^2} = 0.2\). For HEMaP, the robot is allowed to move to any of the 8 neighboring grid cells centered at the current grid location within the time interval \(\Delta t\). 
For HEDAC, the robot follows second-order dynamics according to this GitLab implementation\footnote{\url{https://gitlab.idiap.ch/rli/robotics-codes-from-scratch/-/blob/master/python/ergodic_control_HEDAC_2D.py?ref_type=heads}}.
For GESCE, the robot travels directly according to a probabilistic road map (PRM) planner, i.e., it follows a straight line from one collision-free node to the next without any low-level dynamical constraints.   

The target distribution is uniform for each region, but varies among the regions, as shown in Fig.~\ref{fig:HEMaP_examples}. To account for collision avoidance, grid cells belonging to obstacles are considered an extra region, with a target visitation frequency of \(0\). For SMC, 20 Fourier coefficients are used with a time step of 0.01 s. For HEMaP, the waypoints are selected uniformly randomly in the target region indicated by the graph planner. A* is chosen as the path planning algorithm to navigate between the waypoints. For GESCE, 200 nodes are uniformly randomly sampled from the free space for the PRM with any pairs of nodes within 0.6 m connected by an edge. The cost function is identical to the SMC implementation. A total of 500 iterations are allowed during the search cycle to maintain a reasonable runtime. The robot is allowed to run for a total of 1,200 s (except for GESCE), and a total of 30 trials are performed for each method. 

\subsubsection{Results} 
\begin{figure*}[thpb!]
    \centering
    %\fbox{
     \begin{subfigure}[t]{0.3\textwidth}
        \centering
        \includegraphics[width=0.95\textwidth]{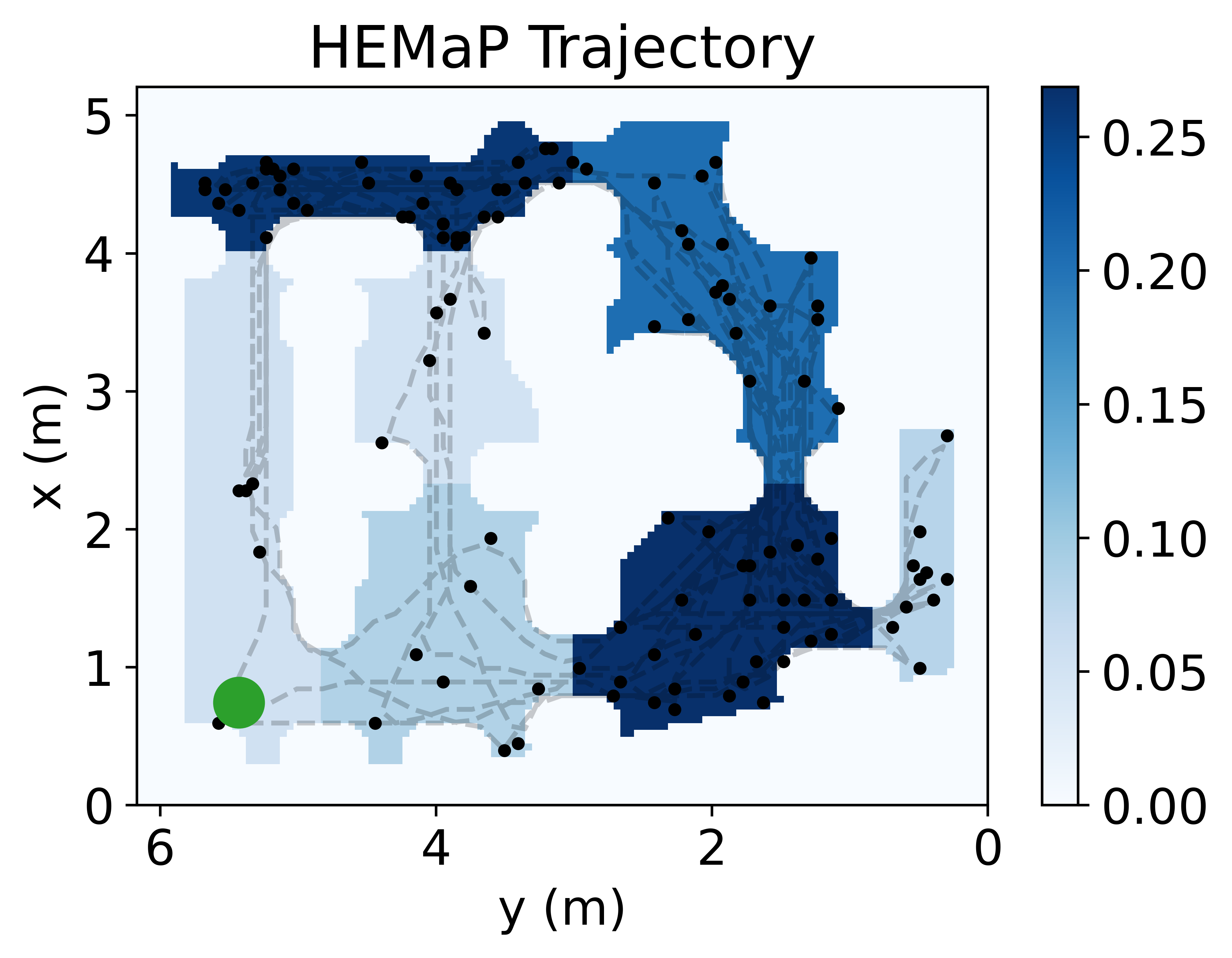}
        \caption{}
    \end{subfigure}
    ~
    \begin{subfigure}[t]{0.3\textwidth}
        \centering
        \includegraphics[width=0.95\textwidth]{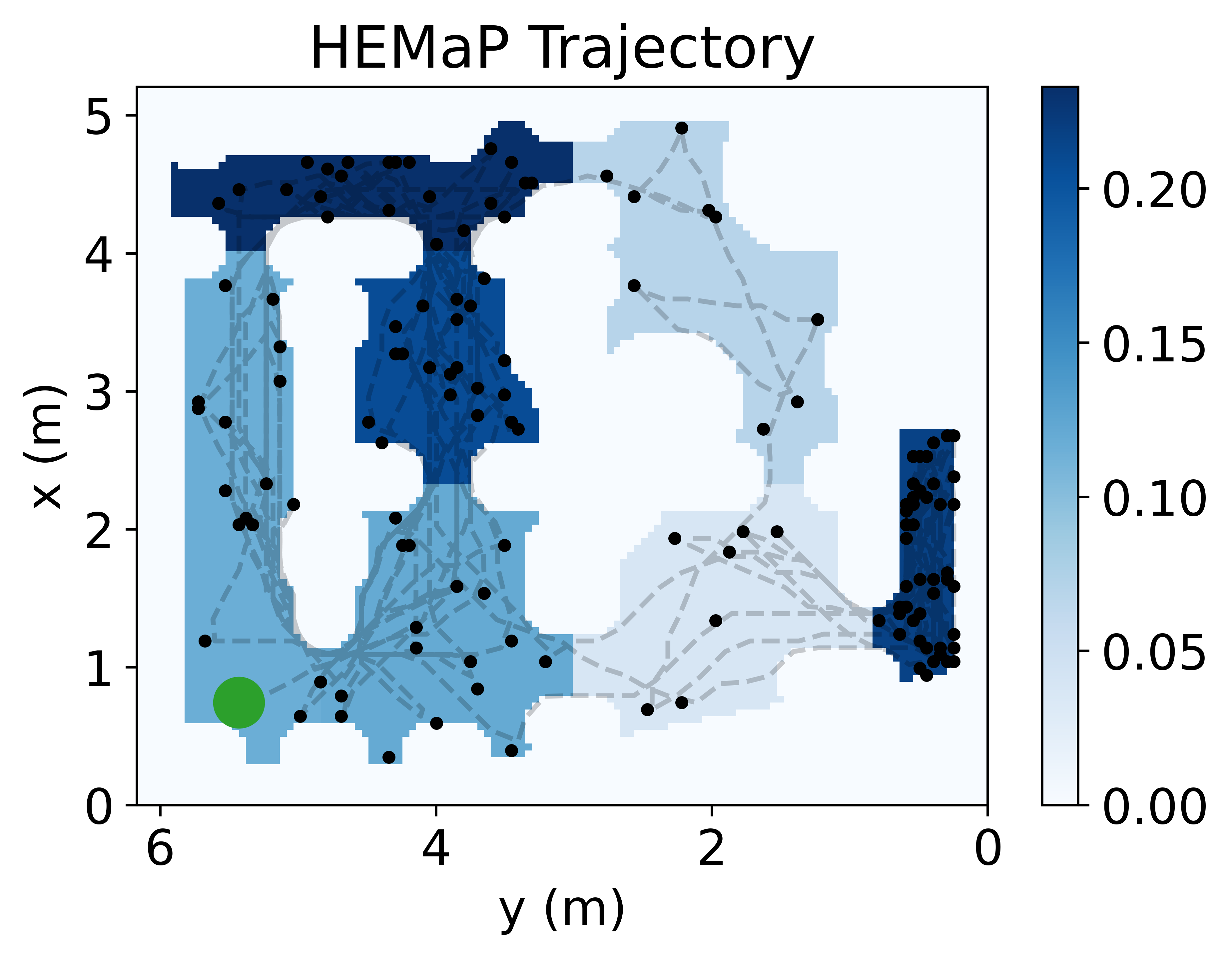}
        \caption{}
    \end{subfigure}%
    ~ 
    \begin{subfigure}[t]{0.3\textwidth}
        \centering
        \includegraphics[width=0.95\textwidth]{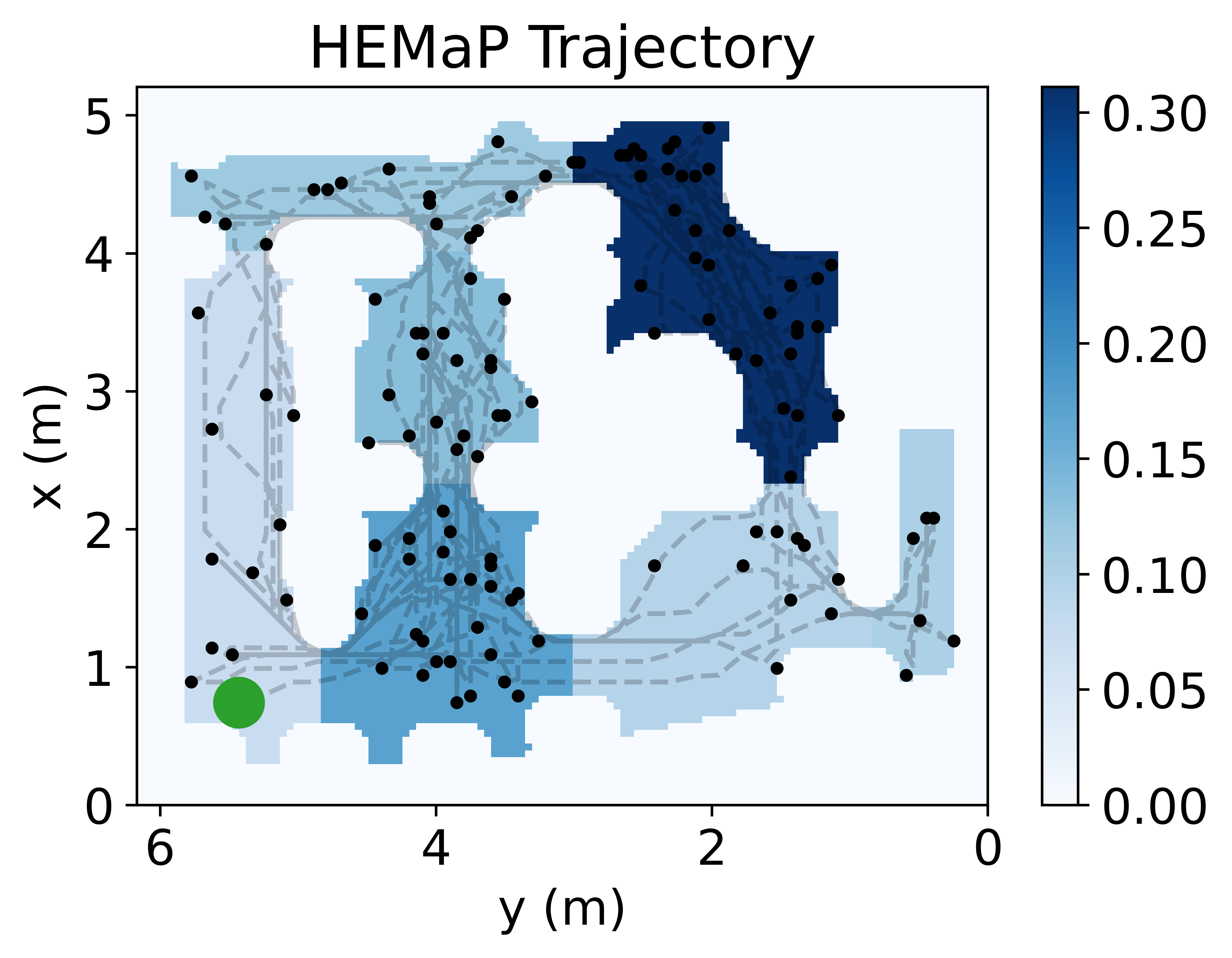}
        \caption{}
    \end{subfigure}
    %}
    \caption{Example trajectories generated by HEMaP with A* as the low level planner and uniformly random waypoint placement. This shows that ergodic planning can adapt to various spatial distributions \(\bar\rho\), where the number of sampled points (shown as black circles) is proportional to the desired region visitation frequency (shown as shade of blue). }
      \label{fig:HEMaP_examples}
\end{figure*}

 \begin{figure}[thpb!]
    \centering
      %\fbox{ 
    \begin{subfigure}[t]{0.25\textwidth}
        \centering
        \includegraphics[width=0.95\textwidth]{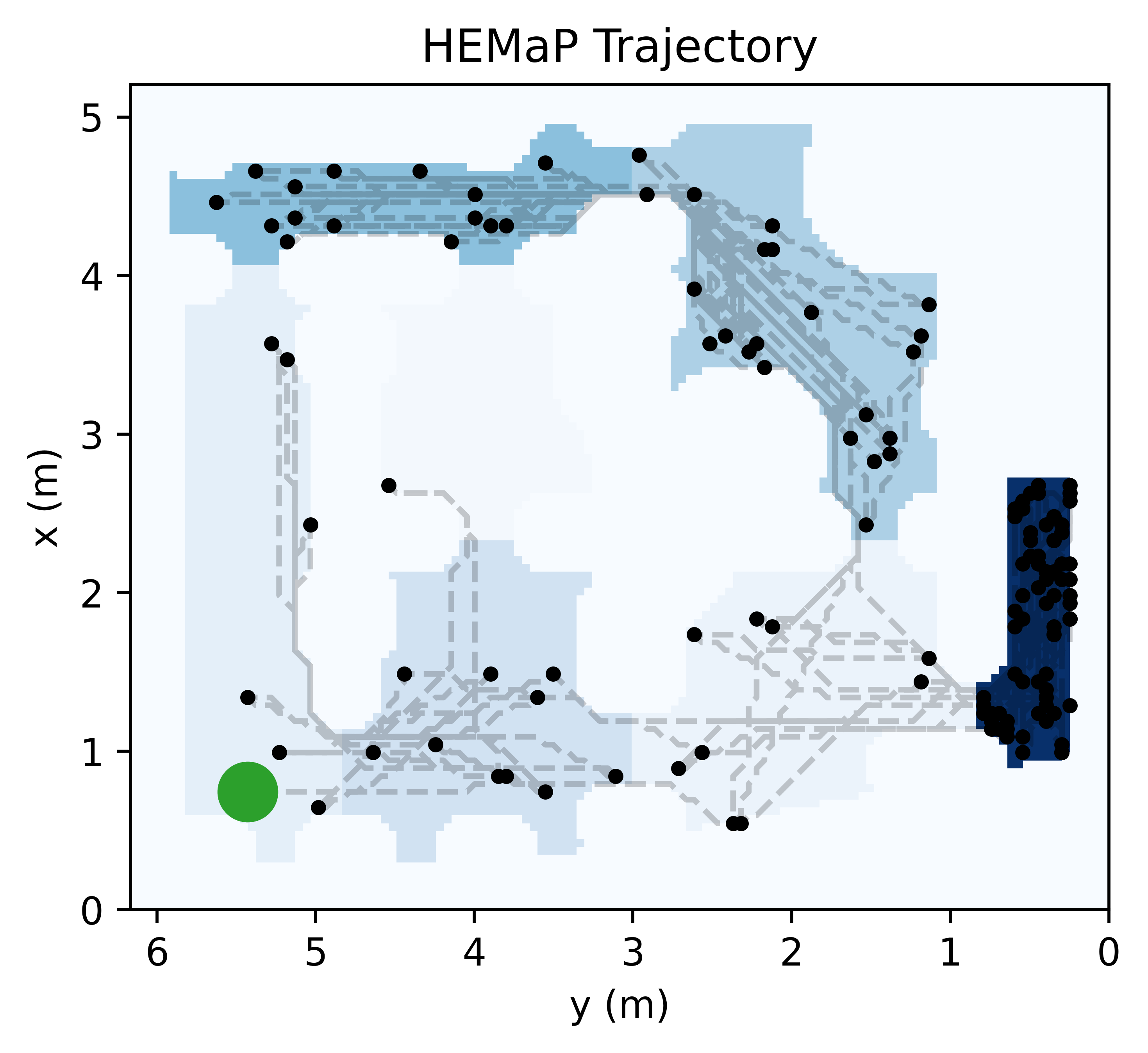}
        \caption{HEMaP}
    \end{subfigure}%
    ~ 
    \begin{subfigure}[t]{0.25\textwidth}
        \centering
        \includegraphics[width=0.95\textwidth]{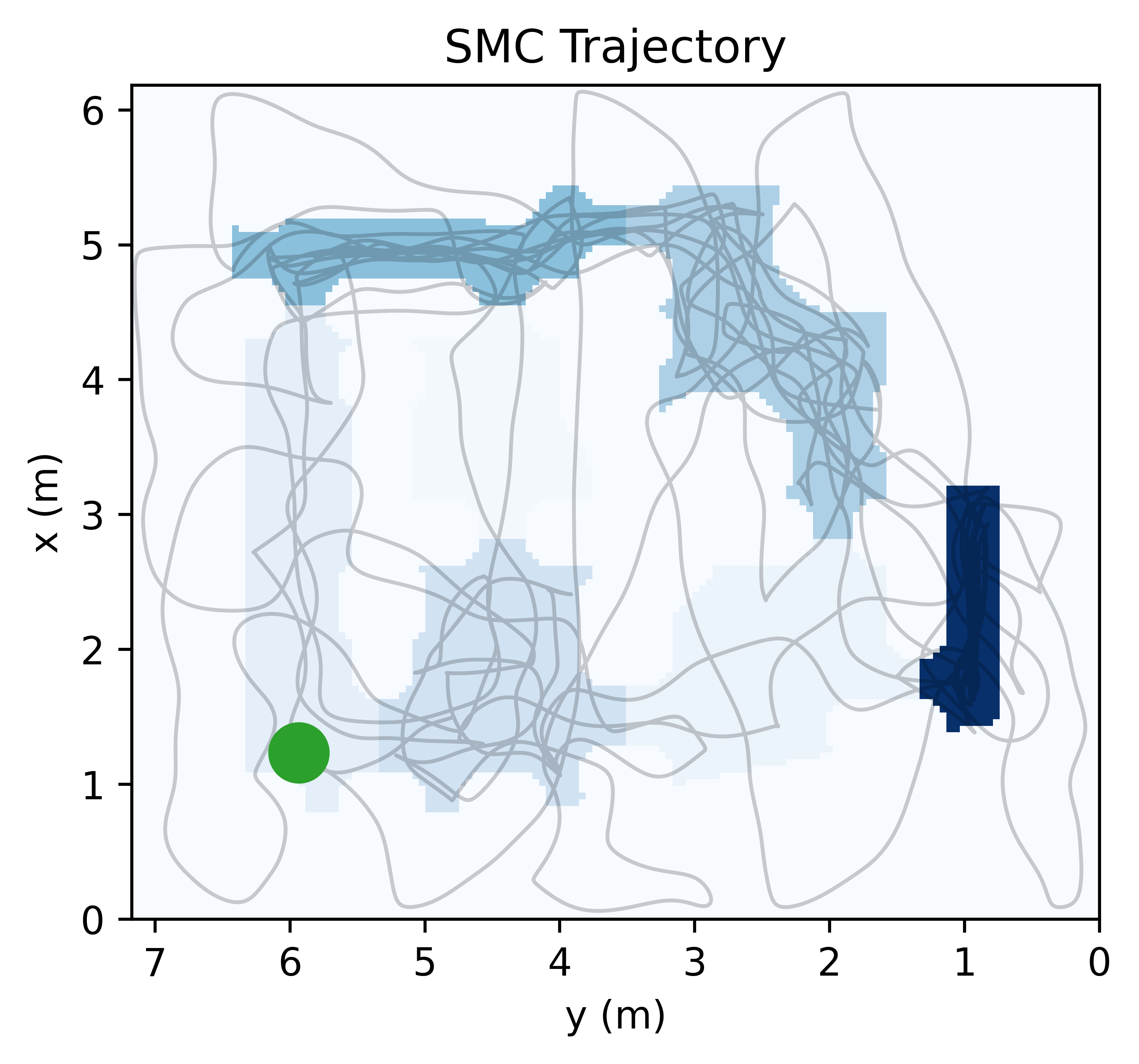}
        \caption{SMC}
    \end{subfigure}
    %}
    \caption{Example ergodic trajectories comparison for HEMaP and SMC, with the darker blue shade indicating higher target frequency. The green circles represent the initial location; the black dots in (a) represent the inspection waypoints and dashed gray line shows intermittent trajectory; the gray line in (b) shows continuous trajectory. 
    %It can be seen that 
    SMC yields more complex trajectories that enter the obstacle region even if the desired distribution is set to 0.}
         \label{fig_example_traj}
\end{figure}

A few representative robot inspection trajectories are shown in Fig. \ref{fig:HEMaP_examples}, which demonstrate that HEMaP works well regardless of the target distribution \(\bar\rho\) that varies according to the region-level anomaly entropy. To provide a contrast with the continuous space ergodic planner trajectories, Fig.~\ref{fig_example_traj} shows comparisons between HEMaP and SMC for the same choice of \(\bar\rho\). 
%As shown in Fig.~\ref{fig_example_traj}(a), the combination of graph planner and A* is able to generate a trajectory that 
While HEMaP trajectory is completely contained in the workspace, SMC cannot always avoid the obstacles as it considers a finite number of Fourier coefficients. 
%Even worse, 
The Gibbs phenomenon also causes the robot to actively run into obstacles, as seen by the long arcs stretching outside the workspace. In a similar vein, SMC does not have any knowledge of the topology of the space, which allows the robot to pass through obstacles and reach different regions directly. Additionally, the path generated by the A* planner in HEMaP is more straightforward for a ground robot to follow than the complex trajectory generated by SMC. 

Fig. \ref{fig_smc_vs_markov_plot} shows region-wise ergodic deviations at each time step for all the ergodic planners averaged over the complete set of trials. Fig. \ref{fig:pairwise_plot} plots the statistics on the pairwise differences between HEMaP and the continuous space planners on the same set of trials. For SMC and HEDAC, 
%since they are continuous space methods, 
the relative frequency is calculated based on the current location (region) of the robot, with obstacles treated as an extra region. For HEMaP, %since it is a discrete space planner, 
the relative frequency is calculated for each waypoint and the ergodicity deviation is held constant from the previous waypoint while the robot is in transition. Similarly, the relative frequency for GESCE is calculated for each node in the PRM graph and held constant during the traversal between the nodes.

   \begin{figure}[thpb]
      \centering
      \includegraphics[width=0.35\textwidth]{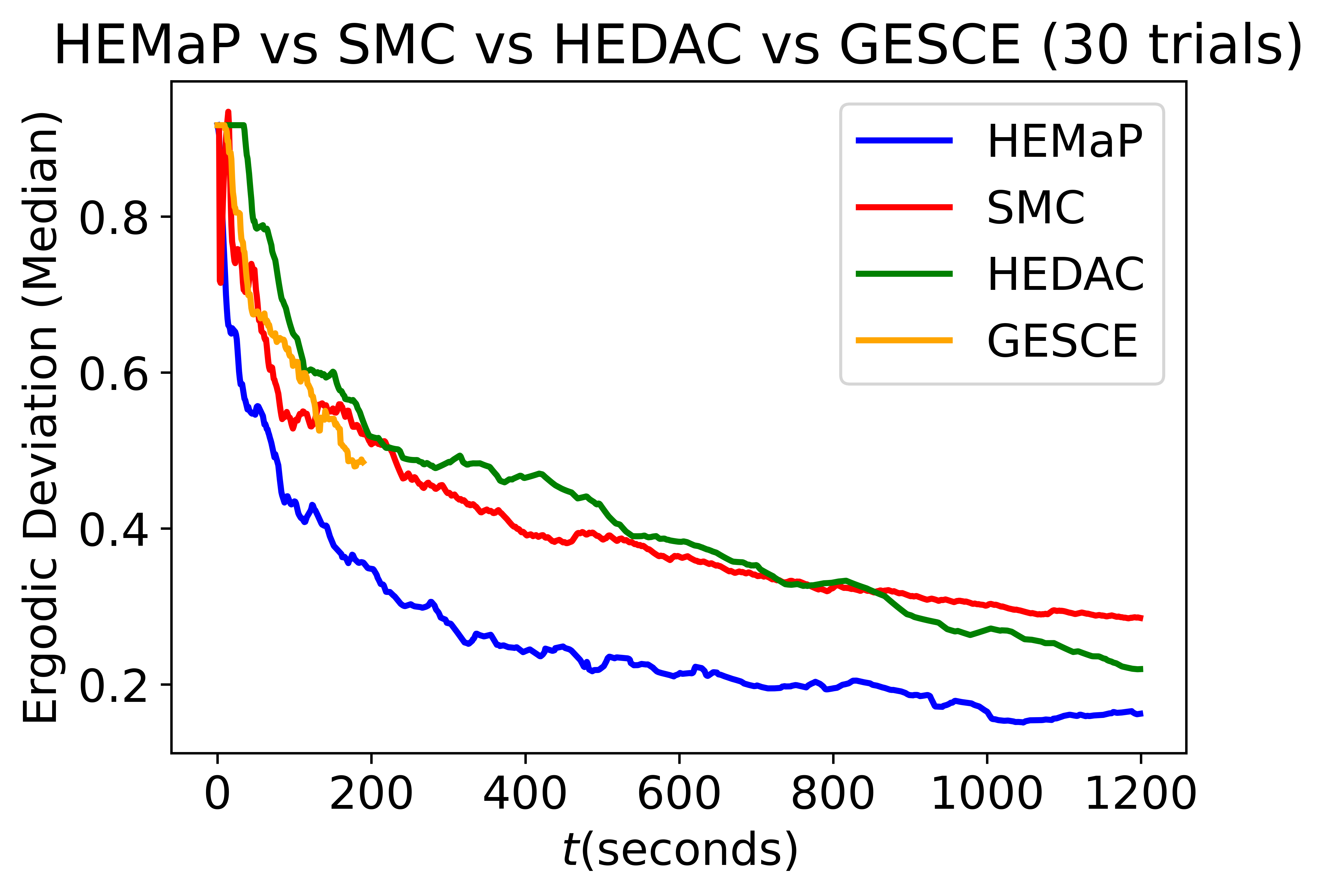}
      \caption{Comparison of region-based ergodicity between HEMaP and the various continuous space ergodic planners in the ballast tank (shown in Fig. \ref{h_graph}) with the robot running at 0.2 m/s for 1,200 seconds. HEMaP has better region-based ergodicity on average.}
      \label{fig_smc_vs_markov_plot}
   \end{figure}

\begin{figure*}[thpb!]
    \centering
    %\fbox{
     \begin{subfigure}[t]{0.3\textwidth}
        \centering
        \includegraphics[width=0.95\textwidth]{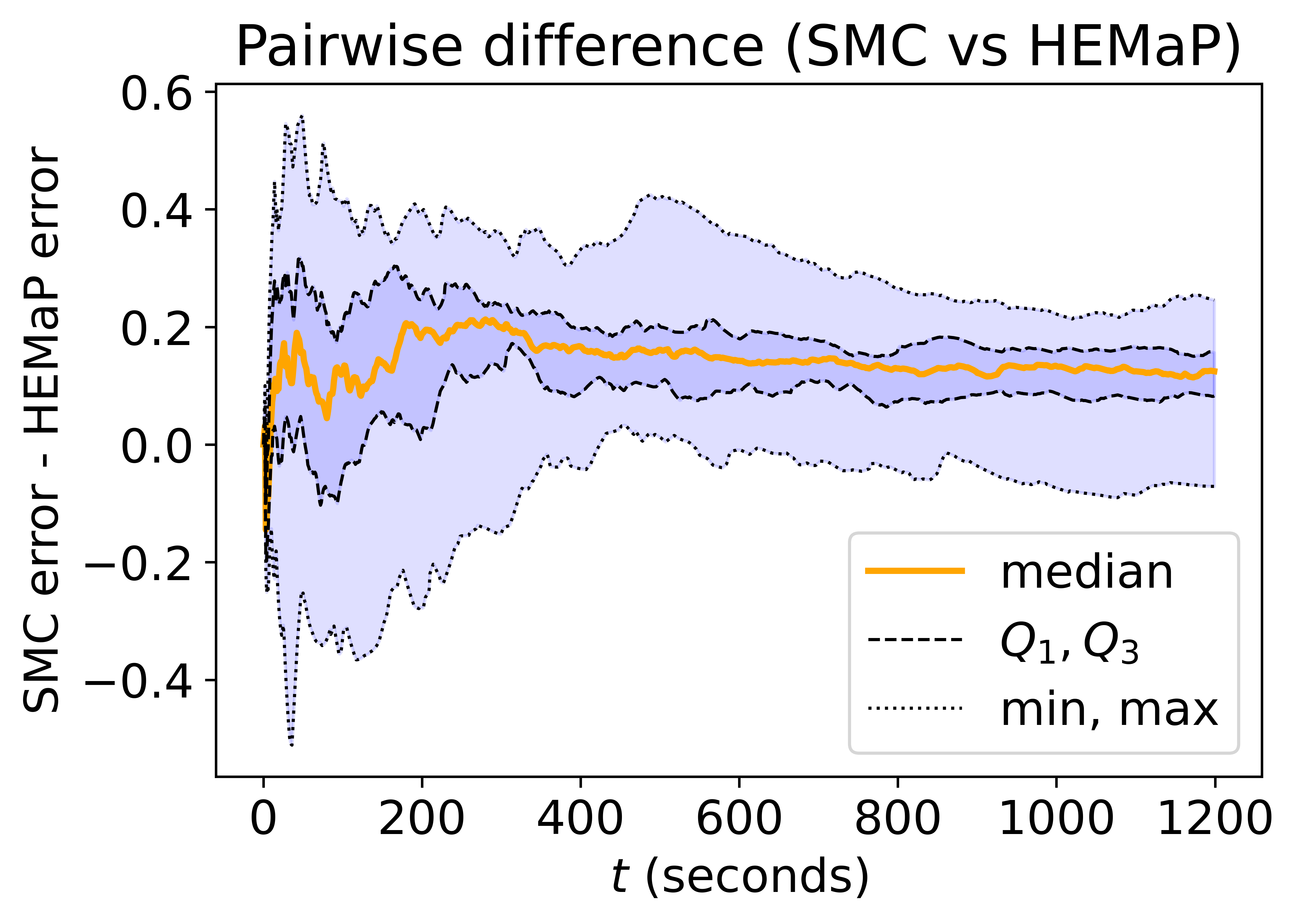}
        \caption{SMC}
    \end{subfigure}
    ~
    \begin{subfigure}[t]{0.3\textwidth}
    
        \centering
        \includegraphics[width=0.95\textwidth]{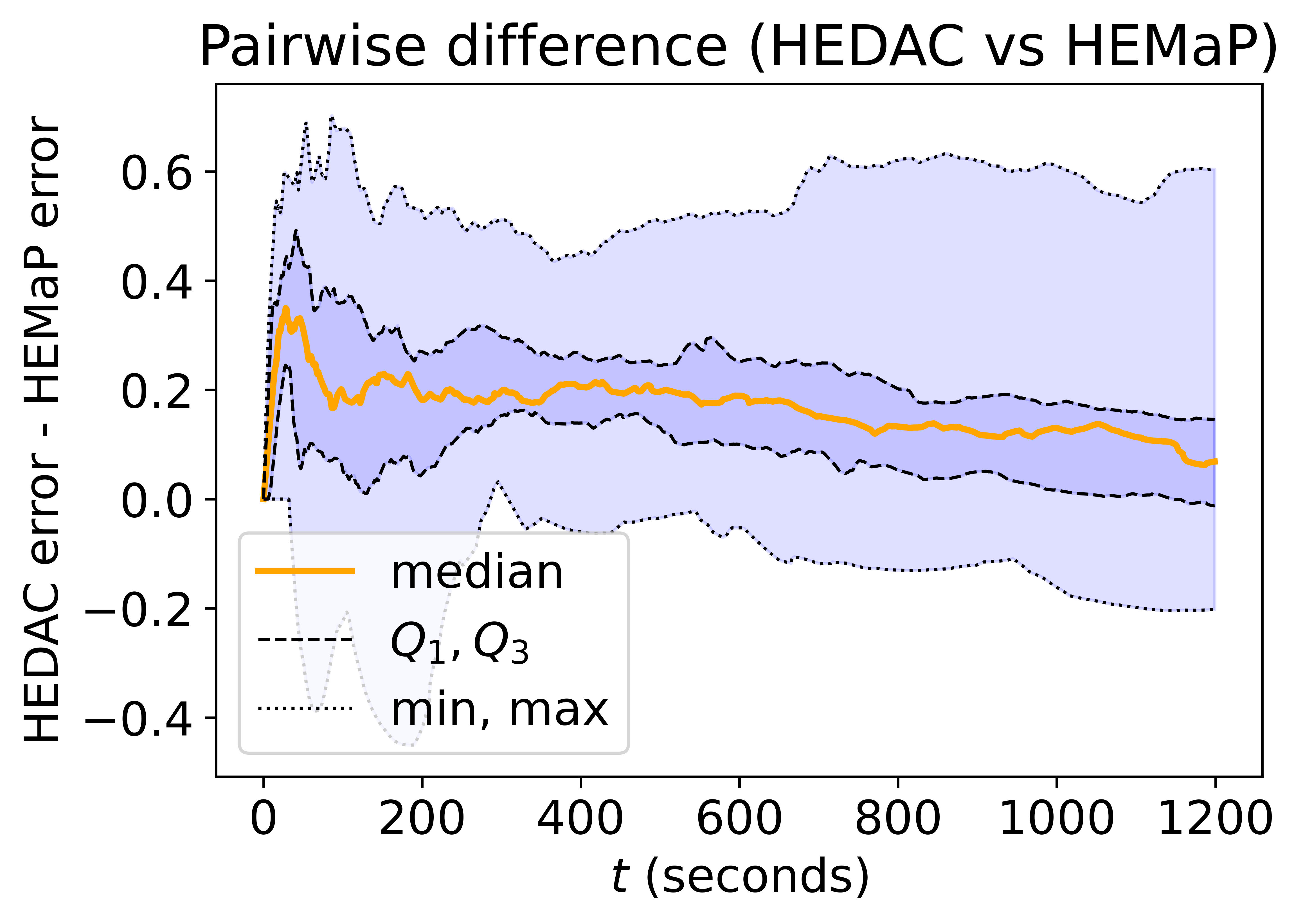}
        \caption{HEDAC}
    \end{subfigure}%
    ~ 
    \begin{subfigure}[t]{0.3\textwidth}
        \centering
        \includegraphics[width=0.95\textwidth]{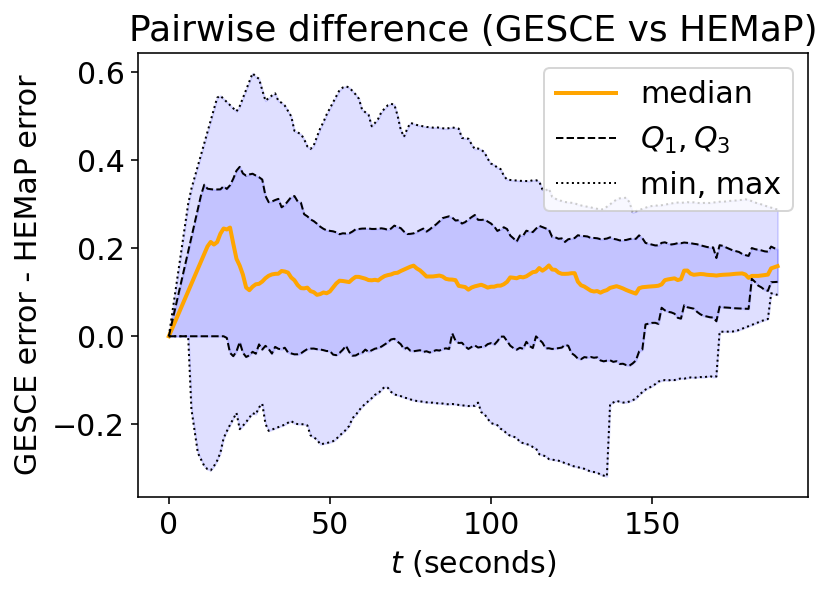}
        \caption{GESCE}
    \end{subfigure}
    %}
    \caption{Statistics on pair-wise differences in ergodicity deviations (errors) between HEMaP and the three continuous space planners over the same ser of trials used in Fig. \ref{fig_smc_vs_markov_plot}. (a) shows that for the same initial condition and target distribution, HEMaP consistently has less error on average compared to SMC, with the 3rd quartile above \(0\) after 200 seconds. (b) shows a similar trend for HEDAC, where HEDAC has a higher initial error but lower asymptotic error because of its collision avoidance capability. (c) shows GESCE has an almost identical performance to HEMaP in terms of the 1st quartile values (worse with respect to the median and 3rd quartile values though), but is unable to run for more than 200 seconds. } 
\label{fig:pairwise_plot}
\end{figure*}

The results show that HEMaP is able to reduce the region-based ergodicity deviation faster than all the other planners on average. This happens even though SMC does not consider the topology of the explored space, and can be explained by the fact that SMC is designed to handle rectangular domains that are partitioned into regular grid-like regions at varying scales using a Fourier basis. However, the ballast tank regions are irregular and non-rectangular with sharp changes in the desired distribution values at the region boundaries, which are hard to capture using a Fourier series. 
HEDAC, on the other hand, has a higher error than SMC initially due to obstacle avoidance. However, it is capable of approaching the performance of HEMaP at around 800 s, since it has significantly less collisions compared to SMC\footnote{Minor collisions, similar to Fig. 4 in \cite{IvicStefan2017ECMA}, can still be observed. This is likely due to discretized second-order dynamics, where the potential field does not produce sufficient acceleration for collision avoidance.}. GESCE has similar initial performance as SMC while maintaining obstacle avoidance, and improves substantially if the search is allowed to run longer. The main challenge is that the trajectory length cannot be specified, due to which none of the trials exceed the 200 second mark, while all the other methods can run indefinitely. 

In general, 
%based on the current implementation, 
the main advantage of HEMaP %compared to continuous space methods 
is that it discards fine-grained information in pursuit of \textit{rapid} convergence to the target visitation distribution of inspection regions. Therefore, it takes the shortest A* path between the regions, which enhances ergodicity at the region level. This is in contrast to all the continuous space methods that try to provide intra-regional coverage, which affects region-level ergodicity deviation and leads to fractal-like trajectories. If the fine-grained information are crucial and the robot is able to execute complex trajectories, then one can consider a continuous space method. Alternatively, since HEMaP is modular and hierarchical, we envision the possibility of replacing the A* planner with any compatible continuous space method, which is further discussed in Section \ref{section:discussion}.

\subsection{Gazebo Simulation}
\label{sec:gazebo_sim}
The following experiment compares the performance of two different graph traversal methods against HEMaP, in the context of FOD detection in the ballast tank of Fig. \ref{fig:ballast_cad}. Specifically, we consider random walk, where the planner samples from all the edges in each region uniformly; and a greedy method, where the planner chooses the edge that leads to the neighboring region with the highest estimated information measure. Additionally, the waypoint placements within the regions are changed accordingly, with random walk choosing the sampled waypoints randomly; and the greedy method choosing the sampled waypoints with the maximum estimated information. 

   \begin{figure}[thpb]
      \centering
      \includegraphics[width=0.35\textwidth]{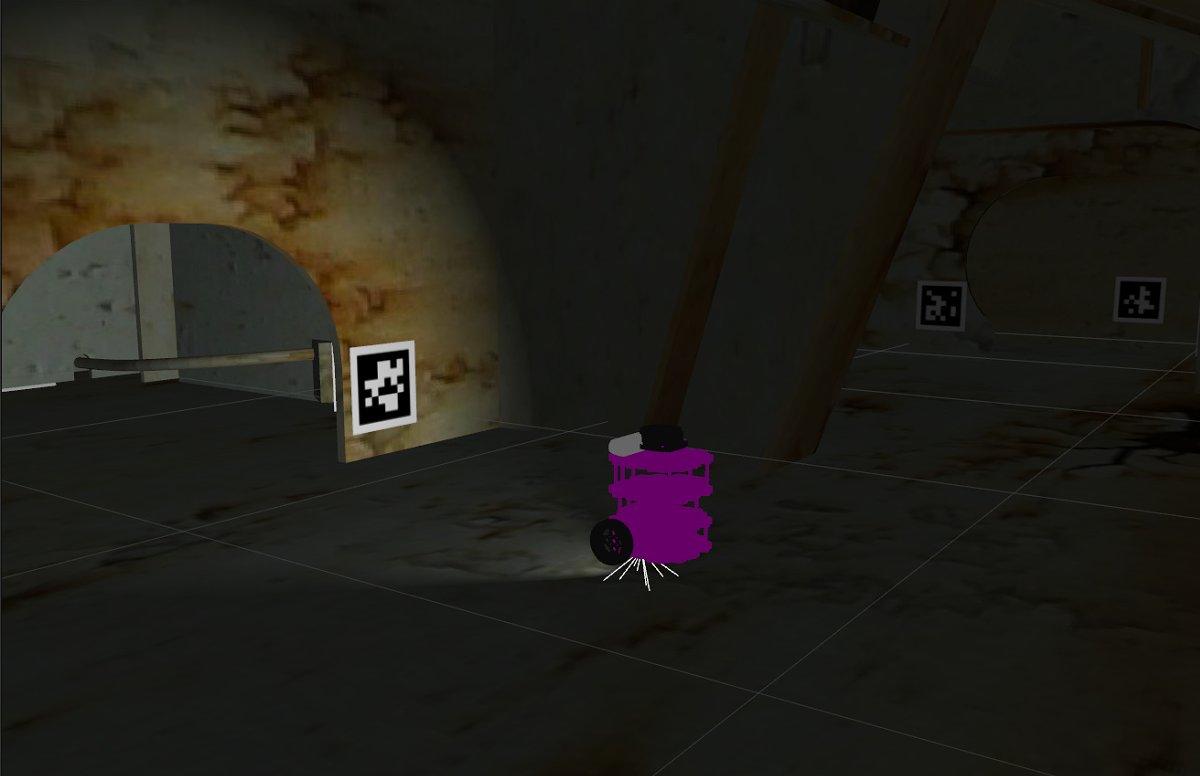}
      \caption{Example picture of the Gazebo simulation with the purple TurtleBot3 Burger running in the ballast tank in Fig.\ref{fig:ballast_cad}, with rusty texture and no external light source. A few scattered AprilTags are used as known features in the tank.}
      \label{fig_gazebo}
   \end{figure}

\subsubsection{Methodology}
Simulation is performed in Gazebo using the ballast tank model in Fig. \ref{fig:ballast_cad}. A rusty texture is applied to the entire tank structure and all the light sources are removed. The robot is a TurtleBot3 Burger with the original camera replaced by a RealSense D435. 
%using the depth camera plug-in. 
AprilTags are scattered around the tank as partially known features and the robot wheel encoders provide odometry information for the hierarchical SLAM framework described in Section \ref{section:slam}.
%the odometry source being wheel odometry. 
A representative image of the simulated environment is shown in Fig. \ref{fig_gazebo}. The same occupancy grid map from Section \ref{section:grid} is used as the static global cost map for the ROS navigation planner and a LiDAR is used to generate the local cost map for the local planner. The number of points in the reference point cloud, \(\mathcal{P}_\text{ref}\), is empirically selected as \(100,000\) to provide sufficient density for the detection of FODs, given the size of the tank and the FODs. The buffer zone threshold, \(d_\text{buffer}\), in Fig.~\ref{fig:anomaly_detect} is selected to be \(3.5\)
cm. The smoothing factor, \(\delta\), in Algorithm \ref{alg:remc} is selected to be \(0.001\).

Each method is allowed to travel on the graph for 35 steps, with two waypoints in the same region per step. 1,000 waypoints are sampled for waypoint placement. The hierarchical SLAM has an estimation horizon of 5 pose nodes, including the inspection nodes. FOD detection is performed when the inspection nodes are pruned from the estimation horizon. The anomaly detection update rule \eqref{eq_bayesian_detect} is initialized with a prior null hypothesis of \(\mathbb{P}(H_0) = 0.8\) given that proving the null hypothesis is more challenging than proving the alternative hypothesis. In other words, the robot assumes the tank is FOD-less until proven otherwise. A set of six common hand tools is chosen as the FOD, comprising 1) power drill \(\times\) 2; 2) spirit level; 3) sander; 4) hammer; 5) measuring tape. \(4\) to \(6\) FODs are sampled for each trial, and each FOD is placed by uniform sampling from the edges of the occupancy map of the tank. If the FODs are hidden behind a structure or blocking the doorways, they are moved to the nearest available location. 

   \begin{figure}[thpb]
      \centering
      \framebox{\parbox{0.35\textwidth}{ \includegraphics[width=0.35\textwidth]{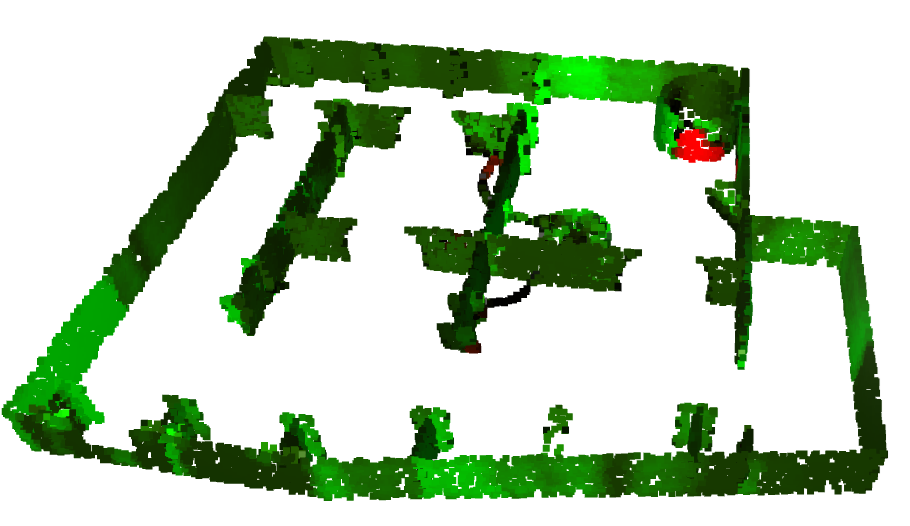}}
}
      \caption{Example of a typical anomaly belief of the reference point cloud for HEMaP inspection with the hue (from red to green) representing the null hypothesis, \(\mathbb{P}(H_0\)), (from 0 to 1); and the value (from color to black) representing the entropy of the hypotheses (from 0 to 0.69, corresponding to the minimum and maximum entropy of the Bernoulli distribution). }
      \label{fig_reference_detected}
   \end{figure}
\begin{figure}[thpb!]
    \centering
    %\fbox{
    \begin{subfigure}[t]{0.24\textwidth}
    
        \centering
        \includegraphics[width=0.95\textwidth]{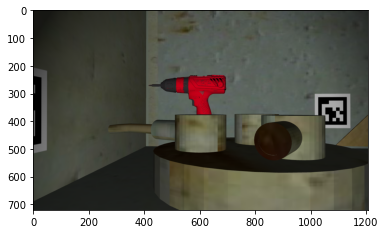}
        \caption{FOD Image}
    \end{subfigure}%
    ~ 
    \begin{subfigure}[t]{0.24\textwidth}
        \centering
        \includegraphics[width=0.95\textwidth]{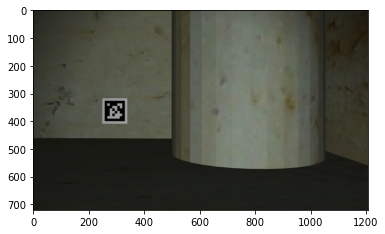}
        \caption{False Positive}
    \end{subfigure}
    %}
    \caption{Example images collected by the robot during the inspection process with (a) showing a power drill behind a cable hub, and (b) showing a false positive with a pillar in the tank.}
         \label{fig_example_imgs}
\end{figure}

The robot is placed in region 1, near the access hole shown in Fig. \ref{fig:ballast_cad}, at the beginning of each trial. After an inspection trial is complete, all the points with \(\mathbb{P}(H_1) \geq 0.5\) from the reference point clouds, shown in Fig. \ref{fig_reference_detected}, are extracted and clustered using hierarchical clustering \cite{h_cluster}, with a distance criterion of \(0.5m\). The centroids of the clusters are calculated, and photos of the tank sections taken from the inspection waypoints with the centroids in view are saved as candidate FODs. During actual deployment, these photos can be presented to a human operator, who can then decide whether the candidates are actually FODs. 
Fig. \ref{fig_example_imgs} shows two example photos saved by the robot. A total of 15 trials are performed and the FODs captured in the photos are analyzed.

\begin{table}[!thpb]
\centering
\begin{tabular}{|c|c|c|c|}
\hline
        & Detected FODs & Missed FODs & False Positives \\ \hline
Random  & 45            & 31          & 19              \\ \hline
Greedy  & 39            & 37          & \textbf{7}               \\ \hline
\textbf{HEMaP} & \textbf{64}            & \textbf{12}          & 21              \\ \hline
\end{tabular}
\caption{Total FOD detection counts (15 trials).}
\label{table_fod_counts}
\end{table}

\subsubsection{Results}
Table \ref{table_fod_counts} shows the results for the FODs detected in the photos by the three inspection methods over 15 trials. A total of 76 FODs are present, of which HEMaP only misses 12 FODs. However, the random and greedy methods miss a much larger number of FODs. This occurs due to an inherent bias toward visiting highly connected regions and high-information regions by the random and greedy methods, respectively. As a result, some of the other tank regions are not inspected adequately when the total number of time steps is limited. The visitation frequency for each region is discussed in more detail in our subsequent analysis.  
%with 64 of them detected by the ergodic method, 45 by the random method, and 39 by the greedy method. 
HEMaP and random methods share a similar number of false positive photos, 
%with 21 and 19 respectively and 
while the greedy has the lowest number. This happens because both HEMaP and random methods distribute the inspection points, whereas, the greedy method focuses on refining several high information points. 

\begin{table}[!thpb]
\centering
\begin{tabular}{|c|c|c|c|}
\hline
        & Mean & Standard Dev. & $p$-value (w/ HEMaP)\\ \hline
Random  & 0.59            & 0.24          & 1.1\(\times 10^{-3}\)              \\ \hline
Greedy  & 0.51            & 0.18          & \(4.8 \times 10^{-9}\)              \\ \hline
\textbf{HEMaP} & \textbf{0.84}            & \textbf{0.13}          &  N/A              \\ \hline
\end{tabular}
\caption{FOD detection rate statistics for 15 simulation trials with hypothesis tests done using a paired $t$-test.}
\label{table_det_rate}
\end{table}

A paired $t$ test is also performed between the HEMaP and the other two FOD detection methods (across all the simulation trials) to investigate any %statistically significant 
difference in detection performance. The results are shown in Table \ref{table_det_rate}. Consistent with the total count results, HEMaP has the highest mean detection rate along with the lowest standard deviation, indicating that it is more consistent in detecting FODs. %Paired $t$-test shows that when comparing the detection rate against ergodic method per trial, 
The HEMaP also shows statistically significant improvements over the random and greedy methods, with corresponding $p$-values of \(1.1\times 10^{-3}\) and \(4.8\times10^{-9}\), respectively. 

\begin{table}[!thpb]
\centering
\begin{tabular}{|c|c|c|c|c|c|c|c|}
\hline
Regions & 1    & 2    & 3    & 4    & 5    & 6    & 7    \\ \hline
Random  & 0.24 & 0.21 & 0.15 & 0.23 & 0.04 & 0.10  & 0.02 \\ \hline
Greedy  & 0.30 & 0.26 & 0.00 & 0.04 & 0.36 & 0.03 & 0.00 \\ \hline
\textbf{HEMaP} & 0.23 & 0.16 & 0.10 & 0.13 & 0.22 & 0.10 & 0.05 \\ \hline
\end{tabular}
\caption{Average distributions of waypoints in different regions in the ballast tank with the three traversal methods.}
\label{table_region_dist}
\end{table}

Comparisons among the three traversal methods in terms of the average distribution of waypoints in the regions are reported in Table \ref{table_region_dist}. Using the region graph shown in Fig. \ref{h_graph}, it is seen that the distribution of waypoints for HEMaP is roughly correlated to the region size, with regions 1 and 5 being highly visited and regions 3 and 7 visited less frequently. 
%but not never visited. 
As expected, the random method is highly susceptible to the region graph structure. Accordingly, the highly connected regions of 1, 2, 3, and 4 are visited most frequently, whereas region 5 with the directed edge is significantly under-visited. This results in a large number of missed FODs when they are spawned in these under-visited regions. The greedy method, on the other hand, is susceptible to an overreliance on the highest estimated information. The robot spends a large amount of time in the two locally largest regions of 1 and 2 until they are completely covered. The robot then quickly moves on to the most information-rich region 5 through regions 4 and 6. The robot never leaves region 5 during the allowed 35 steps, resulting in zero inspection for regions 3 and 7.

   \begin{figure}[thpb]
      \centering
\includegraphics[width=0.35\textwidth]{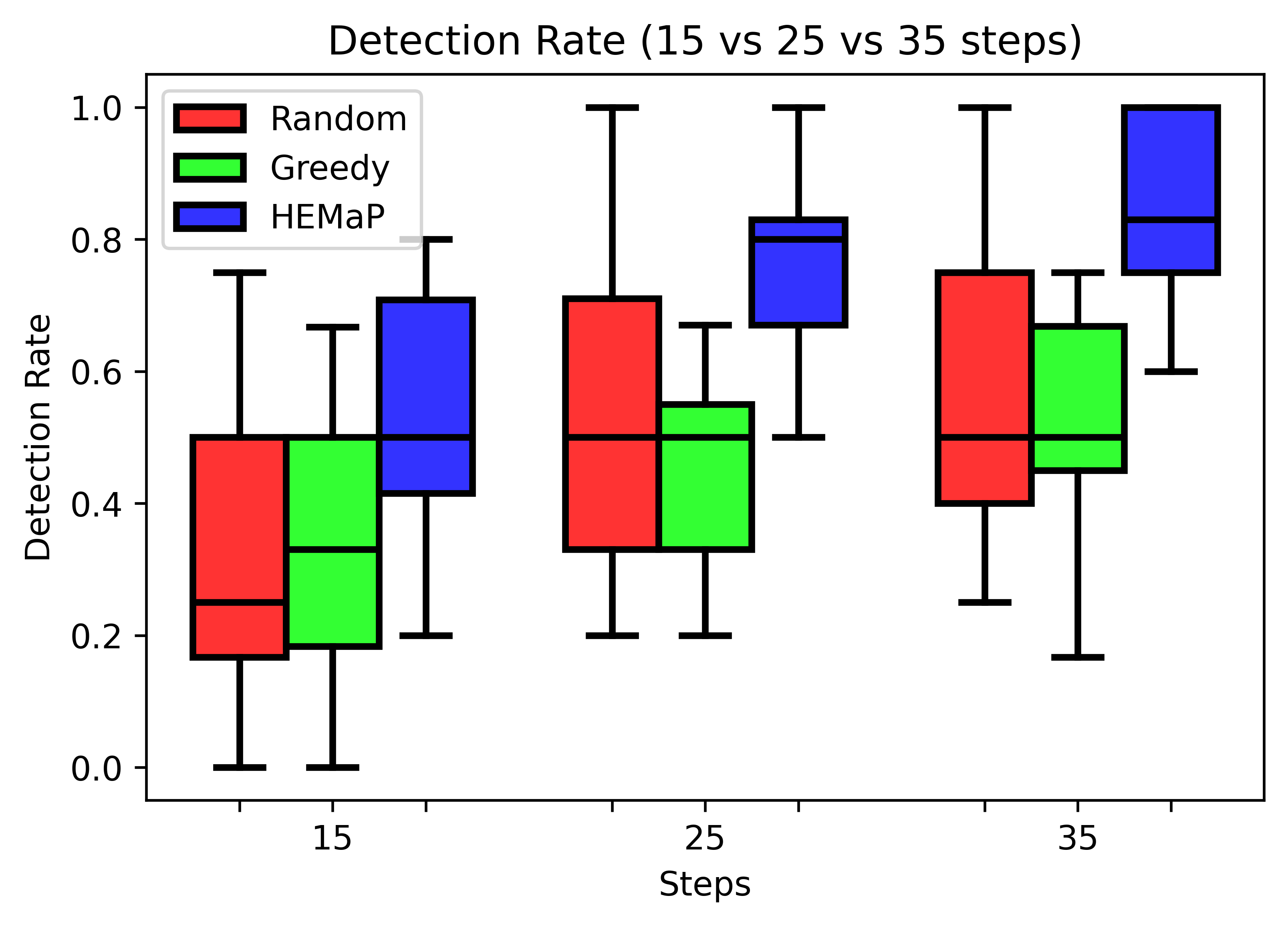}
      \caption{Comparison of FOD detection rates among HEMaP, random, and greedy methods for 15, 25 and 35 steps. Showing that HEMaP is more efficient at detecting FODs even with a limited number of steps.}
      \label{detection_rate}
   \end{figure}

Furthermore, 15 trials with the same setup are performed for 15 and 25 steps, and the detection rates are shown in Fig. \ref{detection_rate}. As expected, all the methods have a decreased detection rate with a smaller number of steps. Nevertheless, even with a small number of steps, HEMaP is %shown to be 
more effective in detecting FODs 
%in the tank 
than the other two methods. This experiment showcases the adaptability of REMC to arbitrary target distributions and how it is %used effectively in conjunction 
combined effectively with the anomaly detector in the HEMaP framework. The random and greedy exploration methods are highly dependent on the graph structure, where the visitation frequencies are determined by the connectivity of the graph, causing the highly connected regions to be visited more frequently. Additionally, the convergence rates of the time averages to the corresponding target distributions are not optimized for the random and greedy methods. Therefore, the short-term visitation frequency is biased toward the regions closer to the starting region, while the far away regions are less likely to be visited.

\subsection{Physical Deployment}
\begin{figure}[thpb]
      \centering
      \framebox{\parbox{0.2\textwidth}{ \includegraphics[width=0.2\textwidth]{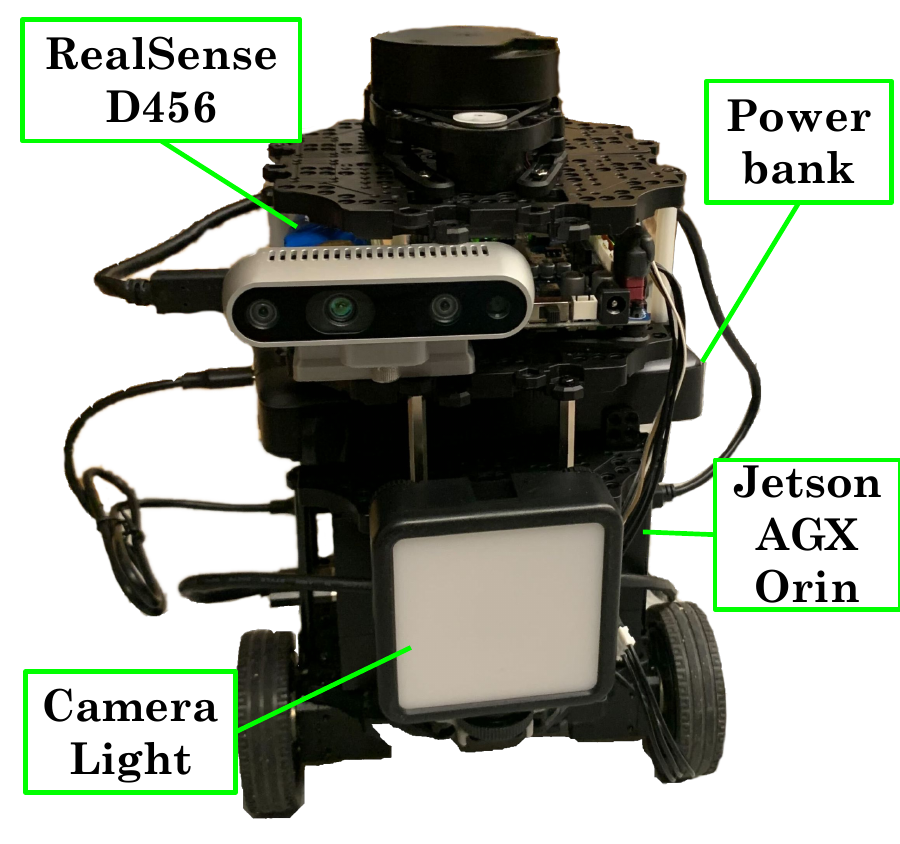}}
}
      \caption{TurtleBot3 Burger used in the physical experiment, customized with camera light, Jetson AGX Orin, RealSense D435 and a power bank.}
      \label{turtlebot}
\end{figure}

The inspection framework using the REMC method from simulation is applied to a physical mock-up ballast tank to verify the simulation results and showcase its capability for real-world deployment. 

\subsubsection{Methodology}
The mock ballast tank is constructed using a combination of plywood and ramboard for the walls. The walls are painted white to mimic a typical tank color and an additional brown paint is used to mimic natural corrosion and rust caused by seawater. AprilTags from the family 36h11 of size \(0.1\) m are scattered around the tank as partially known features for the SLAM. A TurtleBot3 Burger is chosen as the robot platform. The robot is customized by replacing the Raspberry Pi with an Nvidia Jetson AGX Orin running Ubuntu 22.04 and ROS Noetic. The Pi camera is replaced by a RealSense D435 for depth perception. The ball bearing caster wheel is replaced with a \(1.27\) cm swivel caster for better locomotion on uneven ground. An LED camera light is mounted on the robot to illuminate the tank. The TurtleBot LiDAR is used to generate odometry using the RF2O ROS package \cite{JaimezMariano2016Pofa}. 

The hierarchical SLAM is run once in teleoperation mode to map the AprilTags locations before the autonomous inspection trials. The local point clouds are also collected and assembled into a global map to generate the reference CAD model. Similar to Section \ref{sec:gazebo_sim}, the occupancy grid map is generated by projecting the reference CAD model within a height range of \(0.05\) m to \(0.2\) m. Similarly, the occupancy grid map is used for the global planner and the LiDAR is used for the local planner. The number of points in the reference point cloud, \(\mathcal{P}_\text{ref}\), is empirically selected as \(100,000\). The buffer zone threshold, \(d_\text{buffer}\) and the smoothing factor, \(\delta\), are chosen to be \(1\) cm and \(0.001\), respectively. A set of six common hand tools is chosen for the FODs, specifically: 1) power drill \(\times\) 2; 2) masking tape; 3) sander; 4) crimper; 5) sander. A total of 5 trials are performed with 35 steps each. \(3\) to \(5\) FODs were sampled for each trial with a total of \(21\) FODs for the five trials. Each FOD is placed in the tank by uniformly sampling from the edges of the occupancy map. If the FODs are hidden behind a structure or blocking the doorways, they are moved to the nearest available location. 

\begin{figure*}[thpb!]
    \centering
    %\fbox{
     \begin{subfigure}[t]{0.3\textwidth}
        \centering
        \includegraphics[width=0.95\textwidth]{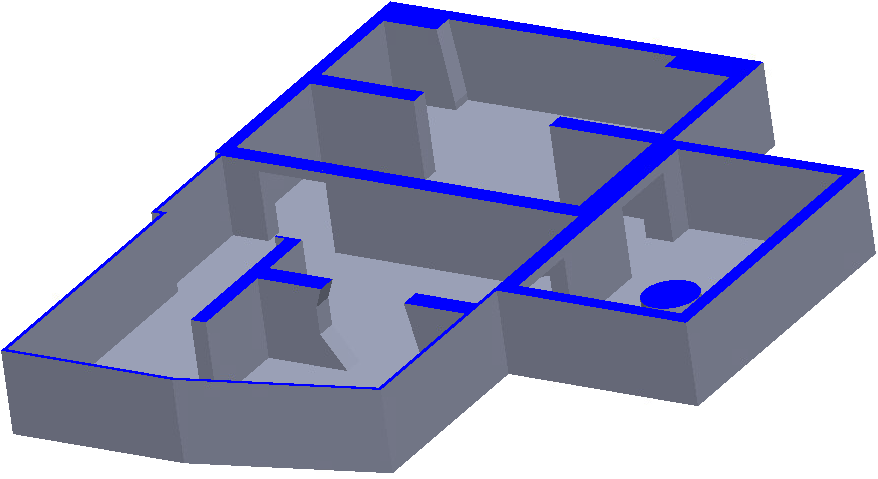}
        \caption{Reference Mesh}
    \end{subfigure}
    ~
    \begin{subfigure}[t]{0.3\textwidth}
    
        \centering
        \includegraphics[width=0.95\textwidth]{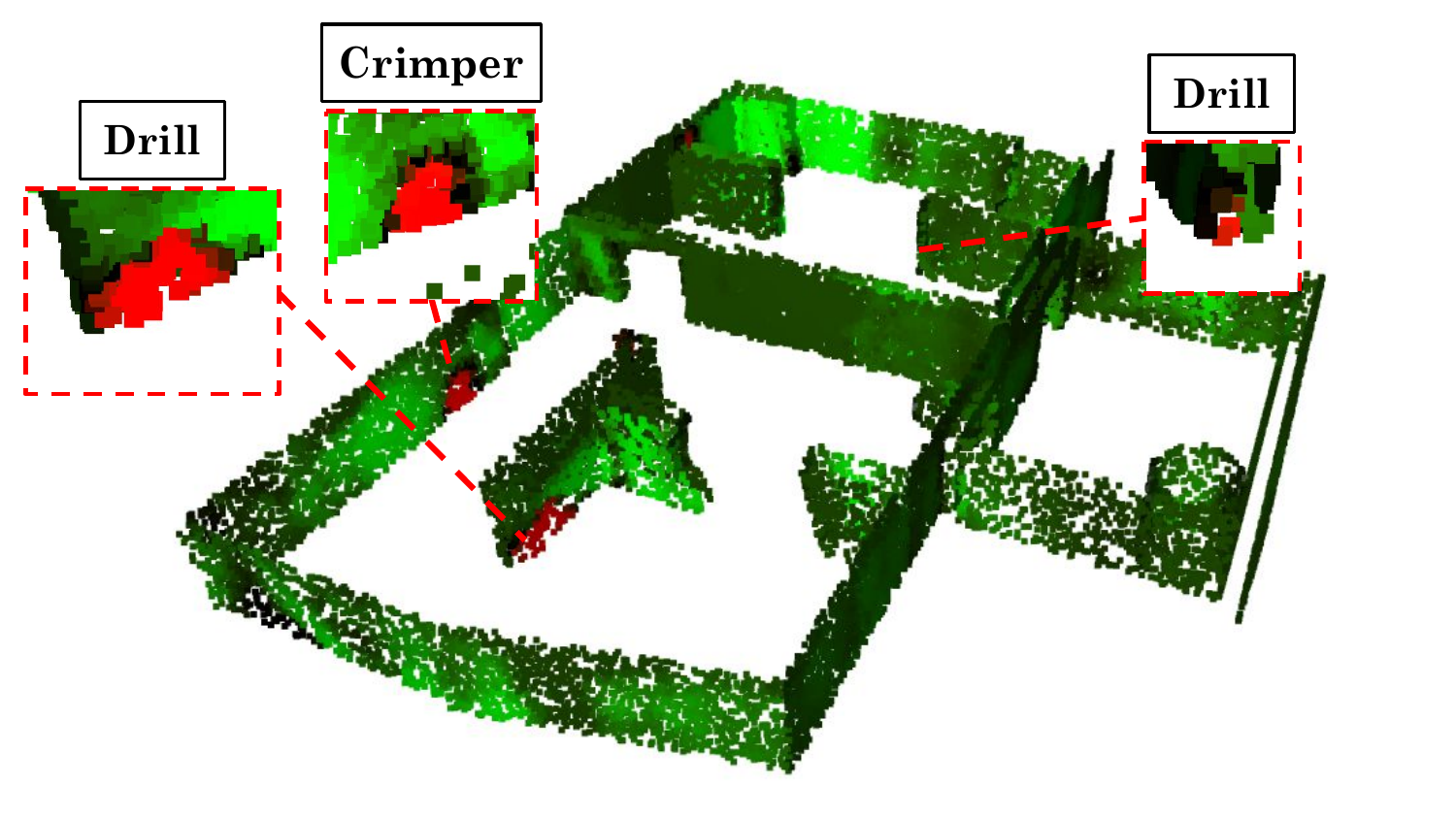}
        \caption{Example trial 1}
    \end{subfigure}%
    ~ 
    \begin{subfigure}[t]{0.3\textwidth}
        \centering
        \includegraphics[width=0.95\textwidth]{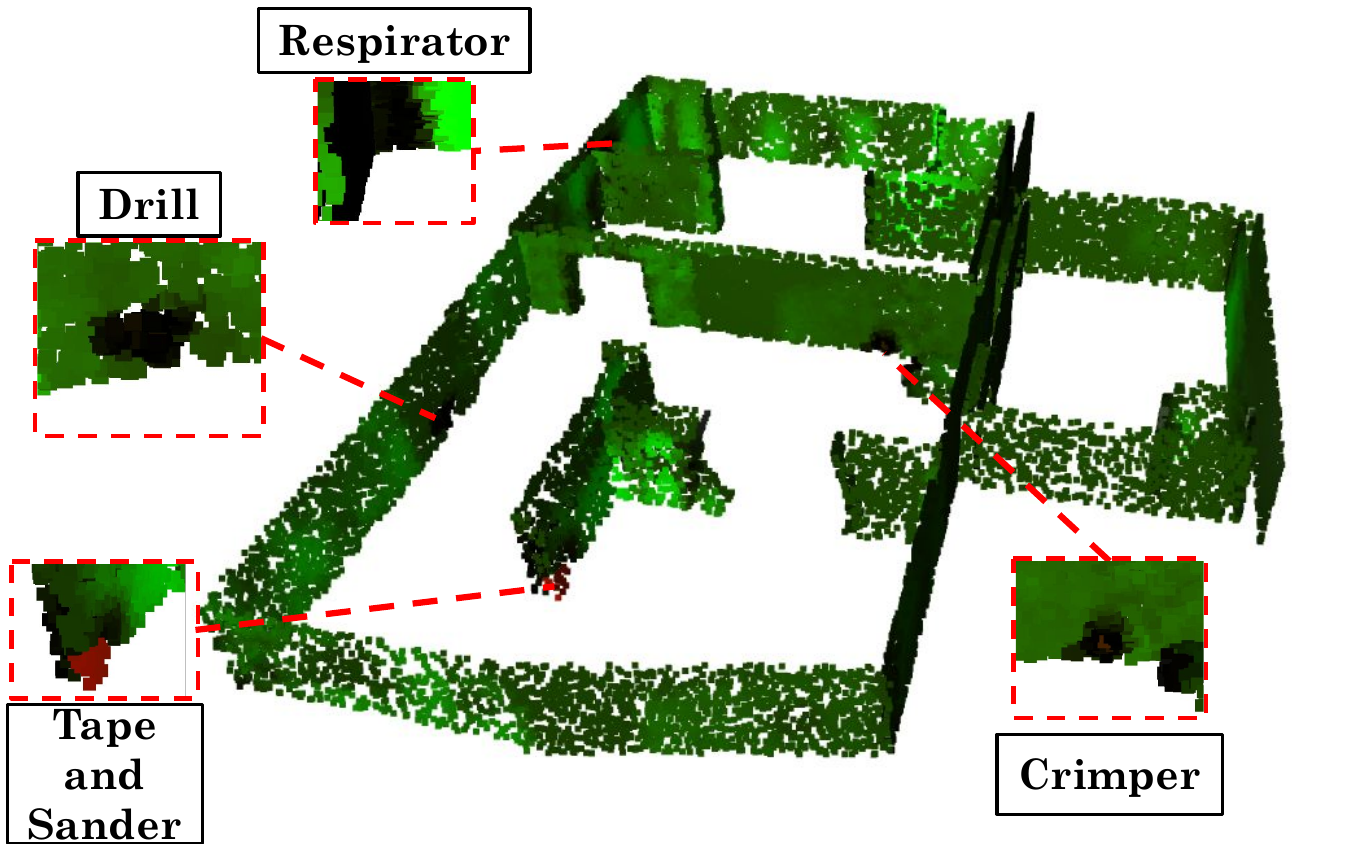}
        \caption{Example trial 2}
    \end{subfigure}
    %}
    \caption{The reference mesh of the mock ballast tank and two example reference point clouds after the inspection showing HEMaP has a good coverage of the workspace in small steps. The true FOD locations are denoted by the text box. The hue (from red to green) representing the null hypothesis, \(\mathbb{P}(H_0\)), (from 0 to 1), i.e. red is considered to be FODs by the robot; and the saturation (from color to black) representing the entropy of the hypotheses (from 0 to 0.69, corresponding to the lowest and highest entropy of Bernoulli distribution, respectively). Figure (a) shows a trial of a set of three FODs and figure (b) shows a trial of a set of 5 FODs. The corresponding detection of these FODs are shown in Fig. \ref{fig_example_imgs_phy}.} 
         \label{fig_ref_phy}
\end{figure*}
\iffalse
\begin{figure*}[thpb!]
    \centering
    %\fbox{
    \begin{subfigure}[t]{0.5\textwidth}
        \centering
        \includegraphics[width=0.95\textwidth]{Figures/trial_4_detection.pdf}
        \caption{Example trial 1}
    \end{subfigure}%
    ~ 
    \begin{subfigure}[t]{0.5\textwidth}
        \centering
        \includegraphics[width=0.95\textwidth]{Figures/trial_8_detection.pdf}
        \caption{Example trial 2}
    \end{subfigure}
    %}
    \caption{\textcolor{red}{rotate map}Occupancy map of the ballast tank showing the high detection rate of HEMaP. Red points shows the candidate FOD locations, where the reference points from the corresponding trial in Fig.\ref{fig_ref_phy} with \(\mathbb{P}(H_1)>0.5\) are segmented and clustered. The photos collected from the inspection by the robot are shown for the true positives.}
    \label{fig_example_imgs_phy}
\end{figure*}
\fi
\begin{figure*}[thpb!]
    \centering
    %\fbox{
    \begin{subfigure}[t]{0.5\textwidth}
        \centering
        \includegraphics[width=0.95\textwidth]{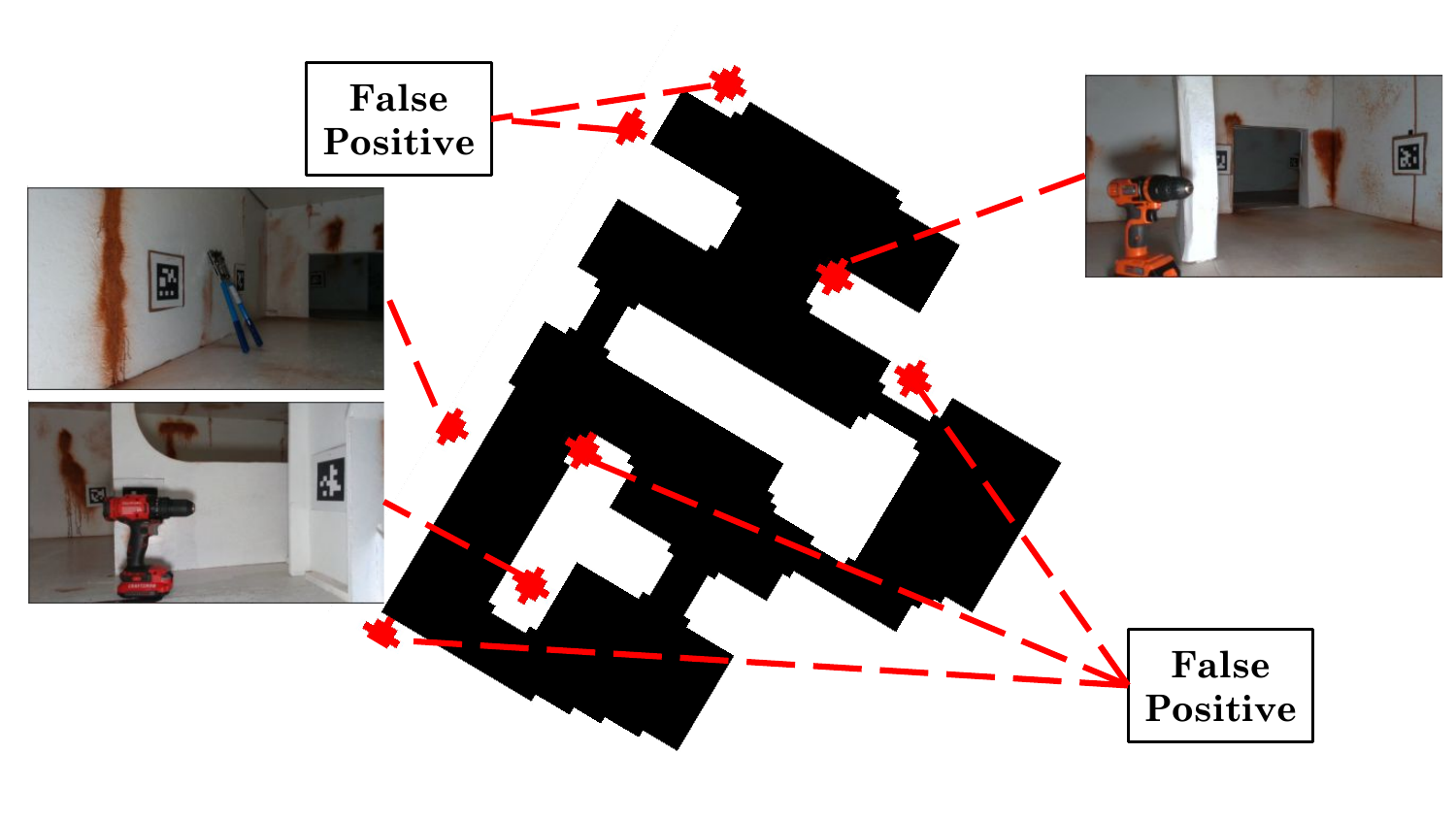}
        \caption{Example trial 1}
    \end{subfigure}%
    ~ 
    \begin{subfigure}[t]{0.5\textwidth}
        \centering
        \includegraphics[width=0.95\textwidth]{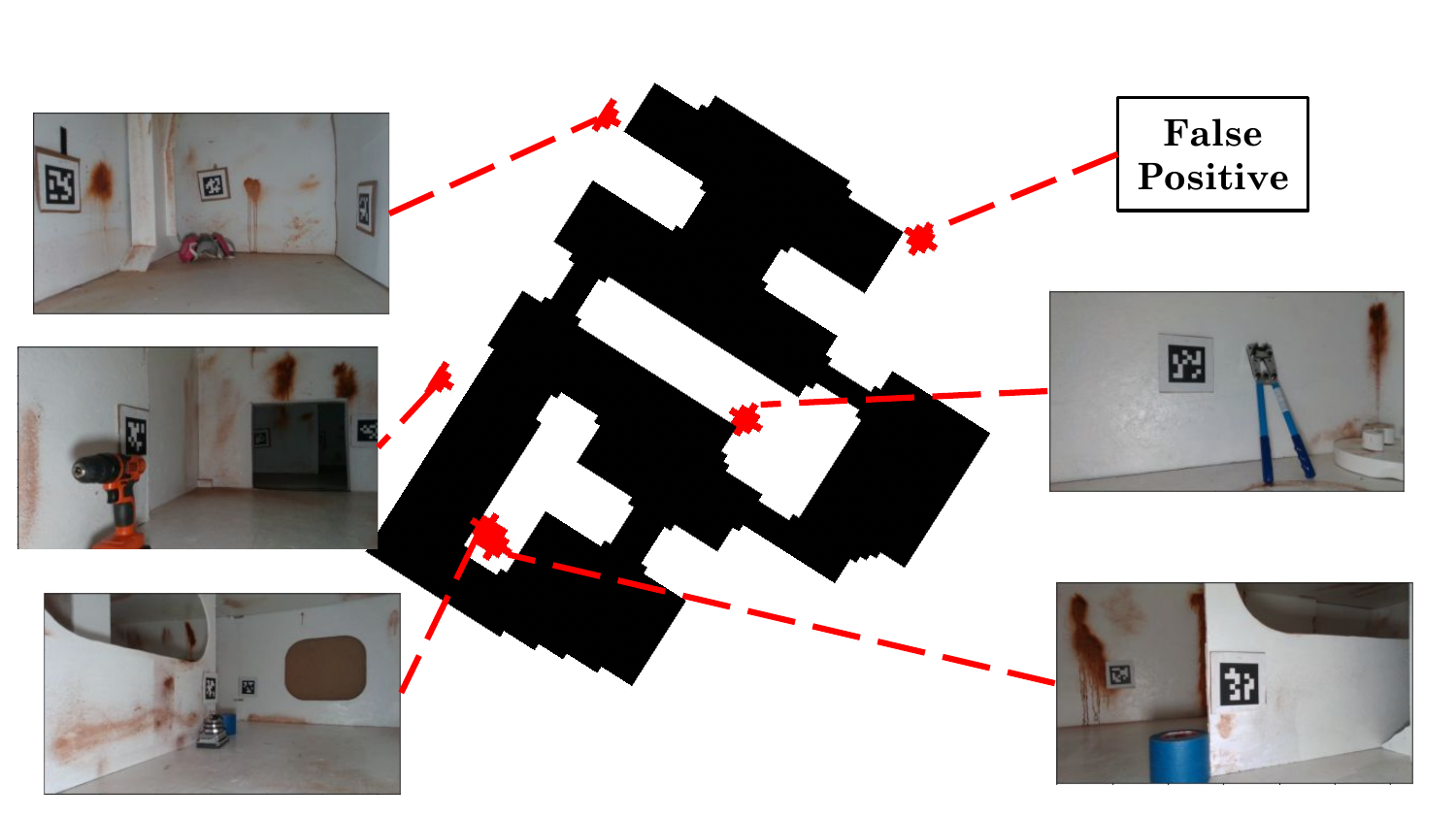}
        \caption{Example trial 2}
    \end{subfigure}
    %}
    \caption{Occupancy map of the ballast tank showing the high detection rate of HEMaP. Red points shows the candidate FOD locations, where the reference points from the corresponding trial in Fig.~\ref{fig_ref_phy} with \(\mathbb{P}(H_1)>0.5\) are segmented and clustered. The photos collected from the robot inspections are shown for the true positives.
    }
    \label{fig_example_imgs_phy}
\end{figure*}

\subsubsection{Results}
The overall FOD detection rate %(average over the trials) 
is \(0.853 \pm 0.2\), which is similar to the detection rate in simulation. The average number of false positives is \(3.4\), which is slightly higher than in simulation. This is expected given that sensor readings are noisier in real life, together with modeling inaccuracy in the reference point cloud. Fig. \ref{fig_ref_phy} shows examples of the reference point cloud from two trials after inspection. More black patches (close to 0.5 posterior confidence) are observed compared to the simulation trials, meaning that there is more confusion from the robot. The locations of the true FODs are often in bright red, representing a high confidence. Fig. \ref{fig_example_imgs_phy} shows the associated points of interest, where the points with an anomaly posterior greater than \(0.5\) are segmented and clustered. In both the trials, all the FODs are found, with trial 1 having more false positives. It can be seen that most of the false positives arise in the corners of the structures. These corners are sensitive to misalignment in the reference model. The depth camera is also noisier around the corners with sharp changes in depth. 
%Both effects contributes to the corner false positives.

\section{Discussion}
\label{section:discussion}
The high variance in performance arising from the stochasticity of the Markov chain can be reduced through Monte Carlo tree search (MCTS) \cite{BestGraeme2019DDpf}, although as shown in Appendix C, the variance of the time average of REMC decreases rapidly with the number of steps. In MCTS, multiple sequences with a predefined time horizon are sampled from the Markov chain, and the sequence with the highest increase in ergodicity is selected instead of sampling the next state at every step. Here, we use the single-step sampling option to investigate the fundamental capability of an ergodic Markov chain and devise an exploration framework that can be extended to communication-limited multi-robot systems. This extended framework can be potentially applied to a wider array of exploration problems beyond inspection, such as wildlife monitoring in a biome and lost object retrieval in a building environment.

An interesting potential of the hierarchical planning framework is that graph ergodicity can be used in conjunction with the various continuous space exploration methods. In this case, the workspace can be partitioned into spaces with simpler topology, and instead of waypoint placement, the robot executes a continuous space ergodic trajectory within each region. This alleviates the need for calculating the basis functions globally for the complex topology and instead focuses on the fine details only in the local region, while the region transition is handled by the Markov chain. 

An ergodic Markov chain can also be potentially used in continuous space. The Metropolis-Hasting algorithm allows one to construct a continuous-space Markov chain with any target distribution by using the detailed balance condition. 
%Another interesting property is that 
Since the detailed balance condition is completely local, the robot can execute an ergodic transition without knowing the global distribution. However, as discussed earlier, the detailed balance condition impacts the convergence rate to time-discounted ergodicity and requires a reversible motion model of the robot. To improve the convergence rate and remove the reversibility requirement, it would be interesting to 
%a direction to extend this work is to 
study the potential of REMC in continuous space.

It is also possible to consider an ergodic Markov chain in the state space instead of the physical space. The state space can be either graph-based or topological, with inherently stochastic dynamics analogous to a Markov decision process. A classical strategy for state space exploration in reinforcement learning (RL) is to uniformly sample an available action \cite[p.32]{bremaud2013markov}. This is analogous to the random method in Section \ref{sec:gazebo_sim}, where the visitation frequency is highly dependent on the connectivity of the space. Therefore, we are interested in investigating whether an on-policy RL method can be developed using an ergodic Markov chain-based stochastic policy.

In this work, the transition times 
%of the robot 
between the regions are not considered. The transition model of \cite{BermanS.2009OSPf} can be potentially applied to REMC to reduce extensive region switching. This also opens another 
%possible 
application of rapid ergodicity to Markov processes, which is the continuous-time version of Markov chains. In the context of confined space work, modeling the transition time, or equivalently, region residence time, opens the method to more types of tasks that require long wait times, such as removal of the detected FODs or on-the-spot maintenance of damaged structures. A time sensitive ergodic Markov chain can explicitly balance the visitation frequency according to how long the robot would reside in a region. 

A final direction would be to extend the anomaly detection method to a broader class of anomalies that include rust patches and structural damages, in addition to FODs. Conceptually, a machine learning-based object detection model with a confidence score, such as \cite{NashWill2022Dlcd,THOR2}, can be used to replace the likelihood calculation in \eqref{eq:null_likelihood}-\eqref{eq:alt_likelihood}, while the Bayesian update and information metric can remain the same. 

\section{Conclusions}

This paper addresses the problem of efficient active robotic inspection of confined spaces using anomaly detection as a case study. 
First, an ergodic exploration method, named rapidly ergodic Markov chain (REMC), is introduced for a graph-based discretization of the confined space. REMC is shown to have an optimal convergence rate to time-discounted ergodicity on any reversible graph, implying it prioritizes early visitations of information-rich graph nodes (regions). REMC also provides an upper bound on the convergence rate to time-discounted ergodicity for any strongly connected graph. Therefore, REMC achieves more meaningful ergodicity in fewer time steps than the classical fastest mixing Markov chain method, which is optimized for the convergence rate to the target distribution. 
Second, a \underline{h}ierarchical \underline{e}rgodic \underline{Ma}rkov \underline{p}lanner (HEMaP) is developed that uses REMC for region graph traversal, a waypoint planner to select the inspection points inside the regions, and a standard local planner for collision-free waypoint navigation.
HEMaP is empirically shown to have better convergence to region-based ergodicity than 
a widely-used and two more-recent, continuous space ergodic control methods (SMC, HEDAC and GESCE, respectively).Finally, a SLAM-driven Bayesian anomaly detection method is incorporated within HEMaP for robust localization of foreign object debris (FODs) in the confined space. Both simulation and physical experiments in a ballast tank using a ground robot show that HEMaP leads to a significantly higher FOD detection rate compared to greedy and random inspection methods.

\bibliographystyle{IEEEtran} 
\bibliography{references}

%\clearpage
\begin{appendices}

\section{Span of initial distributions}
\label{tightness}
\begin{remark}
    \begin{figure}[thpb]
      \centering
      \framebox{\parbox{0.3\textwidth}{ \includegraphics[width=0.3\textwidth]%{Figures/simplex.pdf}
      {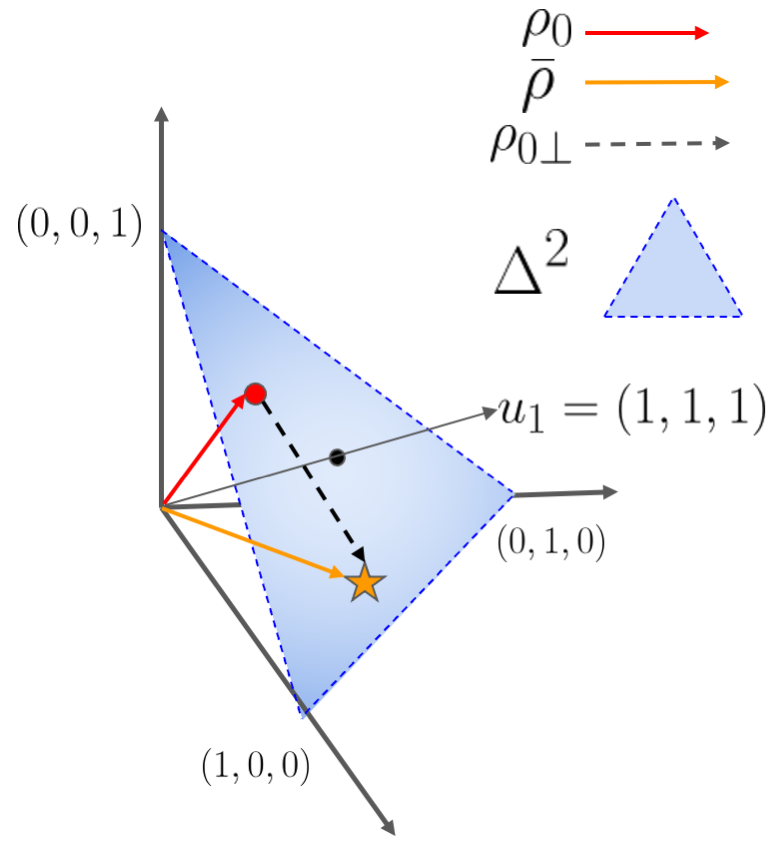}
      }
}
      \caption{Graphical illustration of a simplex for a 3-node graph, such as the one in Fig.~\ref{fig_detailed_balance}, which contains all valid distributions, \(\rho\), and closed under transformation by any stochastic matrix, \(P\). \(u_1\) is the first left eigenvector of the stochastic matrix \(P\); the blue triangle is the 2-D unit simplex, \(\Delta^2\); red arrow is the arbitrary initial distribution, \(\rho_0\); orange arrow is the target distribution, \(\bar{\rho}\);
      and the dashed arrow is the difference between the target distribution and the initial distribution, which is always parallel to the simplex.
      }
      \label{3-simplex}
   \end{figure}
   
    The set of all n-dimensional probability vector forms an \((n-1)\) dimensional unit simplex, denoted by \(\Delta^{n-1}\).
    \begin{equation}
       \Delta^{n-1} \triangleq \{\rho \in \mathbb{R}^n | \textbf{1}^T\rho = 1, \quad \rho \succeq 0 \}
    \end{equation}
    
     The hyperplane formed by the difference between any two members of \(\Delta^{n-1}\) is orthogonal to the first left eigenvector \(u_1 = \textbf{1}\) of any stochastic matrix \(P\). This hyperplane is denoted \(u_1^\perp\). A 3-D graphical example is illustrated in Fig. \ref{3-simplex}. Additionally, if \(P\) is symmetric, then \(v_1 \sim u_1 \) and \(u_1^\perp = v_1^\perp\)
    
\end{remark}
The following lemma shows that for a uniform stationary distribution \(\frac{1}{n}\textbf{1}\), there always exist a distribution \(\rho\) such that the orthogonal components on the simplex  \(\rho_\perp\) spans the hyperplane \(u_1^\perp\). 
\begin{lemma}
\label{tightness_lemma}
  Any nonzero vector $x$ in the orthogonal space of $\textbf{1} $  is the scaling of the orthogonal component of a distribution $\rho$. 
\end{lemma}

\begin{proof}
Consider an arbitrary non-zero vector $x$ orthogonal to $\textbf{1}$
\begin{equation}
    x \in \{x | \textbf{1}^Tx = 0\}
    \label{eq_x_perp_def}
\end{equation}
which implies that $x$ contains both positive and negative values to sum to zero. Therefore, the minimum term of $x$ is less than zero,  
\begin{equation}
\min_i(x_i) < 0 . 
\end{equation}
Consider  $x_{\epsilon}$,  which is pointing in  the same direction of \(x\) given by
\begin{equation}
    x_{\epsilon} =  \epsilon x
\end{equation}
where 
\begin{equation}
    \begin{aligned}
        0<  \epsilon \leq \frac{1}{n |\min(x) |}.
    \end{aligned}
\end{equation}
Then, terms of $\rho = \epsilon x +  \frac{1}{n}\textbf{1}  $ are non-negative, i.e., 
\begin{equation}
    \begin{aligned}
         \rho = \epsilon x +  \frac{1}{n}\textbf{1} \succeq 0.
    \end{aligned}
\end{equation}
Moreover, $\rho$ sums to one from~\eqref{eq_x_perp_def} 
\begin{equation}
    \begin{aligned}
        \textbf{1}^T \rho = \textbf{1}^T\epsilon x +  \textbf{1}^T\frac{1}{n}\textbf{1} = 
        \textbf{1}^T\frac{1}{n}\textbf{1} = 1.
    \end{aligned}
\end{equation}
and is therefore a valid distribution, and its orthogonal component  \(\rho_\perp\) is a scaling of $x$ i.e., 
\begin{equation}
\rho_\perp = \rho - \frac{1}{n}\textbf{1} = \epsilon x .
\end{equation}

\end{proof}

\section{3 Nodes Simplex Example}
\label{section:3_node_example}
\begin{exmp}

 For any stochastic matrix \(P\), the dynamics is closed in the simplex \(\Delta^{n-1}\), i.e. the distribution vector stays as distribution vector when multiplied by P. From that, the dynamic of the Markov chain can be studied by restricting the \(P\) in the subspace of the simplex, resulting in \(P|_{u_1^\perp}\), and the dynamic is operating on the orthogonal component \(\rho_{\perp} = \rho - \bar{\rho}\). 
Fig. \ref{simplex_traj} shows the difference of the trajectory generated by FMMC and REMC, with the simplex transformed to a 2-D space. The figure shows the trajectory of the distribution (red) and the trajectory of the time average (blue) for four time steps. 
As shown in the plots, 
the time average, \(\hat{\rho}\), of REMC converges faster than that of FMMC, even though the FMMC distribution \({\rho}\) converges to the target distribution, \(\bar{\rho}\) faster than the REMC case. 

    \begin{figure}[thpb]
      \centering
    \includegraphics[width=0.4\textwidth]{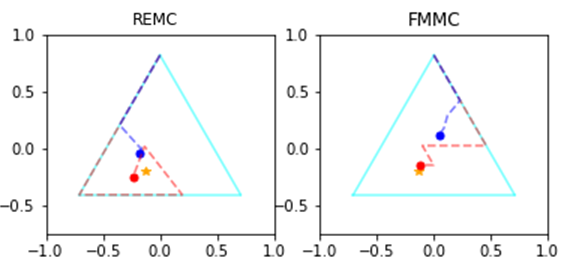}
      \caption{Comparison of REMC and FMMC.  REMC has a faster convergence for the time average, \(\hat{\rho}\) of the distribution 
      \(\rho_k\), 
      %by intensionally overshooting the distribution, \(\rho_k\), 
      compared to FMMC even though the distribution 
      \(\rho_k\) itself converges to the target distribution \(\bar{\rho}_k\) faster with FMMC.
      Here, the cyan line is the boundary of the simplex; the red marker is the distribution \(\rho_k\) at current step \( k=4\); the red dashed line is the trajectory, \(\rho_{0:k}\); the blue marker is the time average \(\hat{\rho}_k\) at current step, \( k\); and the blue dashed line is the time-averaged trajectory, \(\hat{\rho}_{0:k}\). 
      }
      \label{simplex_traj}
   \end{figure}
\end{exmp}

\section{Variance of REMC}
\label{section:variance_remc}
 \begin{figure}[thpb]
      \centering
 \includegraphics[width=0.45\textwidth]{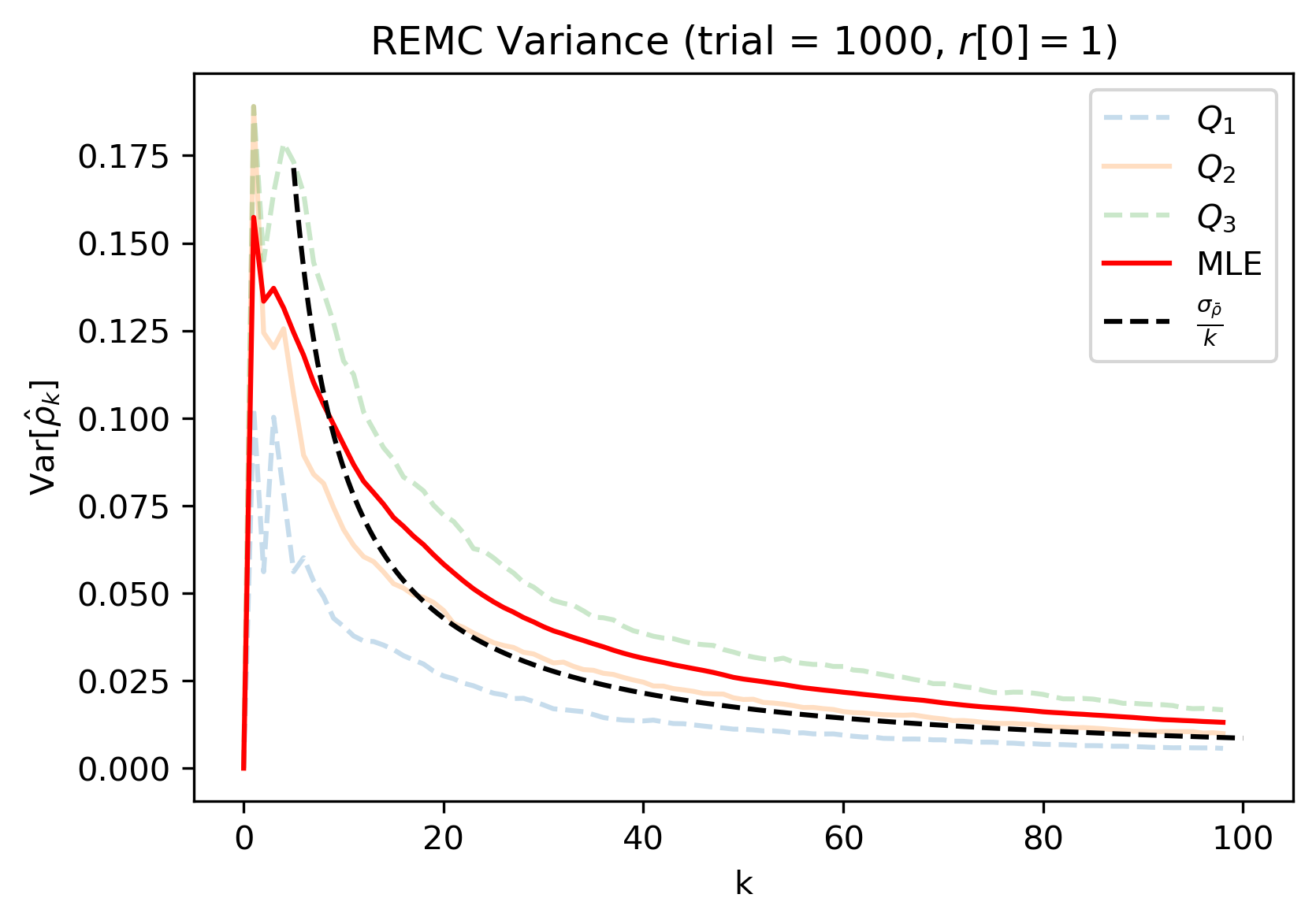}

      \caption{Variance of the time average of REMC at each time step \(k\) over \(1,000\) trials with initial condition \(r[0] = 1\) and a uniform target distribution \(\bar\rho = \frac{1}{n}\mathbf{1}\). Red curve is the maximum likelihood estimation (MLE) of the variance and black dashed line is the variance estimated by Markov chain central limit theorem (CLT).
      }
      \label{fig:variance}
   \end{figure}
   
   As REMC is a stochastic method, one might be interested in the robustness of the coverage generated by such a method. The variance of the (finite) time average of a Markov chain with a general stochastic matrix \(P\) is still an open problem, often found in the Markov chain Monte Carlo literature. 

More precisely, we are interested in the variance at each time step \(k\) with respect to the expected \(\hat\rho_k\), where
    \begin{equation}
        \begin{aligned}
            \hat{\rho}_k &= \frac{1}{K}\sum_{k=0}^{K-1}I(r[k])\\
            \mathbb{E}[\hat{\rho}_k] &= \frac{1}{K}\sum_{k=0}^{K-1}P^k\rho_0\\
            \text{Var}[\hat{\rho}_k] & = 
            \mathbb{E}[(\hat{\rho}_k - \mathbb{E}[\hat{\rho}_k])(\hat{\rho}_k - \mathbb{E}[\hat{\rho}_k])^T].
        \end{aligned}
    \end{equation}
    The Markov chain central limit theorem (CLT) states that the variance of the time average is proportional to \(\frac{1}{k}\), i.e., 
    \begin{equation}
        \begin{aligned}
           \text{Var}(\hat\rho_k) \propto \frac{1}{k}
        \end{aligned}
    \end{equation}
    as \(k\) becomes large \cite{JonesGalinL.2004OtMc}.    

To verify the approximation, a simulation is conducted using the graph from Fig. \ref{h_graph} with target distribution \(\bar\rho=\frac{1}{n}\mathbf{1}\) and initial condition \(r[0]= 1\). \(m = 1,000\) trials are conducted and variance is computed using maximum likelihood estimation (MLE) as
     \begin{equation}
        \begin{aligned}
           \text{tr}(\text{Var}[\hat\rho_k]) \approx \frac{1}{m}\sum_{j=1}^m (\hat\rho_{k,j}-\mathbb{E}[\hat{\rho}_k])^T(\hat\rho_{k,j}-\mathbb{E}[\hat{\rho}_k]).
        \end{aligned}
    \end{equation}
    The CLT variance is calculated assuming i.i.d. samples to be
    \(\text{tr}(\text{Var}(\hat\rho_k)) \approx \frac{1}{k}\sum_i
    ^n\bar\rho_i(1- \bar\rho_i)\). The simulation results are shown in Fig. \ref{fig:variance}.
    
It is observed that the MLE variance has the general shape of \(\frac{1}{k}\) roughly after \(20\) steps, which shows a rapid decrease in the variance. The CLT prediction has a faster decrease, as expected from \cite{JonesGalinL.2004OtMc}, with the MLE converging to it as \(k\) grows. Interestingly, the median \((Q_2)\) values are closely approximated by the CLT predictions. Therefore, we can conclude that REMC 
    %is stochastic, the 
    variance decreases rapidly as the number of steps increases, and the results are closely predicted by the expected value.  
\end{appendices}

\end{document}